\documentclass[11pt]{article}
\pdfoutput=1  
\usepackage{etoolbox}

\usepackage{booktabs} 
\usepackage[suppress]{preamble/color-edits}      %

\newtoggle{icml}
\togglefalse{icml} 
\newtoggle{arxiv}
\toggletrue{arxiv} 

\def\ddefloop#1{\ifx\ddefloop#1\else\ddef{#1}\expandafter\ddefloop\fi}

\def\ddef#1{\expandafter\def\csname bb#1\endcsname{\ensuremath{\mathbb{#1}}}}
\ddefloop ABCDEFGHIJKLMNOPQRSTUVWXYZ\ddefloop

\def\ddefloop#1{\ifx\ddefloop#1\else\ddef{#1}\expandafter\ddefloop\fi}
\def\ddef#1{\expandafter\def\csname frak#1\endcsname{\ensuremath{\mathfrak{#1}}}}
\ddefloop ABCDEFGHIJKLMNOPQRSTUVWXYZ\ddefloop

\def\ddefloop#1{\ifx\ddefloop#1\else\ddef{#1}\expandafter\ddefloop\fi}
\def\ddef#1{\expandafter\def\csname fr#1\endcsname{\ensuremath{\mathfrak{#1}}}}
\ddefloop ABCDEFGHIJKLMNOPQRSTUVWXYZ\ddefloop

\def\ddefloop#1{\ifx\ddefloop#1\else\ddef{#1}\expandafter\ddefloop\fi}
\def\ddef#1{\expandafter\def\csname eul#1\endcsname{\ensuremath{\EuScript{#1}}}}
\ddefloop ABCDEFGHIJKLMNOPQRSTUVWXYZ\ddefloop

\def\ddefloop#1{\ifx\ddefloop#1\else\ddef{#1}\expandafter\ddefloop\fi}
\def\ddef#1{\expandafter\def\csname scr#1\endcsname{\ensuremath{\mathscr{#1}}}}
\ddefloop ABCDEFGHIJKLMNOPQRSTUVWXYZ\ddefloop

\def\ddefloop#1{\ifx\ddefloop#1\else\ddef{#1}\expandafter\ddefloop\fi}
\def\ddef#1{\expandafter\def\csname b#1\endcsname{\ensuremath{\mathbf{#1}}}}
\ddefloop ABCDEFGHIJKLMNOPQRSTUVWXYZ\ddefloop

\def\ddefloop#1{\ifx\ddefloop#1\else\ddef{#1}\expandafter\ddefloop\fi}
\def\ddef#1{\expandafter\def\csname bhat#1\endcsname{\ensuremath{\hat{\mathbf{#1}}}}}
\ddefloop ABCDEFGHIJKLMNOPQRSTUVWXYZ\ddefloop

\def\ddefloop#1{\ifx\ddefloop#1\else\ddef{#1}\expandafter\ddefloop\fi}
\def\ddef#1{\expandafter\def\csname btil#1\endcsname{\ensuremath{\tilde{\mathbf{#1}}}}}
\ddefloop ABCDEFGHIJKLMNOPQRSTUVWXYZ\ddefloop

\def\ddefloop#1{\ifx\ddefloop#1\else\ddef{#1}\expandafter\ddefloop\fi}
\def\ddef#1{\expandafter\def\csname bst#1\endcsname{\ensuremath{\mathbf{#1}^\star}}}
\ddefloop ABCDEFGHIJKLMNOPQRSTUVWXYZ\ddefloop

\def\ddefloop#1{\ifx\ddefloop#1\else\ddef{#1}\expandafter\ddefloop\fi}
\def\ddef#1{\expandafter\def\csname bst#1\endcsname{\ensuremath{\mathbf{#1}^\star}}}
\ddefloop abcdefghijklmnopqrstuvwxyz\ddefloop

\def\ddefloop#1{\ifx\ddefloop#1\else\ddef{#1}\expandafter\ddefloop\fi}
\def\ddef#1{\expandafter\def\csname bhat#1\endcsname{\ensuremath{\hat{\mathbf{#1}}}}}
\ddefloop abcdefghijklmnopqrstuvwxyz\ddefloop

\def\ddefloop#1{\ifx\ddefloop#1\else\ddef{#1}\expandafter\ddefloop\fi}
\def\ddef#1{\expandafter\def\csname sy#1\endcsname{\ensuremath{{\boldsymbol{#1}}}}}
\ddefloop abcdeghijklmnopqrstuvwxyz\ddefloop

\def\ddefloop#1{\ifx\ddefloop#1\else\ddef{#1}\expandafter\ddefloop\fi}
\def\ddef#1{\expandafter\def\csname b#1\endcsname{\ensuremath{\mathbf{#1}}}}
\ddefloop abcdeghijklnopqrstuvwxyz\ddefloop

\def\ddefloop#1{\ifx\ddefloop#1\else\ddef{#1}\expandafter\ddefloop\fi}
\def\ddef#1{\expandafter\def\csname barb#1\endcsname{\ensuremath{\bar{\mathbf{#1}}}}}
\ddefloop abcdefghijklmnopqrstuvwxyz\ddefloop

\def\ddef#1{\expandafter\def\csname c#1\endcsname{\ensuremath{\mathcal{#1}}}}
\ddefloop ABCDEFGHIJKLMNOPQRSTUVWXYZ\ddefloop
\def\ddef#1{\expandafter\def\csname h#1\endcsname{\ensuremath{\widehat{#1}}}}
\ddefloop ABCDEFGHIJKLMNOPQRSTUVWXYZ\ddefloop
\def\ddef#1{\expandafter\def\csname hc#1\endcsname{\ensuremath{\widehat{\mathcal{#1}}}}}
\ddefloop ABCDEFGHIJKLMNOPQRSTUVWXYZ\ddefloop
\def\ddef#1{\expandafter\def\csname t#1\endcsname{\ensuremath{\widetilde{#1}}}}
\ddefloop ABCDEFGHIJKLMNOPQRSTUVWXYZ\ddefloop

\def\ddef#1{\expandafter\def\csname tc#1\endcsname{\ensuremath{\widetilde{\mathcal{#1}}}}}
\ddefloop ABCDEFGHIJKLMNOPQRSTUVWXYZ\ddefloop

\newcommand{\bxi}{\bm{\xi}}

\newcommand{\btau}{\bm{\tau}}

\usepackage{multirow,nicefrac}
\usepackage{makecell,upgreek}
\usepackage{footnote}
\usepackage{longtable}
\usepackage[shortlabels]{enumitem}
\usepackage{tablefootnote}
\usepackage[T1]{fontenc}
\usepackage{verbatim} 
\usepackage[utf8]{inputenc}

\usepackage{booktabs}   
\usepackage{float}

\usepackage{xcolor}              
\definecolor{Orange}{rgb}{1,0.5,0.25}

\addauthor{kz}{brown}
\addauthor{ms}{magenta}
\addauthor{aa}{cyan}
\addauthor{ak}{Orange}

\usepackage{amsthm}

\iftoggle{arxiv}
{
\usepackage{subfig}
  \usepackage[breaklinks=true]{hyperref}

  \usepackage{boxedminipage}
  \usepackage{fullpage}
  \usepackage[noend]{algpseudocode}
  \usepackage{algorithm}
}
{
   \usepackage[noend]{algpseudocode}
  \usepackage{algorithm}
    \usepackage[breaklinks=true]{hyperref}

}

\usepackage{amsfonts,amssymb,mathrsfs}
\usepackage{mathtools}
\DeclareMathSymbol{\shortminus}{\mathbin}{AMSa}{"39}

	\usepackage{natbib}

\usepackage{prettyref}

\usepackage[capitalise,nameinlink]{cleveref}
\Crefname{equation}{Eq.}{Eqs.}
\Crefname{assumption}{Assm.}{Assms.}
\Crefname{condition}{Condition}{Conditions}
\Crefname{claim}{Claim}{Claims}
\crefformat{footnote}{#2\footnotemark[#1]#3}

\newcommand{\icmlpar}[1]{\iftoggle{arxiv}{\paragraph{#1}}{\textbf{#1}}}

\usepackage{breakcites,bm}

\hypersetup{colorlinks,citecolor=blue,linkcolor = blue}
\usepackage{latexsym}
\usepackage{relsize}
\usepackage{thm-restate}
\usepackage{appendix}

\usepackage{xcolor}
\usepackage{dsfont}

\newcommand{\N}{\mathbb{N}}

\newcommand{\R}{\mathbb{R}}

\numberwithin{equation}{section}
\newcommand\numberthis{\addtocounter{equation}{1}\tag{\theequation}}

\newcommand{\ftil}{\tilde{f}}

\newcommand{\rmd}{\mathrm{d}}

\newcommand{\bzero}{\ensuremath{\mathbf 0}}

\def\bu{\mathbf{u}}

\def\by{\mathbf{y}}
\def\bz{\mathbf{z}}

\def\bu{\mathbf{u}}

\def\bw{\mathbf{w}}

\def\bI{\mathbf{I}}

\DeclareFontFamily{U}{mathx}{\hyphenchar\font45}
\DeclareFontShape{U}{mathx}{m}{n}{
      <5> <6> <7> <8> <9> <10>
      <10.95> <12> <14.4> <17.28> <20.74> <24.88>
      mathx10
      }{}
\DeclareSymbolFont{mathx}{U}{mathx}{m}{n}
\DeclareFontSubstitution{U}{mathx}{m}{n}
\DeclareMathAccent{\widecheck}{0}{mathx}{"71}
\DeclareMathAccent{\wideparen}{0}{mathx}{"75}

\newcommand{\ignore}[1]{}

\DeclareMathOperator{\BigOm}{\mathcal{O}}

\newcommand{\BigOh}[1]{\BigOm\left({#1}\right)}
\DeclareMathOperator{\BigOmtil}{\widetilde{\mathcal{O}}}
\newcommand{\BigOhTil}[1]{\BigOmtil\left({#1}\right)}

\newcommand{\iidsim}{\overset{\mathrm{i.i.d}}{\sim}}

	\theoremstyle{plain}
	\newtheorem{theorem}{Theorem}

	\newtheorem{lemma}{Lemma}[section]
	\newtheorem{claim}[lemma]{Claim}
	\newtheorem{corollary}{Corollary}[section]
	\newtheorem{proposition}[lemma]{Proposition}

	\theoremstyle{definition}

	\newtheorem{definition}{Definition}[section]
	\newtheorem{example}{Example}[section]

	\newtheorem{remark}{Remark}[section]
    \newtheorem{observation}[lemma]{Observation}

  \newtheorem{fact}{Fact}[section]	
  \newtheorem{assumption}{Assumption}[section]
  \newtheorem{condition}{Condition}[section]

\makeatletter
\newcommand{\neutralize}[1]{\expandafter\let\csname c@#1\endcsname\count@}
\makeatother

\newtheorem*{theorem*}{Theorem}
\newtheorem*{lemma*}{Lemma}
\newtheorem*{corollary*}{Corollary}
\newtheorem*{proposition*}{Proposition}
\newtheorem*{claim*}{Claim}
\newtheorem*{fact*}{Fact}
\newtheorem*{observation*}{Observation}

\newtheorem*{definition*}{Definition}
\newtheorem*{remark*}{Remark}
\newtheorem*{example*}{Example}

\newtheoremstyle{named}{}{}{\itshape}{}{\bfseries}{}{.5em}{\Cref{#3} {\normalfont (informal)} }
{}
\theoremstyle{named}

\theoremstyle{plain}

\DeclareMathAlphabet{\mathbfsf}{\encodingdefault}{\sfdefault}{bx}{n}

\DeclareMathOperator*{\argmin}{arg\,min}

\DeclareMathOperator*{\conv}{conv}

\let\Pr\relax
\DeclareMathOperator{\Pr}{\mathbb{P}}

\newcommand{\lrbra}[1]{\!\left[#1\right]\!}

\newcommand{\ceil}[1]{\lceil #1 \rceil}

\newcommand{\E}{\mathbb{E}}
\newcommand{\EE}[1]{\E\lrbra{#1}}

\newcommand{\eps}{\varepsilon}

\renewcommand{\leq}{~\le~}
\renewcommand{\geq}{~\ge~}

\let\oldtfrac\tfrac
\renewcommand{\tfrac}[2]{\smash{\oldtfrac{#1}{#2}}}

\let\nablaold\nabla
\renewcommand{\nabla}{\nablaold\mkern-2.5mu}

\newcommand{\Exp}{\mathbb{E}}

\newcommand{\Z}{\mathbb{Z}}

\newcommand{\I}{\mathbb{I}}

\DeclareMathOperator{\Var}{{\rm Var}}                        
\newcommand*{\zero}{{\bm 0}}

\newcommand{\veps}{\varepsilon}

\newcommand{\Hsum}{\cH_{\mathtt{sum}}}
\newcommand{\Fbet}{\cF_{\upbeta}}
\newcommand{\hycon}{hypercontractivity}
\newcommand{\Hycon}{Hypercontractivity}
\newcommand{\radbar}{\overline{\rad}}

\newcommand{\delnpzero}{\updelta_{n,\Pzero}}

\newcommand{\Radenp}{\Rade_{n,\Pr}}
\newcommand{\Radenpzero}{\Rade_{n,\Pzero}}

\newcommand{\starhull}{\mathrm{star}}

\newcommand{\Efn}{\widehat{\cR}^{\mathtt{crs}}_{n}}

\newcommand{\covnum}{\scrN}
\newcommand{\mnum}{\scrM}
\newcommand{\Rnalphsig}{\widehat{\cR}^{\cF}_{n,\alpha,\sigma}}
\newcommand{\Rnsigf}{\hat{\cR}^{\cF}_{n,\sigma}}

\newcommand{\Ezero}{\Exp_{\envzero}}

\newcommand{\ccf}{\cC_{\mathtt{loc}}}

\newcommand{\fst}{f_{\star}}
\newcommand{\gst}{g_{\star}}
\newcommand{\hst}{h_{\star}}
\newcommand{\xivar}{\boldsymbol{\xi}}
\newcommand{\epsvar}{\boldsymbol{\veps}}
\newcommand{\Vzero}{\Var_{\envzero}}

\newcommand{\Eenv}{\Exp_{\env}}
\newcommand{\xvar}{\bx}
\newcommand{\yvar}{\by}
\newcommand{\zvar}{\bz}
\newcommand{\Lossn}{\hat{\cL}_n}
\newcommand{\Rade}{\mathscr{R}}
\newcommand{\Gauss}{\mathscr{G}}

\newcommand{\Raden}{\Rade_n}

\newcommand{\Rfn}{\hat{\cR}_{\mathtt{loc},n}}

\newcommand{\rcrossn}{\updelta_{n,\mathtt{crs}}}

\newcommand{\fhatn}{\hat{f}_n}
\newcommand{\ghatn}{\hat{g}_n}

\newcommand{\delnst}{\updelta_{n}}
\newcommand{\rhonst}{\uprho_{n}}

\newcommand{\Fcent}{\cF_{\mathtt{cnt}}}
\newcommand{\Gcent}{\cG_{\mathtt{cnt}}}
 \newcommand{\entclass}{\mathsf{EntFam}}

\newcommand{\rad}{\mathsf{rad}}
\newcommand{\Dudfun}{\mathscr{D}}

\newcommand{\Ptest}{\Pr_{\mathrm{test}}}
\newcommand{\Ptrain}{\Pr_{\mathrm{train}}}
\newcommand{\Etest}{\Exp_{\mathrm{test}}}
\newcommand{\Etrain}{\Exp_{\mathrm{train}}}
\newcommand{\Rtest}{\cR_{\mathrm{test}}}
\newcommand{\Rtrain}{\cR_{\mathrm{train}}}

\newcommand{\fhat}{\hat{f}}
\newcommand{\ghat}{\hat{g}}

\newcommand{\Rzero}{\mathcal{R}_{\envzero}}

\newcommand{\env}{\mathsf{e}}

\newcommand{\Renv}{\cR_{\env}}

\newcommand{\Pzero}{\Pr_{\envzero}}
\renewcommand{\cY}{\mathcal{Y}}
\newcommand{\Prenv}{\Pr_{\env}}

\usepackage{times}

\title{Statistical Learning under Heterogeneous Distribution Shift}

\author{Max Simchowitz\footnote{\href{mailto:msimchow@csail.mit.edu}{msimchow@csail.mit.edu}, equal contributor.} \\ MIT \and Anurag Ajay\footnote{\href{mailto:aajay@mit.edu}{aajay@mit.edu}, equal contributor.} \\ MIT \and Pulkit Agrawal\footnote{\href{mailto:pulkitag@mit.edu}{pulkitag@mit.edu}} \\ MIT \and Akshay Krishamurthy\footnote{\href{mailto:akshaykr@microsoft.com}{akshaykr@microsoft.com}} \\ MSR NY}

\date{\today}

\begin{document}

\maketitle
\begin{abstract}
This paper studies the prediction of a target $\mathbf{z}$ from a pair of random variables $(\mathbf{x},\mathbf{y})$, where the ground-truth predictor is additive $\mathbb{E}[\mathbf{z} \mid \mathbf{x},\mathbf{y}] = f_\star(\mathbf{x}) +g_{\star}(\mathbf{y})$. We study the performance of empirical risk minimization (ERM) over functions $f+g$, $f \in \mathcal{F}$ and $g \in \mathcal{G}$, fit on a given training distribution, but evaluated on a test distribution which exhibits covariate shift. We show that, when the class $\mathcal{F}$ is 
``simpler" than $\mathcal{G}$ (measured, e.g., in terms of its metric entropy), our predictor is more resilient to \emph{heterogeneous covariate shifts} in which the shift in $\mathbf{x}$ is much greater than that in $\mathbf{y}$. \msedit{Our analysis proceeds by demonstrating that ERM behaves \emph{qualitatively similarly 
to orthogonal machine learning}: the rate at which ERM recovers the $f$-component of the predictor has only a lower-order dependence on the complexity of the class $\mathcal{G}$, adjusted for partial non-indentifiability introduced by the additive structure. }
These results rely on a novel H\"older style inequality for the Dudley integral which may be of independent interest. Moreover, we corroborate our theoretical findings with experiments demonstrating improved resilience to shifts in ``simpler'' features across numerous domains.


\end{abstract}


\section{Introduction}
Modern machine learning systems are routinely deployed under distribution shift \citep{taori2020measuring,koh2021wilds}. However, statistical learning theory has primarily focused on studying the generalization error in the situation where the test and the training distributions are identical \citep{bartlett2002rademacher,vapnik2006estimation}. 
In the setting of \emph{covariate shift}---where only features/covariates change between training and testing, but the target function remains fixed---guarantees from statistical learning theory can be applied
via a reweighting argument, leading to the classical bound involving density ratios depicted in \Cref{eq:bound_naive}. 
However, this approach may be overly pessimistic and may not account for the relative differences in performance degradation between different distribution shift settings.

Well-specified linear regression is perhaps the simplest setting that
admits favorable distribution shift behavior \citep{lei2021near}. Here,
out-of-distribution generalization is controlled by the alignment
between the second moment matrices of the training and test
distribution, rather than the significantly worse density
ratios. {Beyond the linear setting, ML models including neural
  networks often suffer from \emph{spurious correlation}~\citep[c.f.,][]{arjovsky2019invariant}, where the model
  exploits correlations in the training distribution to learn an
  accurate-but-incorrect predictor that fails to generalize to a
  de-correlated distribution. 
Though this phenomenon and other
  related ones  are well-documented
  experimentally, a general theory of distribution
  shift---particularly one that explains the behavior of deep learning
  models in practice---has remained undeveloped.} 

  A useful theory of distribution shift should make {predictions} as to which shifts a learned model is most sensitive to in possible test environments, given properties of the model {which can be evaluated from training data} \citep{xiao2020noise,koh2021wilds,rahimian19dro}. We illustrate this point with the following example.
\begin{example}\label{exmp:quadruped}
Consider a quadruped carrying different payloads across multiple terrains, with a policy trained via reinforcement learning. Should one expect more degradation in performance with new shapes or sizes of payloads? Or should a policy suffer more from novel terrains? If our policy requires camera inputs, should we expect that changes in lighting conditions or times of day have more of an effect? Or if the policy relies on tactile sensation, should we expect changes in weather (e.g., rain on the tactile sensors) to present more of an obstacle?
\end{example}

\msedit{Any theory that attemps to quantify ``difficulty'' of covariate shifts in different features should further acount for how algorithmic decisions affect out-of-distribution performance. Notably, it is known to be challenging in general to outperform pure supervised learning on out-of-distribution benchmarks \cite{koh2021wilds}. Why might pure supervised learning perform \emph{better than expected} under covariate shift, relative to alternatives that attempt to explicitly guard against said shift?   }

\icmlpar{Contributions. } This paper gestures towards a richer theory of generalization
under covariate shift; one that makes such actionable predictions about the relative resilience of a model to the 
kinds of multifarious shifts illustrated in \Cref{exmp:quadruped}.  More specifically, we highlight a setting we call \emph{heterogeneous covariate shift}, where the distribution of one feature shifts more than another.

Theoretically, we study supervised prediction from a pair of (possibly non-independent) random variables $(\bx,\by)$. We think of $\bx$ as corresponding to ``simple features'' and $\by$ to more complex ones. We show greater resilience to heterogeneous distribution shifts in which
the shift in the marginal of $\by$ is significantly smaller than that
of the joint distribution. Specifically, our analysis restricts its attention to regression functions which decompose additively as $f(\bx) + g(\by)$. We show that empirical risk minimization
(ERM) over functions of the form $f(\bx) + g(\by)$ leads to much more
favorable generalization guarantees than those obtained via the
na\"{i}ve covariate shift bound.  {In the most favorable
  setting, we obtain a test error bound that scales only with the
  covariate shift in the marginal of the ``complex feature'' $\by$, so that even though
  spurious correlations between $\bx$ and $\by$ are present, they play no role in the
  generalization performance of the ERM. 
 }

While limited, the  additive framework proposes a useful metric to evaluate relative complexity of the features:  the richness of their associated function classes. This suggests a more general hypothesis that can be formulated \emph{without the additivity assumption:} we can determine resilience to shifts in a given feature by evaluating the ``complexity'' of a model's dependence on that feature. Using in-distribution generalization as a proxy for model complexity,  we find that deep
learning models are \emph{consistenly more resilient to shifts in
simpler features than they are to shifts in complex features}; this finding holds across a range of tasks, including synthetic settings,
computer vision benchmarks, and imitation learning. \msedit{We believe that this adaptivity of empirical risk minimization to may explain why pure supervised learning may be so hard to outperform for distribution shift resilience \cite{koh2021wilds}.}
 We hope that, taken together, our theoretical and experimental results
initiate a further dialogue between the field of statistical learning
theory and the study of distribution shift in machine learning more
broadly.

\icmlpar{Proof Techniques.} 
 The technical challenge to obtaining favorable distribution
  shift is correlation between $\bx$ and $\by$, which, among other things, 
leads to
  unidentifiability of the generalizing predictor. We show that when
  $\cG$ is sufficiently expressive, 
  the simple predictor $f$ can be learned, up to a bias arising from
  identifiability, at a rate that exhibits a lower order dependence on
  $\cG$.  Although this predictor \emph{is} affected by distribution shifts
  in $\bx$, the low complexity of the function class $\cF$ and the
  lower order dependence on $\cG$ implies that the impact on the overall performance 
is
  rather small. Then ERM can learn a $g$ that corrects for the bias in
  $f$ and is unaffected by distribution shifts in $\bx$. The core
  technical result for this argument is the generalization bound for
  $\cF$ { which disentangles the correlations between $\bx$ and $\by$; this result relies, among other things, on a novel H\"{o}lder-style inequality for the
  Dudley integral of products of function classes, which may be of
  independent interest. }



\icmlpar{Related Work.} 
Our results and techniques are very much in the spirit of classical statistical learning theory \citep{bartlett2005local,bousquet2002stability,bartlett2002rademacher,vapnik2006estimation}, but also have the flavor of more recent work on orthogonal/double machine learning \citep{chernozhukov2017double,foster2019orthogonal,mackey2018orthogonal}. In that parlance, we can view $g$ as a nuisance parameter for estimating $f$ and our results show similar (but not quite matching) recovery guarantees without explicit double-training interventions. {We discuss comparisons to orthogonal ML in the sequel.}
\iftoggle{icml}
{}{
  
}
Resilience to distribution shift has received considerable attention in recent years \citep{miller2021accuracy,taori2020measuring,santurkar2020breeds,koh2021wilds,zhou2022domain}, with the vast majority of the work being empirical. While the present work focuses on studying vanilla empirical risk minimization, there have been many methods produced to explicity tackle distribution shift including \textsc{coral} \citep{sun2016deep}, \textsc{irm} \cite{arjovsky2019invariant}, and distributionally robust optimization, the latter having seen recent advances on both empirical and theoretical fronts~\citep{schmidt2018adversarially,rahimian19dro,sinha18dro}.
\iftoggle{icml}
{}{
  
}
Though the statistical properties of distribution shift under empirical risk minimization has garnered substantially less attention, recent work has given precise characterizations of the effects of covariate shift for certain specific function classes, notably kernels \citep{ma2022optimally} and H\"older smooth classes \citep{pathak2022new}. Our work complements these by considering structural situations in which interesting generalization phenomena arise for arbitrary function classes. Lastly, \cite{dong2023first} establish Laplacian-like connectivity conditions under which test-error of additive predictors $f(\bx) + g(\bx)$ (as in this work) can be bounded in terms of train-error, focusing on (a) situations where the marginals over $\bx,\by$ between test- and train-distibutions coincide but joint distributions differ and (b) discrete-  Gaussian-distributed features. By contrast, our work allows for changes in both joint and marginal distributions (albeit with cruder measures of shift), general feature distributions, and exposes statistical phenomena not addressed by the former work.


\newcommand{\envtrain}{\mathsf{train}}
\newcommand{\envtest}{\mathsf{test}}
\renewcommand{\Etrain}{\Exp_{\envtrain}}
\renewcommand{\Etest}{\Exp_{\envtest}}
\renewcommand{\Ptrain}{\Pr_{\envtrain}}
\renewcommand{\Rzero}{\mathcal{R}_{\envtrain}}
\renewcommand{\Ezero}{\Exp_{\envtrain}}
\renewcommand{\Vzero}{\Var_{\envtrain}}
\renewcommand{\Pzero}{\Ptrain}
\renewcommand{\Rtest}{\mathcal{R}_{\envtest}}
\vspace{-.5em}
\section{Theoretical Setup}
\vspace{-.5em}
We study the prediction of a scalar $\zvar \in \R$ from two covariates $\xvar \in \cX, \yvar \in \cY$ under distribution shift. We postulate a pair of testing and training environments denoted $\env \in \{\envtest,\envtrain\}$, each of which index laws $\Prenv$ over $(\xvar,\yvar,\zvar)$, and whose expectation operators are denoted by $\Eenv$. 
We assume the environments do not differ in the Bayes regression function, i.e., they exhibit only \emph{covariate shift}:
\begin{assumption}[Covariate Shift]\label{asm:cov_shift} We assume that, for all $\xvar,\yvar$, $\Etrain[\zvar \mid \xvar,\yvar] = \Etest[\zvar \mid \xvar,\yvar]$.
\end{assumption}
Next, we assume that the we have access to a class of functions that
capture the conditional expectations $\Etrain[\zvar \mid \xvar,\yvar]$
via additive structure.  Specifically, we assume access to classes
$\cF: \cX \to \R$ and $\cG: \cY \to \R$ for which $(x,y) \mapsto
\Etrain[\zvar \mid \xvar=x,\yvar=x] \in \cF + \cG$.  This is typically
referred to as being realizable or well-specified.
\begin{assumption}[Additive well-specification]\label{asm:well_spec} For some $\fst \in \cF$ and $\gst \in \cG$, it holds thats
\begin{align}
 \Ptrain[\bz \mid \bx = x, \by=x] \sim \cN(\fst(x) + \gst(y),\sigma^2) \label{eq:additive_well_spec}
\end{align}
\end{assumption}
Via universality of Gaussian processes, our results can be extended to general subgaussian noise. 
Since the model is well-specified, a natural performance measure of a predictor $(f,g)$ is its excess square-loss risk, denoted $\Renv(f,g)$:
\iftoggle{arxiv}
{
\begin{align*}
\Renv(f,g) &:= \Exp_{\env}(f(\bx)+g(\yvar) - \zvar)^2 - \inf_{  f' \in \cF,  g' \in \cG}\Exp_{\env}( f'(\xvar)+ g'(\yvar) - \zvar)^2 \\
&= \Exp_{\env}((f-\fst)(\xvar) + (g-\gst)(\yvar))^2.
\end{align*}
}
{
   \begin{align*}
\Renv(f,g) &:= \Exp_{\env}((f-\fst)(\xvar) + (g-\gst)(\yvar))^2.
\end{align*} 
}

\paragraph{Empirical Risk Minimization. } We study the excess risk under $\Ptest$ of square-loss empirical risk minimizers, or ERMs, for $\Ptrain$.  Given a number $n \in \N$, we collect $(\xvar_i,\yvar_i,\zvar_i)_{i\in [n]} \iidsim \Ptrain$ samples and let $(\hat f_n,\hat g_n)$ denote (any) empirical risk minimizer of the samples: 
\iftoggle{arxiv}
{
    \begin{align}
(\hat f_n,\hat g_n) \in \argmin_{(f,g)\in \cF \times \cG} \Lossn(f,g), \quad  \Lossn(f,g) := \frac{1}{n}\sum_{i=1}^n (f(\xvar_i)+g(\yvar_i)-\zvar_i)^2. \label{eq:erm}
\end{align}
}
{$(\hat f_n,\hat g_n) \in \argmin_{(f,g)\in \cF \times \cG} \Lossn(f,g)$, where
\begin{align}
 \textstyle\Lossn(f,g) := \frac{1}{n}\sum_{i=1}^n (f(\xvar_i)+g(\yvar_i)-\zvar_i)^2. \label{eq:erm}
\end{align}

}
\paragraph{Distribution Shift.} 
Although we have samples from $\Ptrain$, we are primarily interested in the excess square loss under $\Ptest$. 
For simplicity, the body of this paper focuses on when the density ratios between these distributions are upper bounded; as discussed in \Cref{sec:refined_measures}, these conditions can be weakened considerably. 
\iftoggle{icml}
{}
{We introduce the density ratio coefficients for the joint distribution $(\xvar,\yvar)$ as well as the marginal distribution over $\yvar$.}
\begin{definition}\label{def:nux} 
Define the \emph{density ratio coefficients} $\nu_{x,y},\nu_y \geq 1$ to be the smallest scalars such that
 for all measurable sets $A \subset \cX \times \cY$ and $B \subset \cY$,
\iftoggle{arxiv}
{
    \begin{align*}
\Ptest[(\xvar,\yvar) \in A] \le \nu_{x,y}\Ptrain[(\xvar,\yvar) \in A], \quad \Ptest[\yvar \in B] \le \nu_{y}\Ptrain[\yvar \in B].
\end{align*}
}
{
    \begin{align*}
\Ptest[(\xvar,\yvar) \in A] &\le \nu_{x,y}\Ptrain[(\xvar,\yvar) \in A]\\
\Ptest[\yvar \in B] &\le \nu_{y}\Ptrain[\yvar \in B].
\end{align*}
}
\end{definition}
The interesting regime is where $\nu_{x,y},\nu_y$ are finite.
A standard covariate shift argument upper bounds the excess risk on $\Ptest$ by the joint density ratio, $\nu_{x,y}$, times the excess risk on $\Ptrain$. 
Our aim is to show that much better bounds are possible. Specifically, if the class $\cF$ is ``smaller'' than the class $\cG$, then the excess risk on $\Ptest$ is \emph{less sensitive} to shifts in the joint distribution (i.e., $\nu_{x,y}$) than it is to shifts in the $\yvar$-marginal (i.e., $\nu_y$).
Such an improvement is most interesting in the regime where $\nu_{x,y} \gg \nu_y$, which requires that $\xvar$ is not a measurable function of $\yvar$. 

Controlling distribution shift via bounded density ratios is popular in the offline reinforcement learning, where such terms are called \emph{concentrability coefficients} \citep{xie2020q,xie2022role}. We stress that the uniform density ratio bounds in this section are merely for convenience; we discuss generalizations at length in \Cref{sec:refined_measures}.



\paragraph{Conditional Completeness.}
Notice that $(\fst,\gst)$ may not be identifiable in the model \Cref{eq:additive_well_spec}. 
The most glaring counterexample occurs when $\xvar = \yvar$, and $\fst + \gst \in \cF \cap \cG$. 
Then, $(f,g) = (\fst + \gst,\zero)$ and $(f,g) = (\zero,\fst+\gst)$ are both optimal pairs of predictors. 
However, this setting is uninteresting for our purposes, since $\xvar=\yvar$ implies that $\nu_{x,y} = \nu_y$. 
On the other hand, when $\xvar$ and $\yvar$ are independent, the model is identifiable up to a constant offset, i.e., $(\fst+c,\gst-c)$ is an optimal pair.
This line of reasoning suggests that the indentifiable part of $\fst$ in \Cref{eq:additive_well_spec} corresponds to the part of $\xvar$ that is orthogonal to $\yvar$. 
To capture this effect, we introduce the conditional \emph{bias} of $f$ given $\yvar$ under the \emph{training distribution}:
\begin{align}
\upbeta_{f}(\cdot) = \Etrain[(f-\fst)(\xvar) \mid \yvar = \cdot]. \label{eq:bet_f}
\end{align}
Note that this is a function of $y$, not $x$. 
One can check that $\Rtrain(f,g) = 0$ if and only if $(f(\xvar),g(\yvar)) = (\fst(\xvar) + \upbeta_f(\yvar), \gst(\yvar) - \upbeta_f(\yvar))$ with probability one over $(\xvar,\yvar) \sim \Ptrain$. Note, in particular, that this requires $\upbeta_f$ is almost surely (under $\Ptrain$) equal to a measurable function of $\bx$. 
This allows, for example, $(f,g) = (\fst-c,\gst +c)$ for constants $c \in \R$, and, in particular, $(\fst,\gst)$ meet these requirements since $\upbeta_{\fst} = 0$.

We now introduce our final, and arguably only non-standard, assumption. 
\begin{assumption}[$\gamma$-Conditional Completeness]\label{asm:conditional_completeness} There exists some $\gamma > 0$ such that, for any $(f,g) \in \cF \times \cG$ satisfying $\Rtrain(f,g) \le \gamma^2$, it holds that $g - \upbeta_f \in \cG$.
\end{assumption}
Conditional completeness is somewhat non-intuitive but it is satisfied in some natural cases. 
We list them here informally, and defer formal exposition to \Cref{app:on_conditiona_completeness}.
First, as aluded to above, when $\xvar \perp \yvar$, $\upbeta_f(\yvar)$ is constant in $\yvar$ and so 
conditional completeness holds as long as $\cG$ is closed under affine translation.
Second, it holds when $\cF$ and $\cG$ are linear classes and $\xvar$ and $\yvar$ are jointly Gaussian; this follows since the conditional distribution $\Etrain[\xvar \mid \yvar=y]$ is linear in $y$.
The latter example extends to nonparametric settings: conditional completeness holds if the conditional expectations $\xvar \mid \yvar$ are smooth and $\cG$ contains correspondingly smooth functions. 

The restriction to $\Rtrain(f,g) \le \gamma^2$ allows us to make the assumption compatible with the following, standard boundedness assumption (for otherwise we would need to have $g - k \upbeta_f \in \cG$ for all $k \in \N$, see~\Cref{rem:boundedness_cc}.)
\begin{assumption}[Boundedness]\label{asm:bounded} We assume that for all $f \in \cF$ and $g \in \cG$, $|f(\xvar)|$ and $|g(\yvar)|$ are uniformly bounded by some $B > 0$.
For simplicity, we also assume $\cF$ and $\cG$ contain the zero predictor.
\end{assumption}

\icmlpar{Notation.} We use $a \lesssim b$ to denote inequality up to universal constants, and use $\BigOh{\cdot}$ and $\BigOhTil{\cdot}$ as informal notation suppressing problem-dependent constants and logarithmic factors, respectively. A scalar-valued random variable is standard normal if $Z \sim \cN(0,1)$ and Rademacher if $Z$ is uniform on $\{-1,1\}$. For $v  = (v_1,\dots,v_n) \in \R^n$ and $q \in [1,\infty)$, define the normalized $q$-norms  $\|v\|_{q,n} = (\frac{1}{n}\sum_{i=1}^n |v_i|^q)^{1/q}$ and $\|v\|_{\infty,n} = \|v\|_{\infty} = \max_{ i \in [n]}|v_i|$. We let $\cW = \cX \times \cY$ with elements $w \in \cW$, so we can view classes $f \in \cF, g \in \cG$, and $\upbeta_{f}$ as mappings with type $\cW \to \R$. Given $h \in \cH$ and a sequence $w_{1:n} \in \cW^n$, define the evaluation vector $h[w_{1:n}] := (h(w_1),\dots,h(w_n)) \in \R^n$ and evaluated class $\cH[w_{1:n}] := \{h[w_{1:n}]: h \in \cH\} \subset \R^n$. 

\newcommand{\rate}{\mathrm{rate}}

\newcommand{\rateqn}[1][q]{\rate_{n,#1}}
\newcommand{\ratestn}{\rate_{n,\star}}

\section{Results}
All of our results follow from the same schematic: we argue that if
$\cF$ is simpler than $\cG$, it is much easier to recover $\fst$ than
it is to recover $\gst$, subject to the identifiability issues
introduced by $\upbeta_f$. 
To express this, we introduce the per-function risks, for $\env \in \{\envtrain,\envtest\}$:
\iftoggle{arxiv}
{
\begin{align*}
\Renv[f] := \Exp_{\env}[(f-\fst - \upbeta_f)^2], \quad \Renv[g; f] := \Eenv[(g-\gst + \upbeta_f)^2]
\end{align*}
}
{
	\begin{align*}
\Renv[f] &:= \Exp_{\env}[(f-\fst - \upbeta_f)^2].\\
\Renv[g; f] &:= \Eenv[(g-\gst + \upbeta_f)^2].
\end{align*}
}
Our schematic shows that $\Rtrain[\hat{f}_n] \ll \Rtrain[\hat{g}_n;\hat{f}_n]$,  with precise convergence rates.
The expression $\Renv[f]$ reflects that $f$ is identifiable only up to a bias, while $\Renv[g; f]$ can be thought of as the residual error after accounting for the bias in $f$. 
A straightforward consequence of these definitions is the following risk decomposition:
\begin{lemma}\label{lem:excess_decomp} Let $(f,g)\in \cF \times \cG$. Then, under \Cref{asm:cov_shift,asm:well_spec}, $\Rtrain(f,g) = \Rtrain[f] + \Rtrain[g;f]$. Morever, $\Rtest(f,g) \le 2(\Rtest[f] + \Rtest[g;f])$. Therefore,
\iftoggle{arxiv}
{
	\begin{align}
\Rtest(f,g) \le 2(\nu_{x,y}\Rtrain[f] + \nu_y\Rtrain[g;f]) \le 2(\nu_{x,y}\Rtrain[f] + \nu_y\Rtrain(f,g)). \label{eq:extrap_excess}
\end{align}
}
{
	\begin{align}
\Rtest(f,g) \le 2(\nu_{x,y}\Rtrain[f] + \nu_y\Rtrain(f,g)). \label{eq:extrap_excess}
\end{align}
}

\end{lemma}

\Cref{eq:extrap_excess} is the starting point for our results. 
By comparison, the standard distribution shift bound is
\begin{align}
\Rtest(f,g)  \le \nu_{x,y}\Rtrain(f,g). \label{eq:bound_naive}
\end{align}
Hence, \Cref{eq:extrap_excess} 
leads to sharper estimates for ERM in the regime where $\nu_y \ll
\nu_{x,y}$ and $\Rtrain[\hat{f}_n] \ll \Rtest[\hat{f}_n,\hat{g}_n)$,
  i.e., when the shift in $\yvar$ is less than the shift in the joint
  distribution and when the estimate of $\fst$ is more accurate than
  the estimate of $\fst + \gst$.
\iftoggle{arxiv}
{
  
}
{}
The bulk of the analysis involves obtaining sharp bounds on
$\Rtrain[\hat{f}_n]$, this is sketched in 
\Cref{sec:analysis_overview}. In the remainder of this section, we
describe implications for various settings of interest.

\newcommand{\vst}{v^\star}
\newcommand{\nulin}{\nu_{\,\mathrm{lin}}}
\newcommand{\herm}{\mathsf{H}}

\subsection{Nonparametric Rates}
We begin by demonstrating improvements in the \emph{non-parametric regime}, where we measure the complexity of function classes by their metric entropies. 
Recall that an $\epsilon$-\emph{cover} of a set $\bbV$ in a norm $\|\cdot\|$ is a set $\bbV' \subset \bbV$ such that, for any $v \in \bbV$, there exists $v' \in \bbV'$ for which $\|v-v'\| \le \epsilon$. The \emph{covering number} of $\bbV$ at scale $\epsilon$ in norm $\|\cdot\|$ is the minimal cardinality of an $\epsilon$-cover, denoted $\covnum(\bbV,\|\cdot\|,\veps)$. 
Metric entropies of function classes are defined via the logarithm of the covering number.
\begin{definition}[Metric Entropy]\label{defn:metric_entropy} We define the $q$-norm metric entropy of a function class $\cH: \cW \to \R$ as $\cM_q(\epsilon,\cH):= \sup_{n}\sup_{w_{1:n}}\log \covnum(\cH[w_{1:n}],\|\cdot\|_{q,n},\veps)$.
\end{definition}
As in classical results in statistical learning theory, rates of
convergence depend on function class complexity primarily through the
\emph{growth rate} of the metric entropy, i.e., how
$\cM_q(\epsilon,\cH)$ scales as a function of $\epsilon$. We state our
first main result informally, in line with this tradition.

\begin{theorem}[Informal]\label{thm:nonpar} 
Under~\Cref{asm:cov_shift}-\Cref{asm:bounded}, the error of $(\fhatn,\ghatn)$ under $\Ptest$ is bounded as follows with high probability:
\begin{align*}
\Rtest(\fhatn,\ghatn) \lesssim \BigOhTil{\nu_{x,y} \left(\rateqn[2](\cF) + \ratestn(\cG)^2 + \frac{\sigma^2}{n}\right) + \nu_{y} \rateqn[2](\cG) },
\end{align*}
where above we define
\begin{align}
\rateqn(\cH) = \begin{cases} \frac{d}{n} & \cM_{q}(\epsilon,\cH) = \BigOh{d \log(1/\epsilon)}\\
n^{-\frac{2}{2+p}} & \cM_{q}(\epsilon,\cH) = \BigOh{\epsilon^{-p}},~ p \le 2\\
n^{-\frac{1}{p}} & \cM_{q}(\epsilon,\cH) = \BigOh{\epsilon^{-p}}, ~p > 2
\end{cases}, \label{eq:rateqn_informal}
\end{align}
and $\ratestn(\cH) = n^{-(1/2 \wedge 1/p)}$ for $\cM_{\infty}(\epsilon,\cH) = \BigOh{\epsilon^{-p}}$. 
\end{theorem}
A formal statement is given in \Cref{sec:instantiating_the_rates}. As a preliminary point of comparison, the naive analysis would yield a bound of the form
\begin{align*}
  \Rtest(\fhatn,\ghatn) \leq \BigOhTil{\nu_{x,y}(\rateqn[2](\cF) + \rateqn[2](\cG)) + \frac{\sigma^2}{n}}, \tag{naive analysis, covariate shift}
\end{align*}
which can be worse than the above bound when $\nu_y \ll \nu_{x,y}$ and $\ratestn(\cG)^2 \ll \rateqn[2](\cG)$. The rate in \Cref{thm:nonpar} is a consequence of the second result:
\begin{theorem}[Faster recovery of $\fst$ up to bias, informal]\label{thm:double_ml} Adopt the notation of \Cref{thm:nonpar}. With high probability, it holds that
\begin{align}
\Rtrain[\fhatn] = \Etrain[(\fhat - \fst - \upbeta_{\fhatn})^2] \lesssim \BigOhTil{\rateqn[2](\cF) + \ratestn(\cG)^2 + \frac{\sigma^2}{n}},
\end{align}
\end{theorem}
It is crucial to note that the interaction between the complexity of the class $\cG$ and the distribution shift parameter $\nu_{x,y}$ in \Cref{thm:nonpar}, as well as the dependence of the bias-adjusted risk $\Rtrain[\fhatn]$ of $\cG$ in \Cref{thm:double_ml}, scales with the  \emph{squared} convergence rate for $\cG$.

Analogously, naively upper bounding $\Rtrain[\fhatn] \le \Rtrain(f,g)$ would yield
\begin{align}
\Rtrain[\fhatn] = \Etrain[(\fhat - \fst - \upbeta_{\fhatn})^2] \lesssim \BigOhTil{\rateqn[2](\cF) + \rate_{n,2}(\cG) + \frac{\sigma^2}{n}}, \tag{naive analysis, recovery of $\fst$}.
\end{align}
Examining the definition of the $\rate$ functions in \Cref{thm:nonpar}, we see that when the $\ell_2$ and $\ell_{\infty}$ metric entropies of $\cG$ are comparable, we can see that $\rate_{\star,n}(\cG)^2 \ll \rate_{n,2}(\cG)$, and, when bounding the rate function with  exponent $p \ge 2$ in \eqref{eq:rateqn_informal} (above the so-called Donsker threshold), $\rate_{\star,n}(\cG)^2 \sim \rate_{n,2}(\cG)^2$. In these cases, \Cref{thm:nonpar,thm:double_ml} yield substantial improvements of the naive counterparts.

\subsection{Comparison with Orthogonal ML}
\label{sec:orthogonal}
The style of our results is similar to those appearing in the
literature on Neyman orthogonalization (also referred to as
Double/Debiased ML or orthogonal statistical
learning)~\citep[c.f.,][]{chernozhukov2017double,mackey2018orthogonal,foster2019orthogonal}. At
a high level, orthogonal ML considers a situation with an unknown pair
$(\fst,\gst)$, where we are primarily interested in learning $\fst$,
referring to $\gst$ as a nuisance function. We describe two categories of differences: difference in \emph{problem specification} and difference in \emph{statistical rates}.

\paragraph{Differences in problem specification.}
In orthogonal ML, the parameter $\gst$ is truly a nuissance whose confounding effect on $\fst$ is to be removed. In our setting, however, the optimal predictor depends on both $\fst$ and $\gst$ through their sum, and thus $\gst$ cannot be neglected in the prediction.  

Moreover, orthogonal ML leverages an auxiliary supervision
mechanism to learn $\gst$ in order to remove it. In contrast, we reason about the statistical convergence of single-step ERM without access to auxilliary information 

\paragraph{Differences in statistical rates.} In orthogonal ML with ERM, it is shown in \cite{foster2019orthogonal} that the dependence of recovery of $\fst$ on the class $\cG$ scales as 
\begin{align}
\Exp[(\hat f_{\mathrm{OrthogonalML}} - \fst)^2] \lesssim \BigOhTil{\rate_{n,2}(\cF)  + \rate_{n,2}(\cG)^2 + \frac{\sigma^2}{n}}.
\end{align}
 Qualitatively, the rates are similar to those in \Cref{thm:nonpar,thm:double_ml}, with the exception that we replace $\ratestn(\cG)$ with $\rate_{n,2}(\cG)$.  There are two comparative weakness in our bound:
 \begin{itemize}
 	\item[(a)] First, $\ratestn(\cG)$ dependence on the $\ell_{\infty}$ covering numbers of $\cG$, whereas $\rate_{n,2}(\cG)$ depends on the $\ell_2$ covering numbers. 
 	\item[(b)] For $p \le 1/2$ (below the so-called Donsker threshold), $\rate_{n,2}(\cG)^2$ can decay to zero faster than $\cO(1/n)$, leaving the $\sigma^2/n$ term to dominate it. On the other hand, $\ratestn(\cG)^2$ scales as $1/n$ with some constant factor prepended, and thus, can dominate the $\sigma^2/n$ term when this constant factor is large. Similarly, dependence on $\sigma^2$ may differ between the two. We partially address this limitation for finite (and more generally, parameteric) function classes, as discussed in \Cref{sec:finite_fn}. 
 \end{itemize}
The dependence on $\ell_{\infty}$ covering numbers arises from our H\"older Inequality for the Dudley integral, \Cref{prop:dudley_holder}, applied to bounding the cross-interactions between the $\cF$ and $\cG$ classes. 
 The suboptimal  $\ratestn(\cG) = \cO(n^{-1/2})$ for $p\le 1/2$ arises from the same proposition, which incurs a dependence on the unlocalized complexity of the class $\cG$ rather than the localized complexities which determine $\rate_{n,2}$. By comparison, \cite{foster2019orthogonal} use independent data to learn $\gst$ beforehand, and thus  do not need to decorrelate $\fhat$ and $\ghat$ in the same way.  It is an open question if this discrepancy reflects a limitation
  in our analysis or is a fundamental limitation of ERM.

 Aside from the above situations, our rates coincide. We summarize this observation:
 \begin{observation} Let $\sigma^2 \ge 1$ and suppose that the class $\cG$ satisfies $\cM_{\infty}(\epsilon,\cG) \le C\cM_{2}(\epsilon,\cG)$ for all $\epsilon > 0$ and some constant $C$. Then,  for some constant $C'$ depending only on $\cG$ such that 
 \begin{align}
 \ratestn(\cG)^2 + \frac{\sigma^2}{n} \le C'\left(\rate_{n,2}(\cG)^2 + \frac{\sigma^2}{n}\right).
 \end{align} 
 \end{observation}


\paragraph{Orthogonal ML without Orthogonal ML}
Despite its limitations, our bound can be  somewhat more practical than what is found in the orthogonal machine learning literature, as it applies
to ERM directly and does not require algorithmic modifications or an
auxiliary supervision signal. The key difference here is that whereas orthogonal ML aims for \emph{inference} -- consistent recovery of $\fst$ --  we care only about the \emph{prediction error} of $\fst + \gst$. Thus, we need not address the identifiability challenges present in orthogonal ML. As a consequence, we bypass algorithmic
modifications that typically require more precise modeling of the data
generating process, and which typically render orthogonal ML more susceptible to
misspecification issues. Finally, we should note that in canonical
settings for orthogonal learning, we can show that our main
assumption, conditional completeness, holds. In this sense, our work
shows that, in typically settings for orthogonal learning, one can
obtain similar statistical improvements \emph{with ERM alone} and
\emph{without auxiliary supervision}.

Please see~\Cref{app:assumptions} for an even more detailed discussion.

\subsection{Finite Function Classes}\label{sec:finite_fn}
When $\cF$ and $\cG$ are finite function classes with  $ \log|\cF| \le d_1$ and $ \log |\cG| \le d_2$, an application of \Cref{thm:nonpar} gives the rate of $\Rtest(\fhatn,\ghatn) \lesssim \nu_{x,y} \cdot \frac{d_1 + d_2}{n}$, which is precisely what one obtains via naive change of measure arguments. 
Although direct application of \Cref{thm:nonpar} does not yield improvements---precisely because of the lack of localization as discussed above--- we \emph{can} improve upon this bound with an additional \emph{hypercontractivity} assumption, often popular in the statistical learning literature \citep{mendelson2015learning}. We defer formal definitions, a formal theorem statement, and proofs to \Cref{sec:finite_function_classes}; the following informal theorem summarizes our findings.
\begin{theorem}[Informal]\label{thm:finite_class_informal} Under certain hypercontractivity conditions detailed in \Cref{sec:finite_function_classes}, it holds  with high probability that
\begin{align*}
\Rtest(\fhatn,\ghatn) \lesssim \frac{1}{n}\left(\nu_{x,y}d_1 +  \nu_y d_2 + \nu_{x,y} d_2 \cdot \phi_n(d_1,d_2)\right),
\end{align*} 
where $\phi_n(d_1,d_2) = (\frac{d_2}{n})^{c_1} + (\frac{d_1}{d_2})^{c_2}$, for constants $c_1,c_2 > 0$ depending on the hypercontractivity exponents. 
\end{theorem}

When $d_2 \gg d_1$, the bound replaces the dimension term $d_2 \nu_{x,y}$ with  $d_2 \phi(d_1,d_2)\nu_{x,y} + \nu_y d_2$, a strict improvement when $\nu_y \ll \nu_{x,y}$ and $\phi(d_1,d_2) \ll 1$. The above bound can be extended to function classes with ``parametric'' metric entropy  (\Cref{rem:ext_to_par}). In all cases, $\phi_n(d_1,d_2) \ge \frac{d_2}{n}$, which is still weaker than {an idealized version of \Cref{thm:nonpar} where $\rateqn[2](\cG)^2$ replaces $\ratestn(\cG)^2$.}

\subsection{Refined Measures of Distribution Shift}\label{sec:refined_measures}
The decomposition in \Cref{lem:excess_decomp} and all subsequent guarantees  can be refined considerably. First, we can replace uniform bounds on the density ratios (\Cref{def:nux}) with the following function-dependent quantities:
\begin{align}
\nu_1  &:= \sup_{f \in \cF} \frac{\Etest[(f-\fst - \upbeta_f)^2]}{\Etrain[(f-\fst - \upbeta_f)^2]}\label{eq:nu1}\\
\nu_2  &:= \sup_{f \in \cF, g \in \cG} \frac{\Etest[(g-\gst - \upbeta_f)^2]}{\Etrain[(g-\gst - \upbeta_f)^2]} \label{eq:nu2},
\end{align}
\iftoggle{icml}{
  \vspace{-.3em}
}
{
}
\begin{corollary}\label{cor:excess_decomp} Immediately from \Cref{lem:excess_decomp}, it holds that  $\Rtest(f,g) \le 2(\nu_{1}\Rtrain[f] + \nu_2\Rtrain(f,g))$
\end{corollary}
Both \Cref{thm:nonpar} and \Cref{thm:finite_class_informal} continue to hold using $\nu_1$ and $\nu_2$ instead of $\nu_{x,y}$ and $\nu_y$.
Note that $\nu_1 \le \nu_{x,y}$ and $\nu_{2} \le \nu_y$ always, but they can be much smaller as demonstrated by the follow upper bounds on $\nu_1$.
\begin{lemma}\label{lem:nu_x_indep} Suppose $\bx \perp \by$ under $\Ptrain$. Then $\nu_1 \le \nu_{x} := \sup_{A \subset \cX}\Ptest[\bx \in A]/\Ptrain[\bx \in A]$.
\end{lemma} 
\begin{lemma}\label{lem:lin_shift} Assume (a) $\cX$ is a Hilbert space, (b) the functions $f \in \cF$ are linear in $\xvar$ and (c) there are constants $\nulin > 0$ such that, with $\upbeta_x(\yvar) := \Exp_{\envtrain}[\xvar \mid \yvar]$, 
\iftoggle{arxiv}
{
\begin{align*}
\Etest[(\xvar - \upbeta_x(\yvar))(\xvar - \upbeta_x(\yvar))^{\herm}] \preceq \nulin \cdot \Etrain[(\xvar - \upbeta_x(\yvar))(\xvar - \upbeta_x(\yvar))^{\herm}]
\end{align*} 
}
{
	$\Etest[(\xvar - \upbeta_x(\yvar))(\xvar - \upbeta_x(\yvar))^{\herm}] \preceq \nulin \cdot \Etrain[(\xvar - \upbeta_x(\yvar))(\xvar - \upbeta_x(\yvar))^{\herm}]$.
} Then, $\nu_1 \le \nulin$.
\end{lemma}
\Cref{sec:lem:excess_decomp_lin} proves both lemmas.
Importantly, $\nulin$ can be finite even when $\nu_{x,y}$ is infinite, e.g. if the distribution over $\bx$ is discrete under $\Ptrain$, but continuous under $\Ptest$.

\icmlpar{Beyond Uniform Ratios.} \Cref{eq:nu1,eq:nu2} can be generalized further to allow for additive error.

\begin{restatable}{corollary}{coradd}\label{cor:additive_error} Suppose that, for all $f \in \cF$ and $g \in \cG$,
\begin{align*}
\Etest[(f-\fst - \upbeta_f)^2] &\le \nu_1 \Etrain[(f-\fst - \upbeta_f)^2] + \Delta_1 \\
\Etest[(g-\gst - \upbeta_f)^2] &\le \nu_2\Etrain[(g-\gst - \upbeta_f)^2] + \Delta_2 ,
\end{align*}
Then, immediately from \Cref{lem:excess_decomp},  
\begin{align*}\Rtest(f,g) \le 2(\nu_{1}\Rtrain[f] + \nu_2\Rtrain(f,g)+\Delta_1 + \Delta_2).
\end{align*}
\end{restatable}
This deceptively simple modification allows for situations when the density ratios between the test and train distributions are not uniformly bounded, or possibly even infinite. 
\Cref{sec:beyond_uniform_density} details the many consequences of this observation. We highlight a key one here:
\begin{lemma}\label{lem:chi_sq} Suppose
\Cref{asm:bounded}  holds. Then, for $\Rtrain[f],\Rtrain(f,g)$ sufficiently small, 
\begin{align*}
 \Rtest(f,g) &\le 8B\sqrt{\Rtrain[f]\cdot \chi^2(\Ptest(\bx,\by),\Ptrain(\bx,\by))} \\
 &\quad+ 8B \sqrt{\Rtrain(f,g) \cdot \chi^2(\Ptest(\by),\Ptrain(\by))},
\end{align*}
where $ \chi^2(\Ptest(\bx,\by),\Ptrain(\bx,\by))$ denotes the $\chi^2$ divergence (see e.g. \citet[Chapter 2]{polyanskiy2022information}) between the joint distribution of $(\bx,\by)$ under test and train distributions, and $\chi^2(\Ptest(\by),\Ptrain(\by))$ denotes $\chi^2$ divergence restricted to the marginals of $\by$.
\end{lemma}
The above lemma  is qualitatively similar to \Cref{cor:excess_decomp,lem:excess_decomp}: If $\Rtrain[f] \ll \Rtrain(f,g)$ (as ensured by our analysis, under appropriate assumptions), then we ensure more resilience to the $\chi^2$ divergence between the joint distributions of $(\bx,\by)$ than would naively be expected.

\newcommand{\gtil}{\tilde{g}}
\renewcommand{\rhonst}{\updelta_{n,\mathscr{G}}}
\renewcommand{\delnst}{\updelta_{n,\mathscr{R}}}
\newcommand{\delcrossn}{\updelta_{n,\mathrm{cross}}}
\newcommand{\Xiglob}{\Xi_{\mathrm{glob}}}
\newcommand{\Xiloc}{\Xi_{\mathrm{loc}}}
\newcommand{\delcrossnbar}{\bar{\updelta}_{n,\mathrm{cross}}}
\newcommand{\rcrossntil}{\tilde{\updelta}_{n,\mathtt{crs}}}
\newcommand{\deldud}{\updelta_{n,\mathscr{D}}}
\newcommand{\Hcent}{\cH_{\mathtt{cnt}}}
\newcommand{\sigB}{\sigma_B}
\section{Analysis Overview}\label{sec:analysis_overview}
We begin this section with formal precursors to \Cref{thm:nonpar,thm:double_ml} in terms of Dudley integrals \citep{dudley1967sizes}, stated as \Cref{thm:main_guarante,thm:main_guarantee_f}. 
The rest of the section provides an overview of the proof. 
\Cref{sec:learning_prelim} contains the necessary preliminaries, notably Rademacher and Gaussian complexities and their associated critical radii.
\Cref{sec:proof_thm_main_guarantee} provides the roadmap for the proof of \Cref{thm:main_guarante}, focusing on our novel excess risk bound for $\Rtrain[\fhatn]$ in terms of a ``cross critical radius'' term.
We bound this term in \Cref{sec:Holder_ineq} via a H\"older style inequality for Rademacher complexity. 

For convenience, define the \emph{centered classes} 
\iftoggle{arxiv}
{
\begin{align*}
\Fcent&:=
\{f-\upbeta_f - \fst: f \in \cF\}\\
\Gcent &:= \{g - \gst + \upbeta_f:
f \in \cF, g \in \cG\}\\
\Hcent &:= \{f+g -(\fst + \gst):f \in
\cF, g\in \cG\}.
\end{align*}
}
{
    $\Fcent:=
\{f-\upbeta_f - \fst: f \in \cF\}$, $
\Gcent := \{g - \gst + \upbeta_f:
f \in \cF, g \in \cG\}$ and $\Hcent := \{f+g -(\fst + \gst):f \in
\cF, g\in \cG\}$.
}


\paragraph{Formal Main Result.} We define the Dudley functional, a standard measure of statistical complexity.
\begin{restatable}[Dudley Functional]{definition}{defndudley}\label{defn:dudfunc} 
Let $\rad_q(\bbV) := \sup_{v \in \bbV}\|v\|_{q,n}$ be the $q$-norm radius and $\mnum_q(\bbV;\cdot)$ be the metric entropy in the induced $\ell_{q}$ norm (\Cref{defn:entropies}). 
Given $\bbV \subset \R^n$ define \emph{Dudley's chaining functional} (in the $q$-norm) as 
\iftoggle{arxiv}
{ 
\begin{align*}
\Dudfun_{n,q}(\bbV) := \inf_{\updelta \le \rad_q(\bbV)}\left(2\updelta + \frac{4}{\sqrt{n}}\int_{\updelta}^{\rad_q(\bbV)}\sqrt{\mnum_q(\bbV;\veps/2)}\rmd \veps\right).
\end{align*}
}
{
    \begin{align*}
\Dudfun_{n,q}(\bbV) := \inf_{\updelta \le \rad_q(\bbV)}\left(2\updelta + \textstyle\frac{4}{\sqrt{n}}\int_{\updelta}^{\rad_q(\bbV)}\sqrt{\mnum_q(\bbV;\veps/2)}\rmd \veps\right).
\end{align*}
}
Furthermore, given a function class $\cH$ and letting $\cH[r,w_{1:n}]$ denotes the empirically localized class (\Cref{defn:basic_things} below), define \emph{the Dudley critical radius}
\begin{align*}
\deldud(\cH,c) := \iftoggle{arxiv}{}{\textstyle}\inf\left\{r: \sup_{w_{1:n}}\Dudfun_{n,2}(\cH[r,w_{1:n}]) \le \frac{r^2}{2c} \right\},
\end{align*}
\end{restatable}

We now state the formal version of our main results. Calculations in \Cref{sec:instantiating_the_rates} obtain
\Cref{thm:nonpar} and \Cref{thm:double_ml} by bounding the Dudley functionals using standard
statistical learning arguments. First, we state the precursor to \Cref{thm:nonpar}.
\begin{theorem}\label{thm:main_guarante} Suppose \Cref{asm:cov_shift,asm:well_spec,asm:conditional_completeness,asm:bounded} hold. Let $\sigB := \max\{B,\sigma\}$, let $\nu_1,\nu_2$ be as in \Cref{eq:nu1,eq:nu2}, and let $c_1$ be a sufficiently small universal constant. Then  if \Cref{eq:nice_condition} holds, that  probability at least $1 - \delta$, 
    \begin{align*}
    \Rtest(\fhatn,\ghatn) &\lesssim \nu_{1}\left( \deldud(\Fcent,\sigB)^2 + \sup_{w_{1:n}}\Dudfun_{n,\infty}(\Gcent[w_{1:n}])^2 \right) +  \nu_{2}\cdot\deldud(\Hcent,\sigB)^2 \\
    &\quad+ \frac{(\nu_1+\nu_2)\sigB^2\log(1/\delta)}{n}.
    \end{align*}
\end{theorem}
This is derived from the following precursor to \Cref{thm:double_ml}.
\begin{theorem}\label{thm:main_guarantee_f}  Suppose \Cref{asm:well_spec,asm:conditional_completeness,asm:bounded} hold. Let $\sigB := \max\{B,\sigma\}$, and let $c_1 > 0$ be a sufficiently small universal constant. If  $n$ is sufficiently large that
\begin{align}
\deldud(\Hcent,\sigB)^2 + \frac{\sigB^2\log(1/\delta)}{n} \le c_1 \gamma, \label{eq:nice_condition}
\end{align}  
then it holds that probability at least $1 - \delta/2$,
\begin{align*}
    \Rtrain[\fhatn] &\lesssim  \deldud(\Fcent,\sigB)^2 + \sup_{w_{1:n}}\Dudfun_{n,\infty}(\Gcent[w_{1:n}])^2   + \frac{\sigB^2\log(1/\delta)}{n}.
    \end{align*}
\end{theorem}
\Cref{sec:instantiating_the_rates} converts these results into the \Cref{thm:nonpar,thm:double_ml}. The first step is to replace the dependence on centered classes $\Fcent$ and $\Gcent$ with terms depending only on $\cF$ and $\cG$. Then, one computes the Dudley critical radii for classes with bounded metric entropy. 




\subsection{Learning-Theoretic Preliminaries}\label{sec:learning_prelim}
    We state all definitions for a general class of functions $\cH$ mapping $\cW \to \R$. We define two key notions of \emph{localized} and \emph{product classes}. 
    \begin{definition}[Product and Localized Classes]\label{defn:basic_things} Let $\cH,\cH': \cW \to \R$.
    \iftoggle{arxiv}
    {
    \begin{itemize}
    \item We define the empirically localized function class  as  $\cH[r,w_{1:n}] := \{h \in \cH: \frac{1}{n}\sum_{i=1}^n h(w_i)^2 \le r\}$ and \emph{population localized} class $\cH(r) := \{h \in \cH: \Ezero[h^2] \le r\}$.
    \item  We define the \emph{product class} as $\cH \odot \cH' := \{h\cdot h': h \in \cH,h \in \cH'\}$.  
    \end{itemize}
    }
    {
    Define the  \emph{(Hadamard) product class} $\cH \odot \cH' := \{h\cdot h': h \in \cH,h \in \cH'\}$. Given sequence $w_{1:n} \in \cW^n$, define the \emph{empirically localized} class   $\cH[r,w_{1:n}] := \{h \in \cH: \frac{1}{n}\sum_{i=1}^n h(w_i)^2 \le r\}$ and \emph{population localized} class $\cH(r) := \{h \in \cH: \Ezero[h^2] \le r\}$. 
    }
    \end{definition}
    Next, we define the standard Rademacher and Gaussian complexities and associated quantities~\citep[c.f.,][]{rakhlinnotes,wainwright2019high,bartlett2005local}. For convenience, we state these quantities for a set of $n$-length vectors $\bbV \subset \R^n$ and then instantiate the definition to obtain function class variants. 
    \begin{definition}[Rademacher and Gaussian Complexities: Sets]\label{defn:rg_complexities_sets} Let $n \in \N$, and let $\epsvar_{1:n}$ and $\xivar_n$ denote i.i.d. sequences of Rademacher and standard Normal random variables, respectively. The Rademacher and Gaussian complexities of a subset $\bbV \subset \R^n$ are defined as
    \iftoggle{arxiv}
    {
    \begin{align*}
    \Raden(\bbV) := \frac{1}{n}\Exp_{\epsvar}\sup_{v \in \bbV}\sum_{i=1}^n \epsvar_iv_i, \quad \Gauss_n(\bbV) := \frac{1}{n}\Exp_{\xivar}\sup_{v \in \bbV}\sum_{i=1}^n \xivar_iv_i,
    \end{align*}
    }
    {
    $\Raden(\bbV) := \frac{1}{n}\Exp_{\epsvar}\sup_{v \in \bbV}\sum_{i=1}^n \epsvar_iv_i$ and $\Gauss_n(\bbV) := \frac{1}{n}\Exp_{\xivar}\sup_{v \in \bbV}\sum_{i=1}^n \xivar_iv_i$.
    }
    \end{definition}

    Gaussian and Rademacher complexities of function classes can be defined in terms of \Cref{defn:rg_complexities_sets}. For example, we may consider $\Raden(\cH[w_{1:n}])$, or localized variants like $\Raden(\cH[r,w_{1:n}])$. For the latter, we define the \emph{critical radius} quantities, which are central to localization arguments in statistical learning \citep{bartlett2005local}.
    \begin{definition}[Critical Radii]\label{defn:crit_radi} We define the following worst-case \emph{critical radii}:
    \iftoggle{arxiv}
    {
    \begin{align*}
    \delnst(\cH,c) &:= \inf\left\{r: \sup_{w_{1:n}}\Raden(\cH[r,w_{1:n}]) \le \tfrac{r^2}{2c}\right\},  \quad \rhonst(\cH,c) := \inf\left\{r:  \sup_{w_{1:n}}\Gauss_n(\cH[r,w_{1:n}]) \le \tfrac{r^2}{2c}\right\}.
    \end{align*}
    }
    {
        \\
        $\delnst(\cH,c) := \inf\left\{r: \sup_{w_{1:n}}\Raden(\cH[r,w_{1:n}]) \le \tfrac{r^2}{2c}\right\}$, \\
        $\rhonst(\cH,c) := \inf\left\{r:  \sup_{w_{1:n}}\Gauss_n(\cH[r,w_{1:n}]) \le \tfrac{r^2}{2c}\right\}$. 
    }
    \end{definition}
The following lemma verifies that the Rademacher and Gaussian complexities are upper bounded by the Dudley functional (the proof is standard, but see also \Cref{sec:lem:Dudley_ub} for completeness.)
\begin{lemma}\label{lem:Dudley_ub} For any $\bbV \subset \R^n$, we have
\iftoggle{arxiv}
{
\begin{align*}
\Gauss_n(\bbV) \vee \Raden(\bbV) \le \Dudfun_{n,2}(\bbV),
\end{align*}
}
{
    $\Gauss_n(\bbV) \vee \Raden(\bbV) \le \Dudfun_{n,2}(\bbV)$, 
}
and hence, for all $c > 0$, $\delnst(\cH,c) \vee \rhonst(\cH,c) \le \deldud(\cH,c)$.
\end{lemma}

\subsection{Proof Overview of \Cref{thm:main_guarantee_f,thm:main_guarante}.}\label{sec:proof_thm_main_guarantee}
We begin with the following generic upper bound on the joint risk of $\Rtrain[\fhatn,\ghatn]$, 
proved in \Cref{sec:prop:generic_sum_regret}.
\begin{proposition}\label{prop:generic_sum_regret} With probability at least $1 - \delta$, we have that $\Rtrain[\ghatn;\fhatn] \le \Rtrain(\fhatn,\ghatn) \lesssim \gamma_n(\delta)^2$, where we define
\iftoggle{arxiv}
{
    \begin{align*}
 \gamma_n(\delta)^2 := \delnst(\Hcent,B)^2 +  \rhonst(\Hcent,\sigma)^2 + \frac{(B^2 + \sigma^2)\log(1/\delta)}{n}.
\end{align*}
}
{
    $
 \gamma_n(\delta)^2 := \delnst(\Hcent,B)^2 +  \rhonst(\Hcent,\sigma)^2 + \frac{(B^2 + \sigma^2)\log(1/\delta)}{n}$.
}Hence, $\ghatn \in \Gcent(\gamma_n(\delta))$.
\end{proposition}
The localized Gaussian complexity of the class $\Hcent$, $\rhonst(\Hcent,\sigma)^2$, appears in the sharpest analyses of ERM. The dependence on $\delnst(\Hcent,B)^2$ is suboptimal in general (see, e.g. \citet[Chapter 21]{rakhlinnotes}), but is convenient and essentially sharp in the regime where $\sigma^2 \gtrsim B^2$. 
\iftoggle{arxiv}{
    
}{}
However, to take advantage of \Cref{eq:extrap_excess}, we require sharper control over $\Rtrain[\fhatn]$. This involves a novel term, unique to our \iftoggle{arxiv}{additive prediction} setting; the \emph{cross-critical radius}.
    \begin{definition}[Cross Critical Radii]\label{defn:cross_crit_rad} Given the classes $\Fcent$, and another class $\cH$, we  define 
    \begin{align*}
    \delcrossn(\Fcent;\cH) &:= \inf\left\{r: \sup_{w_{1:n}}\Raden(\Fcent[r,w_{1:n}] \odot \cH[w_{1:n}] )\le \frac{r^2}{2}\right\}.
    \end{align*}
    \end{definition}
    The cross-critical radius measures  the complexity of products $\ftil\cdot h$, where $\ftil \in \Fcent$ and $h \in \cH$, and thus captures the extent to which $h$ can obfuscate recovery of $\fst$. We invoke the cross-critical radii with $\cH = \Gcent(\gamma)$, for some $\gamma$.
    It is crucial that the localization $r > 0$ is \emph{only} on the class $\Fcent$ and not on the class $\cH$. With the cross-critical radius in hand, the  following is proved in \Cref{sec:prop:main_reg}.
    \begin{proposition}\label{prop:main_reg} Suppose that $(\cF,\cG)$ satisfy $\gamma$-conditional completeness. Then, whenever $\Rzero[\ghatn;\fhatn] \le \gamma$, the following holds with probability at least $1 - \delta$,
    \begin{align*}
    \Rzero[\fhatn] &\lesssim \delcrossn(\Fcent;\Gcent(\gamma))^2+ \delnst(\Fcent,B)^2 \\
    &\qquad +  \rhonst( \Fcent,\sigma)^2 +  \frac{(\sigma^2 + B^2)\log(1/\delta)}{n}. 
    \end{align*}
    \end{proposition}
   \newcommand{\Gcentil}{\cG - \gst} 
    \newcommand{\Efirst}{\cE_{\text{(Prop.~\ref{prop:generic_sum_regret})}}} 
     \newcommand{\Esecond}{\cE_{\text{(Prop.~\ref{prop:main_reg})}}} 

     The last ingredient is the following lemma which upper bounds the cross-critical radius, and whose proof is deferred to \Cref{sec:Holder_ineq}. 
     \begin{lemma}[Generic Cross-Critical Radius Bound]\label{cor:dun_fun} For any class $\cG$, it holds that  
     \begin{align}
     \delcrossn(\Fcent;\cG)^2 \le \deldud(\Fcent,2B)^2 + 16\sup_{w_{1:n}}\Dudfun_{n,\infty}(\cG[w_{1:n}])^2.
     \end{align}
    \end{lemma} 
    We now formally conclude the proofs of \Cref{thm:main_guarantee_f,thm:main_guarante}. In what follows, let $\Efirst$ denote the event of \Cref{prop:generic_sum_regret}, and $\Esecond$ the event of \Cref{prop:main_reg}.

   \begin{proof}[Proof of \Cref{thm:main_guarantee_f}]
   It suffices to show that on $\Efirst$ and $\Esecond$, the conclude of the theorem holds.
   Upper bounding Rademacher and Gaussian critical radii by the Dudley radius, and using $\sigB = \max\{\sigma,B\}$, we have that on $\Efirst $,
   \begin{align}
   \Rtrain[\ghatn;\fhatn]  \lesssim \Rtrain(\fhatn,\ghatn) \le \deldud(\sigB,\Hcent)^2 + \frac{\sigB^2\log(1/\delta)}{n}, \label{eq:Efirst} 
   \end{align}
   Hence, if 
   \begin{align}
   \deldud(\sigB,\Hcent)^2 + \sigB^2\log(1/\delta)/n \le c_1 \gamma \label{eq:cgam_cond}
   \end{align}
    for a small enough $c_1 >0$, then, on $\Esecond$,
   \begin{align*}
   \Rzero[\fhatn] &\lesssim \delcrossn(\Fcent;\Gcent(\gamma))^2+ \deldud(\Fcent,\sigB)^2 + \frac{\sigB^2\log\frac{1}{\delta}}{n}\\
   &\le \delcrossn(\Fcent;\Gcent)^2+ \deldud(\Fcent,\sigB)^2 + \frac{\sigB^2\log \frac{1}{\delta}}{n}.
   \end{align*} 
   The key step is to now apply \Cref{cor:dun_fun}, stated above, to upper bound  the cross critical radius:
    \begin{align*}\delcrossn(\Fcent;\Gcent)^2 \le \deldud(\Fcent,2B)^2 + 16\sup_{w_{1:n}}\Dudfun_{n,\infty}(\Gcent[w_{1:n}])^2. 
    \end{align*}  
    By \Cref{lem:dudley_star_prop} and the bound $B \le \sigB$, 
    \begin{align*}\deldud(\Fcent,2B)^2 \lesssim \deldud(\Fcent,B)^2 \le \deldud(\Fcent,\sigB)^2.
    \end{align*} 
    Combining the previous three inequalities, if \Cref{eq:cgam_cond} is met and $\Efirst \cap \Esecond$ hold, then
     \begin{align*}
   \Rzero[\fhatn] &\lesssim \deldud(\Fcent,\sigB)^2 + \sup_{w_{1:n}}\Dudfun_{n,\infty}(\Gcent[w_{1:n}])^2 + \frac{\sigB^2\log \frac{1}{\delta}}{n}, \label{eq:Rzero_bound}
   \end{align*} 
   as needed.

   \end{proof}

   \begin{proof}[Proof of \Cref{thm:main_guarante}]
      \Cref{thm:main_guarante} follows readily. 
 From \Cref{cor:excess_decomp},
   \begin{align}
   \Rtest(f,g)  \lesssim \nu_{1}\Rtrain[f] + \nu_2\Rtrain(f,g)  
   \end{align}
   The result now follows from the inequalities \Cref{eq:Efirst} and \Cref{eq:Rzero_bound}, which hold on $\Efirst\cap\Esecond$ and if \Cref{eq:cgam_cond} is met. 
   \end{proof}

\subsection{Controlling Cross Critical Radius via a  H\"older-Inequality for Dudley's integral}\label{sec:Holder_ineq}
Recall \Cref{lem:Dudley_ub}, which restates the well-known fact that the Rademacher and Gaussian complexities of a function class can be upper bounded by Dudley functional defined in \Cref{defn:dudfunc}~\citep[c.f.,][Chapter 5]{dudley1967sizes, wainwright2019high}. We establish a H\"older style generalization of this upper bound. 
In what follows, given $p,q \in [2,\infty]$, we say $(p,q)$ are \emph{square H\"older conjugates} if $(p/2,q/2)$ are regular H\"older conjugates, i.e. $\frac{2}{p} + \frac{2}{q} = 1$. Examples include $(p,q) = (2,\infty)$, $(p,q) = (4,4)$, and $p,q = (\infty,2)$. If $v,u \in \R^n$ are two vectors, then H\"older's inequality implies that for any square H\"older conjugates $p,q$,
\iftoggle{arxiv}
{
\begin{align*}
\|v\odot u\|_{2,n} \le \|v\|_{2,p}\cdot\|u\|_{q,n}
\end{align*}
}
{
    $\|v\odot u\|_{2,n} \le \|v\|_{2,p}\cdot\|u\|_{q,n}$.
}
It may be tempting to generalize \Cref{lem:Dudley_ub} to product classes via
\begin{align}
\Raden(\bbV\odot \bbU) \lesssim \Dudfun_{n,p}(\bbV) \cdot \Dudfun_{n,q}(\bbU)\label{eq:naive_Holder_dudley},
\end{align}
where we recall $\bbV \odot \bbU := \{v \odot u, v \in \bbV, u \in \bbU\}$. Our key result is that \Cref{eq:naive_Holder_dudley} can be sharpened considerably. The following technical result is proved in \Cref{sec:prop:dudley_holder}. 
\begin{proposition}[Dudley Estimate for Hadamard Products (Sets)]\label{prop:dudley_holder} Let $p,q \in [2,\infty]$ satisfy $1/p + 1/q \le 1/2$, and (for simplicitly) suppose $\zero \in \bbV \cap \bbW$. Then,
\begin{align*}
\Raden(\bbV\odot \bbU) \le \rad_q(\bbU)\Dudfun_{n,p}(\bbV) + \rad_p(\bbV)\Dudfun_{n,q}(\bbU). 
\end{align*}
The same bound holds for $\Raden$ replaced by any process defined where the $\epsvar_i$ are $1$-subGaussian variables (e.g. Gaussian complexity $\Gauss_n$). 
\end{proposition}

Notice that rather than having $\Dudfun_{n,p}(\bbV)$ and $\Dudfun_{n,q}(\bbU)$ multiply each other as in \Cref{eq:naive_Holder_dudley}, each is only multiplied by the (H\"older square conjugate) radius term. This is in general considerably sharper, as typically $\rad_p(\bbV) \ll \Dudfun_{n,p}(\bbV)$ unless $\bbV$ is exceedingly small. By taking $\bbU = \{(1,1,\dots,1) \in \R^n\}$, \Cref{prop:dudley_holder} implies the standard Dudley bound, \Cref{lem:Dudley_ub}, as a corollary (see \Cref{sec:lem:Dudley_ub}). We now use the above proposition to upper bound the cross-critical radius. 
\begin{proof}[Proof of \Cref{cor:dun_fun}] Recall the defintion of $\delcrossn$ (\Cref{defn:cross_crit_rad})
\iftoggle{arxiv}
{
\begin{align*}
\delcrossn(\Fcent;\cG) &:= \inf\left\{r: \sup_{w_{1:n}}\Raden(\Fcent[r,w_{1:n}] \odot \cG[w_{1:n}] )\le \frac{r^2}{2}\right\}.
\end{align*}
}
{
}
By {\Cref{prop:dudley_holder}} with $(p,q) = (2,\infty)$,
\begin{align*}
\Raden(\Fcent[r,w_{1:n}] \odot \cG[w_{1:n}] ) &\le \rad_2(\Fcent[r,w_{1:n}])\Dudfun_{n,\infty}(\cG[w_{1:n}]) + \rad_{\infty}(\cG)\Dudfun_{n,2}(\Fcent[r,w_{1:n}])  \\
&\le B\Dudfun_{n,2}(\Fcent[r,w_{1:n}]) +  r\Dudfun_{n,\infty}(\cG[w_{1:n}]),
\end{align*}
where we use that $\rad_2(\Fcent[r,w_{1:n}])\leq r$ by the definition of localization. In particular, if $r$ satisfies
\begin{align*}
\frac{r^2}{4} \ge B\sup_{w_{1:n}}\Dudfun_{n,2}(\Fcent[r,w_{1:n}]), \quad \frac{r}{4} \ge \sup_{w_{1:n}}\Dudfun_{n,\infty}(\cG[w_{1:n}]).
\end{align*}
then $\delcrossn(\Fcent;\cG) \le r$. Thus,
\begin{align*}
\delcrossn(\Fcent;\Gcent) &\le \inf\left\{r: \sup_{w_{1:n}}\Dudfun_{n,2}(\Fcent[r,w_{1:n}]) \le \frac{r^2}{4B} \right\} \vee 4\sup_{w_{1:n}}\Dudfun_{n,\infty}(\cG[w_{1:n}])
\end{align*}
The bound follows by squaring.
\end{proof}
\begin{remark}
Because we consider the $\infty$-norm Dudley integral of $\Gcent$, it is hard to take advantage of localization of $\Gcent$ at $\Gcent(\gamma)$ in the $\cL_2$ norm. The absence of localization leads to the suboptimal dependence on leading constants compared to what is obtained through Double ML \cite{foster2019orthogonal}. \Cref{sec:finite_function_classes} shows that, for finite-function classes, one can take advantage of localization with strong hypercontractivity assumptions.
\end{remark}

\section{Experiments}
\begin{figure*}[t]
    \centering
    \includegraphics[width=\textwidth]{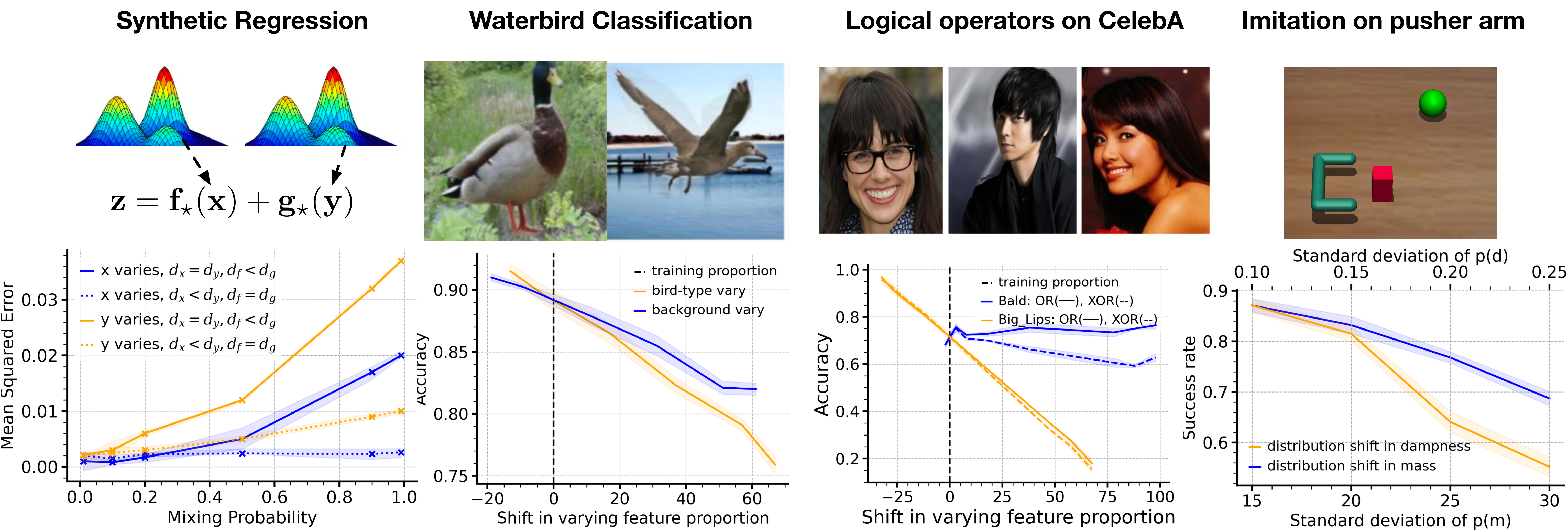}
    \iftoggle{arxiv}{}{\vspace{-1em}}
    \caption{\textbf{Testing resiliency of learned predictive models}. We calculate the performance of learned models as we shift the distribution of either \textit{complex} features (orange color) or \textit{simple} features (blue color). To shift distribution of $\{\bx,\by\}$ in synthetic regression, we vary the mixing probabilities of their distributions. To shift distribution of a feature in Waterbirds and CelebA datsets, we vary its proportion. To shift distribution of a feature in robotic pusher arm control, we increase standard deviation of its sampling distribution. Corroborating our theoretical expectations, we show that these models are more robust to distribution shift in \textit{simple} features.
    }
    \label{fig:comb_fig}
    \iftoggle{arxiv}{}{\vspace{-1em}}
\end{figure*}

In this section we present experiments to validate our
theoretical findings and demonstrate how the conceptual
takeaways---that predictive models are more resilient to distribution
shifts in simple features---applies to a broad range of practical
settings. All of our experiments have a similar form: we (a)
identify simple and complex features and justify these choices and (b) measure how the performance of a predictive model changes with
distribution shifts in these features. We experiment with neural network models on tasks ranging from synthetic
regression problems 
to computer vision benchmarks 
to imitation learning in a robotics
simulator. 
We take the following operational definition of simplicity: 
\begin{quote}\emph{Feature 1 is simpler than feature 2 if the generalization error, without distribution shift, on a predictive task involving feature 1 is smaller than that for an analogous task involving feature 2.}
\end{quote}
 We believe that this is the correct empirical correlate for the complexity measures adopted by our theoretical results. 
  Across domains, we consistently find that
\emph{
predictive models are more resilient to shifts in simpler features, thus defined. } 
%
We now summarize the experimental results, deferring
details to \Cref{app:exp}. In all experiments, we report average
performance and standard error across $4$ replicates.

\icmlpar{Synthetic Regression with Additive Structure.} To closely mirror our theory, we predict $\bz = f_{\star}(\bx) + g_{\star}(\by)$ from an input $(\bx,\by)$, where $\fst:\mathbb{R}^{d_x}\rightarrow\mathbb{R
}$ and $\gst:\mathbb{R}^{d_y}\rightarrow\mathbb{R}$ are randomly initialized $2$-layered multi-layer perceptions (MLP) having hidden dimensions $d_f$ and $d_g$, respectively. We sample $\bx$ and $\by$, respectively, independently from Gaussian mixtures  $p_x = \mathcal{N}(\mathbf{1}_{d_x}, I) + (1-p_x)\mathcal{N}(-\mathbf{1}_{d_x}, \bI_{d_x})$ and $p_y =\mathcal{N}(\mathbf{1}_{d_y}, \bI_{d_y}) + (1-p_y)\mathcal{N}(-\mathbf{1}_{d_y}, \bI_{d_y})$ where $p_x,p_y$ are mixing probabilties, and $\mathbf{1}_{d} \in \R^d$ and $\bI \in \R^{d\times d}$ are the all ones vector and identity matrix. To make $\bx$ simpler, we either have $d_x < d_y$ or $d_f < d_g$: \Cref{sub_app:regression} shows that the auxiliary task of predicting $f_{\star}(\bx)$ has lower generalization error than that of predicting $g_{\star}(\by)$. We train a predictor $\hat{z}=f_\theta(x)+g_\theta(y)$ to minimize mean-square error (MSE) under a training distribution with $p_x = p_y = .01$. We then measure the MSE on shifted distributions, where we hold the mixing probability of one of $\{\bx,\by\}$ fixed, and vary the other in the range $\{0.1, 0.2, 0.5, 0.9, 0.99\}$. Corroborating our theory, MSE declines less with shift in $p_x$ than with shift in $p_y$ (\Cref{fig:comb_fig}). \Cref{sub_app:regression} shows similar results with predictor $\hat{z}=h_\theta(x,y)$ (using concatenated features $(\bx, \by)$) and contains further implementation details. 
\newcommand{\birdtype}{$\mathtt{birdtype}$}
\newcommand{\background}{$\mathtt{background}$}

\icmlpar{Waterbird \& Functional Map of World datasets. } We next test our hypothesis on the paradigmatic Waterbird dataset \citep{sagawa2019distributionally}. The predictive task is to classify images of birds as waterbirds or landbirds, against a background of either land or water. We consider \birdtype{} (resp.  \background)  as the complex (resp. simple) feature. This choice is both intuitive and consistent with our  definition of simplicity: \Cref{sub_app:waterbird} shows that, in the absense of distribution shift, the auxiliary task of predicting \background{} has lower generalization error than that of predicting \birdtype. We then test (\Cref{fig:comb_fig}) the prediction accuracy of \birdtype{} under shifts in the proportions of \background{} and \birdtype{}, finding prediction accuracy degrades less for the former than the latter. See \Cref{sub_app:waterbird} for details. \Cref{sub_app:fmow} applies the same methodology to the Functional Map of World (FMoW) dataset~\citep{koh2021wilds}, with similar findings.

\newcommand{\ORop}{\mathsf{OR}}
\newcommand{\XORop}{\mathsf{XOR}}
\newcommand{\bald}{\mathtt{bald}}
\newcommand{\paleskin}{\mathtt{pale}\_\mathtt{skin}}
\newcommand{\biglips}{\mathtt{big\_lips}}
\newcommand{\narroweyes}{\mathtt{narrow\_eyes}}

\icmlpar{Logical operators on CelebA dataset.}
The CelebA dataset \citep{liu2015deep} consists of celebrity faces labeled with the presence or absence of $40$ different  attributes (e.g., baldness, mustache). Here we re-purpose CelebA to learn logical operators  $\ORop$ and $\XORop$ for two attributes. We first train and test a multi-head binary classifier that detects presence of $40$ different attributes, with one head per attribute, on images from the CelebA ``standard training set'' (CelebA-STS). We select $\bald$ as ``simple'' due to its low generalization error and $\biglips$ as ``complex'' due to its larger generalization error.  In \Cref{fig:comb_fig}, we predict targets $f_{\ORop} = \bald~\ORop~\biglips$ and $f_{\XORop} = \bald ~\XORop~\biglips$, training on CelebA-STS, but testing on distributions where the proportions of $\bald$ and $\biglips$ are varied (\Cref{sub_app:celeba}). Results show greater resilience to shift in the simpler feature $\bald$ than the complex feature $\biglips$. Further details are deferred to \Cref{sub_app:celeba}, where we perform the same experiment for $(\paleskin,\narroweyes)$ with similar findings.

\newcommand{\mass}{\mathtt{mass}}
\newcommand{\dampness}{\mathtt{dampness}}

\icmlpar{Imitation Learning for Pusher Control.} In Pusher Control simulator, we learn an agent that controls a robotic arm to push an object to its goal location. We fix the goal and starting object locations across episodes. We vary object $\mass$ $m$ ($m \sim \mathcal{N}(60,15)$) and joint $\dampness$ $d$ ($d \sim \mathcal{N}(0.5, 0.1)$) of the robot. The agent observes  $\{m,d\}$ at beginning of every episode and aims to learn an optimal policy condition on these variables. To determine the simpler feature, we measure generalization error on the auxiliary task of predicting next-step dynamics where one of $\{m,d\}$ is held fixed and the other is drawn from a distribution that is fixed across training and testing. This methodology ascribes $m$ as the simpler feature and $d$ as complex (\Cref{sub_app:robotic}). We train a policy $\pi_\theta(a|s,m,d)$ using $1000$ expert trajectories where each trajectory contains a new $m$ and $d$ sampled from $\mathcal{N}(60,15)$ and $\mathcal{N}(0.5,0.1)$ 
respectively. We then shift distribution of $m$ and $d$ by increasing their standard deviation, one at a time while keeping the distribution of the other factor fixed, and test policy $\pi_\theta(a|s,m,d)$. We show the results in \Cref{fig:comb_fig} and observe that the success rate of $\pi_\theta(a|s,m,d)$ deteriorates less when we shift distribution of $m$ while keeping distribution of $d$ fixed. Thus, we show that the policy is more resilient to distribution shift in the simpler feature. \Cref{sub_app:robotic} contains further details, including the precise distributions used to determine $m$ as the simpler feature.

\section{Discussion}
This paper sheds new light on the issue of spurious correlation
that arises when considering out of distribution generalization. We
discover that predictive models are more resilient to distribution
shift in simpler features, which we capture via notions of statistical
capacity in our experiments and via generalization when predicting the
feature itself in our experiments. We find that, in most of our
experiments, this latter operational notion is predictive of how deep
learning models behave under heterogeneous distribution shift.
We hope that our work inspires future efforts
toward a fine-grained theoretical and experimental understanding of
distribution shift in modern machine learning.




\section*{Acknowledgements}
MS acknowledges support from Amazon.com Services LLC grant; PO\# 2D-06310236.
AA and PA acknowledge supported from a DARPA Machine Common Sense grant, a MURI grant from the Army Research Office under the Cooperative Agreement Number W911NF-21-1-0097, and an MIT-IBM grant. The authors thank Adam Block for his assistance in navigating the relevant learning theory literature. 
\newpage
\bibliographystyle{plainnat}
\bibliography{refs}
\newpage 
\tableofcontents
\appendix
\newcommand{\lawP}{\mathsf{P}}
\newcommand{\lawQ}{\mathsf{Q}}

\newcommand{\dP}{\rmd \lawP}
\newcommand{\dQ}{\rmd \lawQ}
\newcommand{\Dphi}[1][\phi]{\mathrm{D}_{#1}}
\section{Discussion of Assumptions}
\label{app:assumptions}
In this section we provide additional details about our main
assumptions, conditional completeness. As this is closely related to
orthogonal ML, we also provide some additional context and comparisons. 

\subsection{On Conditional Completeness}\label{app:on_conditiona_completeness}
Our results hinge on the conditional completeness assumption, where
$\cG$ is expressive enough to capture the conditional bias functions
$\beta_f$. To clarify when this assumption may hold, we discuss an
example.

\paragraph{Partially Linear Regression.}
First, let us consider a paradigmatic model in econometrics,
statistics, and causal inference. This model, known as the partial
linear regression (PLR)
model~\citep{chernozhukov2017double,robinson1988root}, specifies a
joint distribution over tuples $(\zvar,\xvar,\yvar)$ via the
structural equations:
\begin{align*}
\zvar &= \langle\xvar,\fst\rangle + h_1(\yvar) + \varepsilon \\
\xvar &= h_0(\yvar) + \tau \\
\yvar & \sim P,  \EE{\varepsilon \mid \xvar,\yvar} = 0, \EE{\tau \mid \yvar} = 0
\end{align*}
Here $\xvar$ belongs to a (finite dimensional) vector space, say
$\R^d$, and $\fst$ is a linear function, so we use $\fst$ both to
describe the mapping and the vector itself. In the learning setting
for PLR, we are given access to function class $\cH_0,\cH_1$ such that
$h_0 \in \cH_0, h_1 \in \cH_1$, and we assume that $\fst$ has bounded
norm, say $\|\fst\| \leq B$.

This model is amenable to our techniques whenever $\cH_1$ consists of
linear projections of $\cH_0$, i.e.,
\begin{align*}
  \cH_0 := \{ \yvar \mapsto \langle v, h_1(\yvar) \rangle: v \in \R^d, h_1 \in \cH_1\}.
\end{align*}
Clearly we can apply ERM with $\cF$ as the linear class and $\cG =
\cH_0$. But we must verify that conditional completeness holds. This follows because, for any $f$, we have
\begin{align*}
  \beta_f(\yvar) := \EE{\langle f - \fst, \xvar\rangle \mid \yvar} = \langle f - \fst, \EE{\xvar\mid\yvar}\rangle = \langle f - \fst, h_0(\yvar)\rangle \in \cG.
\end{align*}
Thus, our results demonstrate favorable guarantees for estimating
$\fst$ via ERM in this setting.


\begin{remark}[Compatibility of conditional completeness and boundedness]\label{rem:boundedness_cc} If conditional-completeness were stipulated as a global condition, i.e. for all $(f,g) \in \cF \times \cG$, $g - \upbeta_f \in \cG$, then necessarily $(f,g - k \upbeta_f) \in \cF \times \cG$ for all $k \in \N$. Therefore, $\cG$ would not in general consist only of functions which are uniformly bounded (except in the special case where $\upbeta_f = 0$ for all $f \in \cF$). By imposing the restriction $\Rtrain(f,g) \le \gamma^2$, we avoid this pathology because, unles $\upbeta_f \equiv 0$,  $\lim_{k \to \infty}\Rtrain(f, g- k\upbeta_f) = \infty$.
\end{remark}

\subsection{Proof of \Cref{lem:lin_shift,lem:nu_x_indep}}\label{sec:lem:excess_decomp_lin}

\begin{proof}[Proof of \Cref{lem:nu_x_indep}]
Suppose that if $\bx \perp \by$ under $\Ptrain$. Then $\upbeta_f(y) =  \Exp[f(\bx) - \fst(\bx) \mid \by = y] = \Exp[f(\bx) - \fst(\bx)]$ is a constant in $y$. 
\begin{align*}
\nu_1  &:= \sup_{f \in \cF} \frac{\Etest[(f-\fst - \upbeta_f)^2]}{\Etrain[(f-\fst - \upbeta_f)^2]}\\
&\le \sup_{f \in \cF, c \in \R} \frac{\Etest[(f-\fst - c)^2]}{\Etrain[(f-\fst - c)^2]}\\
&\le \sup_{A \subset \cX}\frac{\Ptest[\bx \in A]}{\Ptrain[\bx \in A]} = \nu_x,
\end{align*}
as the functions $f-\fst - c$ are functions of $\bx$ alone.
\end{proof}
\begin{proof}[Proof of \Cref{lem:lin_shift}]
 By linearity (assumption (a) of \Cref{lem:lin_shift}, we can write $f(\xvar) = \langle v, \xvar \rangle$ and $\fst(\xvar) = \langle v^\star, \xvar \rangle$ for some $v,v^\star \in \mathbb{H}$. Then, setting $\upbeta_x(\yvar) = \Etrain[\xvar \mid \yvar]$,
\begin{align*}
&\Exp_{\envtest}[(f(\xvar) - \fst(\xvar) - \upbeta_{f}(\xvar))^2] = \Exp_{\envtest}[\langle v - \vst, \xvar - \upbeta_x(\yvar) \rangle^2]\\
&\le \nulin  \Exp_{\envtrain}[\langle v - \vst, \xvar-\upbeta_x(\yvar)\rangle^2 ] \tag{\Cref{lem:lin_shift}, assumption (c)}\\
&= \nulin  \Exp_{\envtrain}[( f(\xvar)-\fst(\xvar) - \underbrace{\Exp_{\envtrain}[f(\xvar) - \fst(\xvar) \mid \yvar]}_{=\upbeta_f(\yvar)})^2 ],
\end{align*}
as needed.
\end{proof}

\subsection{Beyond Uniform Density Bounds}\label{sec:beyond_uniform_density}
In this section, we consider a generalization of  \Cref{eq:nu1,eq:nu2} to allow for additive slack. We show that this allows for generalizations of \Cref{lem:excess_decomp,cor:excess_decomp} which accomodate density ratios which are possibly unbouned, or even take the value $\infty$ with positive probability (\Cref{sec:unbounded_den_rat}). Finally, we show that our guarantees imply guarantees when the $\chi^2$ divergence (or more generally, power divergence) are bounded (\Cref{sec:f_div_stuff}), thereby establishing the proof of \Cref{lem:chi_sq}. We begin by restating \Cref{cor:additive_error}.
\coradd*

Though seemingly simple, the accomodation of additive slack is deceptively flexible. Given probability laws $\lawP,\lawQ$ on a space $\cX = \{X \in \cX\}$, recall their Radon-Nikodym derivative (see, e.g. \citet[Chapter 2]{polyanskiy2022information}) $\dP(X)/\dQ(X)$ (which may take value $\infty$). 
In the language of Radon-Nikodym derivatives, the worst-case density ratios $\nu_{x,y}$ and $\nu_y$ in \Cref{def:nux} can be defined as 
\begin{align}
\nu_{x,y} := \sup_{\bx,\by} \frac{\rmd \Ptest(\bx,\by)}{\rmd \Ptrain(\bx,\by)}, \quad \nu_y := \sup_{\by} \frac{\rmd \Ptest(\by)}{\rmd \Ptrain(\by)}.
\end{align}
The following corollary, whose proof we give in \Cref{sec:deferred_change_of_measure}, shows that we can replace the dependence on the worst-case density ratios in \Cref{lem:excess_decomp} with a bound that depends only on the \emph{tails} of those density ratios:
\begin{corollary}\label{defn:cor_dens_ratio_tail} Recall the boundedness assumption \Cref{asm:bounded}, such that all of $|f|,|g|,|\fst|,|\gst|$ are uniformly at most $B$ in magntinude. For any functions $f \in \cF$ and $g \in \cG$, it holds that 
\begin{align*}
\Rtest(f,g) \le \inf_{t_1, t_2 > 0} 2(t_{1}\Rtrain[f] + t_2\Rtrain(f,g)+16B^2\Delta_{x,y}(t_1) + 16B^2\Delta_{y}(t_2)),
\end{align*}
where we define
\begin{align*}
\Delta_{x,y}(t) &:= \Pr_{(\bx,\by)\sim \Ptest}\left[\frac{\rmd \Ptest(\bx,\by)}{\rmd \Ptrain(\bx,\by)} > t\right], \quad \Delta_{y}(t) := \Pr_{\by \sim \Ptest}\left[\frac{\rmd \Ptest(\by)}{\rmd \Ptrain(\by)} > t\right].
\end{align*}
\end{corollary}
\begin{remark} As in \Cref{sec:refined_measures}, the above bound can be refined in a function-class dependent fashion by replacing the density ratio tail-bounds in the definitions of $\Delta_{x,y}(t)$ and $\Delta_y(t)$ with the restriction of these density ratios to the $\upsigma$-algebra generated by the function classes $\{(\bx,\by) \mapsto f(\bx) - \fst(\bx) - \upbeta_f(\by): f \in \cF\}$ and $\{\by \mapsto g(\by) - \gst(\by) - \upbeta_f(\by):f\in \cF,g \in \cG\}$. When $\nu_1$ and $\nu_2$, as defined in \Cref{eq:nu1,eq:nu2}, are bounded, then taking $t_1 = \nu_1$ and $t_2 = \nu_2$ recovers the bounds obtained in \Cref{sec:refined_measures}.
\end{remark}

\subsubsection{Examples with unbounded density ratios.}\label{sec:unbounded_den_rat}
We give two simple examples demonstrating how the bound in
\Cref{defn:cor_dens_ratio_tail} can be finite even when $\nu_{x,y}$ and $\nu_y$ are not. Our first illustrative example shows that 
\Cref{defn:cor_dens_ratio_tail}  can be finite even if there are values of $\by$ for which the density ratios are infinte.
\begin{example}[Infinite Density Ratios] Consider a discrete setting where, for simplicity, $\bx = 1$ is deterministic, and $\by \in [n+1] = \{1,\dots,n+1\}$. Suppose that $\by$ under $\Ptrain$ is distributed uniformly on $[n]$, and uniformly on $[n+1]$ under $\Ptest$. For discrete distributions, density ratios are just ratios of probabilities:
\begin{align}
\frac{\Ptest(\bx,\by)}{ \Ptrain(\bx,\by)} = \frac{\rmd \Ptest(\by)}{\rmd \Ptrain(\by)} = \begin{cases} \frac{n}{n+1} & \by \in [n]\\
\infty & \by = [n+1].
\end{cases}
\end{align}
Thus, the density ratios are not uniformly bounded, and may indeed take the value $\infty$. Still, $\Pr_{\by \sim \Ptest}[\by = n+1] = \frac{1}{n}$, so \Cref{defn:cor_dens_ratio_tail} with $t_1 = t_2 = \frac{n}{n+1} \le 1$yields 
\begin{align}
\Rtest(f,g) \le  2(\Rtrain[f] + \Rtrain(f,g))+ \frac{64B^2}{n+1},
\end{align}
which is not only not infinite, but indeed decays with $n$.
\end{example}
Our second example shows the sample can happen even if the density ratios are finite, but unbounded.
\begin{example} Let $\gamma_1,\gamma_2 \in (0,1)$. Again, consider deterministic $\bx = 1$ and discrete $\by \in \Z_{\ge 0}$ with geometric distributions. $\Ptrain[\by = k] = (1-\gamma_1)\gamma_{1}^k, \quad  \Ptest[\by = k] = (1-\gamma_2)\gamma_{2}^k$. 
Then, 
\begin{align*}
\frac{\rmd \Ptest(\by= k)}{\rmd \Ptrain(\by = k)} = \frac{1-\gamma_2}{1-\gamma_1} \cdot \left(\frac{\gamma_2}{\gamma_1}\right)^k, \quad k \in \Z_{\ge 0} \label{eq:geom_ratio}
\end{align*}
When $\gamma_2 > \gamma_1$, $\sup_{\by} \frac{\rmd \Ptest(\by)}{\rmd \Ptrain(\by)} = \infty$. Still, \Cref{defn:cor_dens_ratio_tail} is non vacuous, yielding 
\begin{align*}
\Rtest(f,g) \le  \inf_{t > 0} 2t(\Rtrain[f] + \Rtrain(f,g))+ 64 B^2 \Ptest\left[ \frac{1-\gamma_2}{1-\gamma_1} \cdot \left(\frac{\gamma_2}{\gamma_1}\right)^{\by} > t \right].
\end{align*}
 With some algebra\footnote{The missing steps are as follows. We have $\Ptest\left[ \frac{1-\gamma_2}{1-\gamma_1} \cdot \left(\frac{\gamma_2}{\gamma_1}\right)^{\by} > t \right] = \Ptest\left[ \by > \frac{\log(\frac{1-\gamma_1}{1-\gamma_2}) + \log t}{\log \frac{\gamma_2}{\gamma_1}} \right]$, which can be bounded using the distribution of $\by$ under $\Ptest[\by > u] \le \Ptest[\by \ge u] = \Ptest[\by \ge \ceil{u}] \le (\gamma_2)^{u} = \exp(u \log \gamma_2)$.}, we can compute
\begin{align*}
\Rtest(f,g) &\le  \inf_{t > 0} 2t(\Rtrain[f] + \Rtrain(f,g))+ 64 B^2 \left(t\cdot \frac{1-\gamma_1}{1-\gamma_2}\right)^{-\alpha}, \quad \alpha := \frac{\log (1/\gamma_2)}{\log \gamma_2/\gamma_1} > 0
\end{align*}
An interesting feature of this example is that, even thought the density ratio in \Cref{eq:geom_ratio} appears to grow exponentially $k$, the tradeoff in $t$ is \emph{polynomial} due to the exponential decay of $\by \sim \Ptest$.
\end{example}

\subsubsection{Consequences for certain $f$-divergences, and proof of \Cref{lem:chi_sq}}\label{sec:f_div_stuff} We show that \Cref{defn:cor_dens_ratio_tail} can be instantiated for a subclass of $f$-divergences, which include the $\chi^2$-divergence (with arguments ordered appropriately) as a special case. To avoid confusion with our function class $f$, we shall replace $f$ with the functions $\phi: \R_{ > 0 } \to \R_{\ge 0}$. 

\begin{definition}[$\phi$-divergence, Chapter 2 in \cite{polyanskiy2022information}] Given measures $\lawP,\lawQ$ on the same probability space $\cX = \{X\}$, and $\phi: \R_{> 0} \to \R$, we define 
\begin{align*}
\Dphi(\lawQ,\lawP) := \int_{X} \phi\left(\frac{\rmd \lawQ(X)}{\rmd \lawP(X)}\right) \rmd \lawP(X)
\end{align*}
where $\rmd \lawQ(X)$ and $\rmd \lawP(X)$ denote the Radon-Nikodym derivatives of $\lawQ$ and $\lawP$, respectively, evaluated at $X \in \cX$.  By convention, it is typically required that $\phi$ is convex, $\phi(1) = 0$, and that $\phi(0)$ is defined via $\phi(0) = \lim_{t \to 0^+}\phi(t)$.
\end{definition}

Observe that if the function $\phi$ satisfies $\phi + M$ is non-negative for some $M \in \R$, Markov's inequality implies
\begin{align*}
\Pr_{X \sim \lawP}\left[\phi\left(\frac{\rmd \lawQ(X)}{\rmd \lawP(X)}\right) >  u\right] \le \frac{\Dphi(\lawQ,\lawP)}{u}, \quad u \ge M.
\end{align*}
And, if in addition $\phi(u)$ is strictly decreasing, $\left\{\frac{\rmd \lawP(X)}{\rmd \lawQ(X)} >  t\right\} = \left\{\frac{\rmd \lawQ(X)}{\rmd \lawP(X)} < \frac{1}{t}\right\} =\left\{\phi(\frac{\rmd \lawQ(X)}{\rmd \lawP(X)}) < \phi(\frac{1}{t})\right\}$. Thus, 
\begin{align}
\Pr_{X \sim \lawP}\left[\frac{\rmd \lawP(X)}{\rmd \lawQ(X)} >  t\right] \le \frac{\Dphi(\lawQ,\lawP)}{\phi(1/t)}, \quad \phi(1/t) \ge M.
\end{align}
Then, \Cref{defn:cor_dens_ratio_tail} directly implies the following consequence.
\begin{corollary}\label{cor:phi_div} Consider any non-negative and strictly decreasing functions $\phi_1,\phi_2:\R_{> 0} \to [-M,\infty)$; the other axioms of the $\phi$-divergence need not be met. Then, for any $t_1, t_2 > 0$ for which $\phi_1(1/t_1),\phi_2(1/t_2) \ge M$,
\begin{align*}
\Rtest(f,g) &\le 2(t_{1}\Rtrain[f] + t_2\Rtrain(f,g))\\
&\quad +32B^2\left(\frac{\Dphi[\phi_1](\Ptrain(\bx,\by),\Ptest(\bx,\by))}{\phi_1(1/t_1)} +\frac{\Dphi[\phi_2](\Ptrain(\by),\Ptest(\by))}{\phi_2(1/t_2)}\right),
\end{align*}
where $\Dphi[\phi_1](\Ptrain(\bx,\by),\Ptest(\bx,\by)$ denotes the $\phi_1$-divergence between the joint distribution of $(\bx,\by)$ under $\lawQ = \Ptrain$ and $\lawP = \Ptest$, and $\Dphi[\phi_2](\Ptrain(\by),\Ptest(\by))$ the $\phi_2$-divergence between the marginal of $\by$ under these measures. 
\end{corollary}
We show now describe two special cases of interest.
\begin{example}[Power Divergences]\label{exm:power_div} A special case is when $\phi_i(t) = \frac{1}{t_i^{\alpha_i}} - 1$, $\alpha_i > 0$,\footnote{Note that this ensures that $\phi_i(\cdot)$ is strictly decreasing, $\phi_i + 1 \ge 0$, as well as the $f$-divergence axioms $\phi_i(1) = 0$ and $\phi_i$ is convex.} then for any $t_1,t_2 \ge 1$,
\begin{align*}
\Rtest(f,g) &\le 2(t_{1}\Rtrain[f] + t_2\Rtrain(f,g))\\
&\quad +32B^2\left(\frac{\Dphi[\phi_1](\Ptrain(\bx,\by),\Ptest(\bx,\by))}{t_1^{\alpha_1}} +\frac{\Dphi[\phi_2](\Ptrain(\by),\Ptest(\by))}{t_2^{\alpha_2}}\right).
\end{align*}
\end{example}

We obtain \Cref{lem:chi_sq} as a special case:
\begin{proof}[Proof of \Cref{lem:chi_sq}] An archetypical example of the special case of power-divergences described above is $\phi_1(u) = \phi_2(u) = \frac{1}{u} -1$. Then, it can be shown that $\Dphi(\lawQ,\lawP) = \chi^2(\lawP,\lawQ)$, where $\chi^2(\cdot,\cdot)$ denotes the $\chi^2$ divergence (note the reversed order of the arguments).  Thus, specializing \Cref{exm:power_div} further yields that, for any $t_1,t_2 \ge 1$, 
\begin{align*}
\Rtest(f,g) &\le 2(t_{1}\Rtrain[f] + t_2\Rtrain(f,g))\\
&\quad +32B^2\left(\frac{\chi^2(\Ptest(\bx,\by),\Ptrain(\bx,\by))}{t_1} +\frac{\chi^2(\Ptest(\by),\Ptrain(\by))}{t_2}\right).
\end{align*}
When $\Rtrain[f],\Rtrain(f,g)$ are sufficiently small relative to the above $\chi^2(\cdot,\cdot)$ divergences, we can minimize over $t_1,t_2$ without the $t \ge 1$ constraint:
\begin{align*}
\Rtest(f,g) &\le 8B\sqrt{\Rtrain[f]\cdot \chi^2(\Ptest(\bx,\by),\Ptrain(\bx,\by))} + 8B \sqrt{\Rtrain(f,g) \cdot \chi^2(\Ptest(\by),\Ptrain(\by))}. 
\end{align*}
\end{proof}

\subsubsection{Proof of \Cref{defn:cor_dens_ratio_tail}}\label{sec:deferred_change_of_measure}
We begin with a standard change-of-measure bound.
\begin{lemma}[Change of Measure]\label{lem:h_X_change_of_measure} For any measurable, bounded function $h:\cX \to [0,M]$,
\begin{align*}
\Exp_{X \sim \lawP}[h(X)] \le \inf_{t > 0 } M\Pr_{X\sim \lawP}\left[\left\{\frac{\rmd \lawP(X)}{\rmd \lawQ(X)} > t\right\}\right] + t\Exp_{X \sim \lawQ}[h(X)].
\end{align*}
\end{lemma}
\begin{proof}[Proof of \Cref{lem:h_X_change_of_measure}] Fix an event $\cE \subset \cX$, and suppose that $0 \le h(\cdot) \le M$. 
\begin{align*}
\Exp_{X \sim \lawP}[h(X)] &\le \Exp_{X \sim \lawP}[h(X)\I\{\cE\}] +\Exp_{X \sim \lawP}[h(X)\I\{\cE^c\}] \\
&\le \Exp_{X \sim \lawP}[h(X)\I\{\cE\}] +M\Pr[\cE^c] \\
&= M\Pr[\cE^c] + \int h(X)\I\{\cE\}\rmd\lawP(X)\\
&= M\Pr[\cE^c] + \int h(X)\I\{\cE\}\rmd\lawQ(X)\cdot \frac{\rmd \lawP(X)}{\rmd \lawQ(X)} \\
&\le M\Pr[\cE^c] +\left(\int h(X)\I\{\cE\}\rmd\lawQ(X)\right)\cdot \sup_{X \in \cE}\left(\frac{\rmd \lawP(X)}{\rmd \lawQ(X)}\right) \\
&\le M\Pr[\cE^c] + \Exp_{X \sim \lawQ}[h(X)]\cdot \sup_{X \in \cE}\left(\frac{\rmd \lawP(X)}{\rmd \lawQ(X)}\right).
\end{align*}
To conclulde, we take $\cE:=\{\frac{\rmd \lawP(X)}{\rmd \lawQ(X)} \le t\}$.
\end{proof}
\begin{proof}[Proof of \Cref{defn:cor_dens_ratio_tail}] 
We aim to establish the following:
\begin{align}
\Etest[(f-\fst - \upbeta_f)^2] &\le \nu_1 \Etrain[(f-\fst - \upbeta_f)^2] + \Delta_1 \label{eq:nu1_del}\\
\Etest[(g-\gst - \upbeta_f)^2] &\le \nu_2\Etrain[(g-\gst - \upbeta_f)^2] + \Delta_2 \label{eq:nu2_del},
\end{align}
In view of \Cref{cor:additive_error}, it suffices to check that for any $t_1,t_2 > 0$, \Cref{eq:nu1_del,eq:nu2_del} hold with $(\nu_1,\nu_2) \gets (t_1,t_2)$, and $(\Delta_1,\Delta_2) \gets (16B^2\Delta_{x,y}(t_1), 16B^2\Delta_{y}(t_2)))$. Consider the functions of the form $h_1(\bx,\by) = (f(\bx) - \fst(\bx) - \upbeta_f(\by))^2$ and $h_2(\by) = (g(\by) - \gst(\by) - \upbeta_f(\by))^2$. By \Cref{asm:bounded}, it holds that the image of both functions lies in $[0,16B^2]$. The result now follows by applying \Cref{lem:h_X_change_of_measure} with $M \gets 16B^2$ to the functions $h_1$ (resp $h_2$) with $t \gets t_1$ (resp. $t \gets t_2$).
\end{proof}

\subsection{Correspondence with Double ML}
As we have mentioned, our results have a similar flavor to those in
the literature on Neyman orthogonalization. In this section, we expand
on this comparison, following the treatment
in~\citep{foster2019orthogonal}. As terminology, we refer to these
idea broadly as orthogonal (statistical) learning.

The orthogonal learning setup describes a similar situation where a
pair $(\fst,\gst)$ is unknown, but weare primarily interested in
$\fst$, referring to $\gst$ as a nuisance function. For example, we
may have $\E{\zvar \mid \xvar,\yvar} = \fst(\xvar) + \gst(\yvar)$ as
in our setup, but orthogonal learning is more general in this
respect. Unlike our setting however, orthogonal learning requires some
auxiliary mechanism or data to learn $\ghat$ such that $\E{
  (\ghat(\yvar) - \gst(\yvar))^2} \lesssim \rate_n(\cG)$. Given this
initial estimate, orthogonal learning describes an algorithm and
conditions under which one can learn $\fhat$ satisfying
\begin{align*}
  \E{(\fhat(\xvar) - \fst(\xvar))^2} \lesssim \rate_n(\cF) + \rate_n(\cG)^2.
\end{align*}
Here $\rate_n(\cdot)$ should be thought of as the standard ``fast"
rate for learning with the function class, i.e., $\rate_n(\cF) \asymp
\frac{\log |\cF|}{n}$ when $|\cF| < \infty$. This guarantee naturally
leads to a distribution shift bound of the form:
\begin{align*}
  \Rtest(\hat{f},\hat{g}) \lesssim \nu_x\left(\rate_n(\cF) + \rate_n(\cG)^2\right) + \nu_y\rate_n(\cG).
\end{align*}
As with our bound, the complexity of the class $\cG$ does not interact
with distribution shifts on $\nu_x$ in a significant way. Indeed, the
error rate for $\hat{f}$ has a quadratic dependence on that for
$\hat{g}$, and this is typically lower order when considering
distribution shift settings.  Thus, at a conceptual level, orthogonal
learning can provide a similar robustness to heterogeneous distribution
shifts as our results.

Quantitatively, the bound should be compared with
our~\Cref{thm:main_guarante}. The general takeaway is that our bound
is worse in at least two respects. First, the quadratic dependence on
the rate for $\cG$ cannot exploit localization in our setup, so our
bound is weaker when $\cG$ is small. This is why we need
hypercontractivity conditions to obtain favorable rates when $\cG$ is
finite/parametric, which is not required for orthogonal
learning. Second, when considering distribution shift, we incur a
dependence on $\nu_{x,y}$ rather than just $\nu_x$. This arises from
the identifiability issues that are inherent in our setting, which can
be resolved in the orthogonal learning setting due to the auxiliary
mechanism for estimating $\gst$.

On the other hand, our results compare favorable at a qualitative
level. Most importantly, our bound applies to ERM directly while
orthogonal learning requires algorithmic modifications, which in turn
require more modeling of the data generating process. Additionally,
while we do not believe the assumptions are formally comparable, we
view ours as somewhat more practical. Specifically, it is rather
uncommon that auxiliary information for estimating the nuisance
parameter is available; yet in canonical settings for orthogonal
learning, we can show that conditional completeness holds. One such
example of the latter is the PLR model above.


\newcommand{\Term}{\mathrm{Term}}
\newcommand{\rquad}{r_{\mathrm{quad}}}

\section{Proof of Main Technical Results}\label{sec:proof_main_props}
This section provides the proofs of the most significant technical results in the paper. Specifically, \Cref{sec:lem:excess_decomp} give the proofs of the excess error decompositions, \Cref{lem:excess_decomp}. \Cref{sec:lem:excess_decomp_lin} establishes \Cref{lem:lin_shift,lem:nu_x_indep}, which refine upper bounds bounds on the distribution-shift term $\nu_1$.  Next, we prove \Cref{prop:generic_sum_regret} which upper bounds $\Rtrain[f,g]$, and thus, in view of \Cref{lem:excess_decomp}, $\Rtrain[g;f]$.  \Cref{sec:prop:main_reg} gives the proof of \Cref{prop:main_reg} which provides refined control of $\Rtrain[f]$; the key step is an (empirical) excess-risk decomposition, \Cref{lem:reg_decomp}, which we prove in \Cref{sec:prop_main_reg_supporting}. Finally, \Cref{sec:prop:dudley_holder} establishes our H\"older inequality for Rademacher complexities of Hadamard product classes (\Cref{prop:dudley_holder}). Lastly, \Cref{sec:lem:Dudley_ub} derives the standard Dudley integral bound from the aforementioned proposition for product classes.

\subsection{Proof of \Cref{lem:excess_decomp}}\label{sec:lem:excess_decomp} Write $\upbeta = \upbeta_f$ for simplicity.
For any environment $\env$ and any triplet  $(f,g,\upbeta)$, the polarization identity yields
\begin{align*}
\Renv(f,g) &= \Exp_{\env}[(f+g - \fst - \gst)^2]\\
&= \Exp_{\env}[(f - \fst - \upbeta + g - \gst + \upbeta)^2]\\
&= \Exp_{\env}[(f - \fst - \upbeta)^2] + \Exp_{\env}[(g - \gst + \upbeta)^2] + 2\Exp_{\env}[(f-\fst - \upbeta)(g-\gst + \upbeta)]\\
&= \Renv[f] + \Renv[g;f] + 2\Exp_{\env}[(f-\fst - \upbeta)(g-\gst + \upbeta)].
\end{align*}
For any $\env$ (in particular, $\env = \envtest$), we have 
\begin{align*}
\Exp_{\env}[(f-\fst - \upbeta)(g-\gst + \upbeta)] &\le \Exp_{\env}[(f - \fst - \upbeta)^2]^{1/2} \cdot \Exp_{\env}[(g - \gst + \upbeta)^2]^{1/2} \\
&\le \Exp_{\env}[(f - \fst - \upbeta)^2] + \Exp_{\env}[(g - \gst + \upbeta)^2] = \Renv[f] + \Renv[g;f].
\end{align*}
Hence,
\begin{align*}
\Renv(f,g) \le  2(\Renv[f] + \Renv[g;f]).
\end{align*}
When $\env = \envtrain$, the fact that $\upbeta(\yvar) = \Ezero[(f-\fst)(\xvar) \mid \yvar]$ implies
\begin{align*}
\Etrain[(f-\fst - \upbeta)(g-\gst + \upbeta)] &= \Etrain [((f-\fst)(\xvar) - \upbeta(\yvar))\cdot(g-\gst + \upbeta)(\yvar) \mid \yvar]\\
&= \Etrain [ (\Etrain[(f-\fst)(\xvar) \mid \yvar] - \upbeta(\yvar))\cdot (g-\gst + \upbeta)(\yvar)] = 0.
\end{align*}
Thus,
\begin{align*}
\Rtrain(f,g) = \Rtrain[f] + \Rtrain[g;f].
\end{align*}

This proves the first two parts of the lemma. For the last part, we have
\begin{align*}
\Rtest(f,g) &\le 2\Rtest[f] + 2\Rtest[g;f]\tag{\Cref{lem:excess_decomp}} \\
&\overset{(i)}{\le} 2\nu_{x,y}\Rtrain[f] + 2\nu_y\Rtrain[g;f]  \numberthis \label{eq:nuxy_appear}\\
&\overset{(ii)}{\le} 2\nu_{x,y}\Rtrain[f] + 2\nu_y\Rtrain(f,g),
\end{align*}
where $(i)$ invokes \Cref{def:nux}, and
where in $(ii)$, $\Rtrain[f] \ge 0$ and $\Rtrain(f,g) = \Rtrain[f] + \Rtrain[g;f]$ implies $\Rtrain[g;f] \le \Rtrain(f,g)$.
\qed

\subsection{Proof of \Cref{prop:generic_sum_regret}}\label{sec:prop:generic_sum_regret}
\newcommand{\hhatn}{\hat{h}_n}
This section proves \Cref{prop:generic_sum_regret}, which we use to upper bound $\Rtrain[f,g]$, and thus, by way of \Cref{lem:excess_decomp}, $\Rtrain[g;f]$.
We begin by establishing the following more or less standard guarantee (see, e.g. \cite{liang2015learning}) for a generic function class $\cH$. which controls the so-called ``basic inequality'' in square-loss learning (see, e.g.  \citet[Chapter 13]{wainwright2019high}). 
\begin{lemma}\label{lem:basic_ineq_lb} Let $\cH$ be a functions from $\cW \to [-B,B]$ containing the zero function $h_0(w) \equiv 0$.  Fix $\sigma > 0$, $\tau \ge 1$. Then, for any probability measure $P$, $\bw_1,\dots,\bw_n \iidsim \Pr$ and i.i.d. standard normal random variables $\bxi_1\dots,\bxi_n$,  the following holds with probability $1 - \delta$
\begin{align*}
\sup_{h \in \cH} \frac{1}{4}\|h\|_{\cL_2(P)}^2 - \|h(\bw_{1:n})\|_{2,n}^2 +   \frac{2\tau\sigma}{n} \sum_{i=1}^n \xivar_i h(\bw_i)   \lesssim  \gamma_{n,\delta,\sigma}(\cH)^2,
\end{align*}
where $C > 0$ is a universal constant and 
\begin{align*}
\gamma_{n,\sigma,\tau}(\cH,\delta)^2  = \delnst(\cH,B)^2 +  \tau^2\rhonst( \cH,\sigma)^2 +  \frac{(\tau^2\sigma^2 + B^2)\log(1/\delta)}{n}. 
\end{align*}
\end{lemma}
\begin{proof} We have
\begin{align*}
&\frac{1}{4}\|h\|_{\cL_2(P)}^2 - \|h(\bw_{1:n})\|_{2,n}^2 +   \frac{2\tau\sigma}{n} \sum_{i=1}^n \xivar_i h(\bw_i)\\
&=\frac{1}{4}\underbrace{\left(\|h\|_{\cL_2(P)}^2 - 2\|h(\bw_{1:n})\|_{2,n}^2\right)}_{\mathrm{Term}_1} + \underbrace{\left(\frac{-1}{2}\|h(\bw_{1:n})\|_{2,n}^2+\frac{2\tau\sigma}{n} \sum_{i=1}^n \xivar_i h(\bw_i)\right)}_{\mathrm{Term}_2}
\end{align*}
By \Cref{lem:quad_lb} and \Cref{lem:gauss_offset}, respecitively, the following holds with probability at least $1 -\delta$,
\begin{align*}
\mathrm{Term}_1 &\lesssim \left(\delnst(\cH,B)^2  + \frac{ B^2 \log(1/\delta)}{n} \right)\\
\mathrm{Term}_2 &\lesssim \tau^2\left(\frac{\sigma^2\log(1/\delta)}{n} + \rhonst(\cH,\sigma)^2\right).
\end{align*}
Summing concludes.
\end{proof}
\newcommand{\gammatrain}{\gamma_{\mathsf{train}}}
\begin{proof}[Proof of \Cref{prop:generic_sum_regret}] Let $\bw_i = (\bx_i,\by_i)$ and $\bw = (\bx,\by)$, and set $\hst := \fst + \gst$ and $\hat{h}_n := \fhatn + \ghatn$. Note that $\hhatn - \hst \in \Hcent$. It  follows from the so-called ``basic ineqality'' \citet[Eq. 13.36]{wainwright2019high} that
\begin{align*}
\|(\hhatn - \hst)(\bw_{1:n})\|_{2,n}^2 - 2\xivar_i (\hat{h}_n-\hst)(\bw_i) \le 0,
\end{align*}
Thus, by adding and subtracting $\frac{1}{4}\Exp[\hat{h}_n(\bw)^2]$ and rearranging
\begin{align*}
\Exp[(\hat{h}_n - \hst)(\bw)^2] \le 4\left( - \Exp[(\hat{h}_n - \hst)(\bw)^2] +  \| (\hhatn - \hst)(\bw_{1:n})\|_{2,n}^2 + 2\xivar_i (\hat{h}_n-\hst)(\bw_i)\right).
\end{align*}
Passing to the supremum over all $h \in \Hcent := \cF + \cG - (\fst+\gst)$ and invoking \Cref{lem:basic_ineq_lb} shows that
\begin{align*}
\Rtrain(\fhatn,\ghatn) \lesssim 
 \gamma_n(\delta)^2 := \delnst(\Hcent,B)^2 +  \rhonst(\Hcent,\sigma)^2 + \frac{(B^2 + \sigma^2)\log(1/\delta)}{n}.
\end{align*}
\end{proof}

\subsection{Proof of \Cref{prop:main_reg}}\label{sec:prop:main_reg}

\newcommand{\Rcross}{\hat{\cR}^{\mathrm{cross}}_n}
This section proves \Cref{prop:main_reg}, which we use to bound $\Rtrain[f]$. It is considerably more involved than the proof of \Cref{prop:generic_sum_regret}, as we need to argue that the ``large'' class $\cG$ does not heavily obfuscate recovery in the class $\cF$. Throughout, let us use $\bw = (\bx,\by)$, $\bw_i = (\bx_i,\by_i)$, and $\cW = \cX \times \cY$. Recall the sets 
\begin{align*}
\Fcent := \{f - \upbeta_f - \fst: f \in \cF\}, \quad \Gcent := \{g + \upbeta_f - \fst: f \in \cF, g \in \cG\}.
\end{align*}
We begin by uniformly bounding the elements of $\Fcent$ and $\Gcent$.
\begin{lemma}\label{lem:cent_bound}  For any $h \in \Fcent \cup \Gcent$, $\sup_{w \in \cH} |h(w)| \le 4B$, where $B$ is as in \Cref{asm:bounded}.
\end{lemma}
\begin{proof} Follows directly from \Cref{asm:bounded} and the fact that $\upbeta_f(y) := \Exp[(f-\fst)(\xvar) \mid \yvar = y]$.
\end{proof}

In \Cref{sec:prop_main_reg_supporting}, we prove the following, which shows that the conditional completeness allows us to decompose $\Rtrain[\fhatn]$ across these two terms, the first of which represents a standard excess risk in terms of $\cF$, and the latter of which measures the contamination due to errors in $\cG$. This bound can be thought of as a careful refinement of the standard ``basic inequality'' \citep[Eq. 13.36]{wainwright2019high}.
\begin{lemma}\label{lem:reg_decomp} If $(\cF,\cG)$ satisfies $\gamma$-conditional completeness, then for any empirical risk minimizer $(\fhatn,\ghatn)$ for which $\Rtrain(\fhatn,\ghatn) \le \gamma^2$, the following bound holds deterministically:
\begin{align*}
\Rtrain[\fhatn] &\le  8\sup_{h_1 \in \Fcent} \left(\frac{1}{4}\Exp[h_{1}(\bw)^2]  - \frac{1}{2} \|h_{1}(\bw_{1:n})\|_{2,n}^2 -  2\sigma \cdot \frac{1}{n}\sum_{i=1}^n\xivar_i h_{1}(\bw_i)\right) \tag{$\Term_1$}\\
&\quad+ 32\sup_{h_1 \in \Fcent, h_2 \in \Gcent(\gamma)}\left(- \frac{1}{2} \Exp[h_1(\bw)^2]  - \frac{4}{n}\sum_{i=1}^n h_1(\bw_i) \cdot h_2(\bw_i)\right). \tag{$\Term_2$}
\end{align*}
\end{lemma}
\begin{align*}
\end{align*}
In the remainder of the proof, we apply various learning-theoretic tools to upper bound the right-hand side of \Cref{lem:reg_decomp}. These tools, and their proofs, are detailed in \Cref{app:tech_tools}. While the tools themselves are more-or-less standard, deriving from the offset-Rademacher arguments in \cite{liang2015learning}, their application to the refined decomposition in \Cref{lem:reg_decomp} yields the novelty of \Cref{prop:main_reg}.

\begin{proof}[Proof of \Cref{prop:main_reg}]
We prove the variant of the lemma  of the lemma involving $\delcrossn$, and explain how to modify the proof to obtain dependence on $\delcrossnbar$ at the end.

The first term in the above display can be bounded directly from \Cref{lem:basic_ineq_lb} with $\tau$ set to $1$, yielding 
\begin{align}\label{eq:term_one_bound}
\Term_1 \lesssim \delnst(\Fcent,B)^2 +  \rhonst( \Fcent,\sigma)^2 +  \frac{(\sigma^2+ B^2)\log(1/\delta)}{n}. 
\end{align}
To bound the second term, we use a localization argument. Define the doubly-localized term
\begin{align*}
\Psi(r,\gamma) := \sup_{h_1 \in \Fcent(r), h_2 \in \Gcent(\gamma)}\left(- \frac{r^2}{2}   - \frac{4}{n}\sum_{i=1}^n h_1(\bw_i) \cdot h_2(\bw_i)\right).
\end{align*}
Following the same argument as in \Cref{lem:offset_master}, one can check that
\begin{align*}
\Term_2 = 32\sup_{r > 0}\Psi(r,\gamma), 
\end{align*}
and that 
\begin{align}
\Term_2 \le 32\cdot\inf\left\{r^2:\Psi(r,\gamma) \le \frac{r^2}{2}\right\}.  \label{eq:Term_Two_Thing}
\end{align}
Hence, let us exhibit an $r$ for which $\Psi(r,\gamma) \le 0$ with high probability. Define the shorthand $\cH_{r,\gamma} := \Fcent(r) \odot \Gcent(\gamma)$.  For an $r > 0$ to be chosen, \Cref{lem:rad_uniform_convergence} implies that with probability $1 - \delta/4$,
\begin{align*}
\Psi(r,\gamma)  \le c\left(\Exp_{\bw_{1:n}}[\Raden(\cH_{r,\gamma}[\bw_{1:n}])] + r\sqrt{\frac{ \log(1/\delta)}{n}} + \frac{ B\log(1/\delta)}{n}\right) - \frac{r^2}{8}
\end{align*}
where $c \ge 1$ is a universal constant. By AM-GM, we have
\begin{align*}
\Psi(r,\gamma)  &\le c\left(\Exp_{\bw_{1:n}}[\Raden(\cH_{r,\gamma}[\bw_{1:n}])] + \frac{ r^2}{16c} + \frac{(8c + B)\log(1/\delta)}{n}\right)  - \frac{r^2}{8}\\
&= c\left(\Exp_{\bw_{1:n}}[\Raden(\cH[\bw_{1:n}])] - \frac{ r^2}{16c} + \frac{(8c + B)\log(1/\delta)}{n}\right) \numberthis \label{eq:Psi_last_line}
\end{align*}
We now compute
\begin{align*}
&\Exp_{\bw_{1:n}\sim P}[\Raden(\cH_{r,\gamma}[\bw_{1:n}])] - \frac{ r^2}{16c} \\
&= \Exp_{\bw_{1:n}}\Exp_{\epsvar_{1:n}}\left[\sup_{h \in \cH_{r,\gamma}}\frac{1}{n}\sum_{i=1}^n \epsvar_i h(\bw_i)\right] - \frac{ r^2}{16c}\\
&= \Exp_{\bw_{1:n}}\Exp_{\epsvar_{1:n}}\left[\sup_{h_1 \in \Fcent(r), h_2 \in \Gcent(\gamma) }\frac{1}{n}\sum_{i=1}^n \epsvar_i h_1(\bw_i)h_2(\bw_i)\right] - \frac{ r^2}{16c} \tag{Definition of $\cH_{r,\gamma}$}\\
&= \Exp_{\bw_{1:n}}\Exp_{\epsvar_{1:n}}\left[\sup_{h_1 \in \Fcent(r), h_2 \in \Gcent(\gamma) }\frac{1}{n}\sum_{i=1}^n \epsvar_i h_1(\bw_i)h_2(\bw_i) - \frac{1}{16c}\Exp[h_1(\bw)^2]\right]  \tag{Localization of $\Fcent(r)$}\\
&\le \frac{1}{16c}T_1 + \frac{1}{32c}T_2, 
\end{align*}
where  we define
\begin{align*}
 T_1&:= \Exp_{\bw_{1:n}}\Exp_{\epsvar_{1:n}}\left[\sup_{h_1 \in \Fcent(r), h_2 \in \Gcent(\gamma) }\frac{16c}{n}\sum_{i=1}^n \epsvar_i h_1(\bw_i)h_2(\bw_i) - \frac{1}{2}\|h_1(\bw_{1:n})\|_{2,n}^2\right]\\
  T_2 &:= 
\Exp_{\bw_{1:n}}\left[\sup_{h_1 \in \Fcent } \|h_1(\bw_{1:n})\|_{2,n}^2 - 2\Exp[h_1(\bw)^2] \right].
 \end{align*}
\paragraph{Bounding $T_1$.} To being, we remove the localization of $\Fcent$ in the term $T_1$, upper bounding
\begin{align*}
T_1 \le T_1' := \Exp_{\bw_{1:n}}\Exp_{\epsvar_{1:n}}\left[\sup_{h_1 \in \Fcent, h_2 \in \Gcent(\gamma) }\frac{16c}{n}\sum_{i=1}^n \epsvar_i h_1(\bw_i)h_2(\bw_i) - \frac{1}{2}\|h_1(\bw_{1:n})\|_{2,n}^2\right]
\end{align*}

 For any fixed $w_{1:n} \in \cW^n$, consider the process
\begin{align*}
\frac{1}{n}\sum_{i=1}^n \epsvar_i h_1(w_i)h_2(w_i).
\end{align*}
Introduce the class 
\begin{align}
\tilde{\cH}[\rho] := \{h_1 \cdot h_2: h_1 \in \Fcent, h_2 \in \Gcent(\gamma), \|h_1(w_{1:n})\|_{2,n} \le \rho \},
\end{align}
which localizes $h_1$ at empirical $\cL_2$-norm $\rho$. Note that that, for any $\tilde{h} \in \tilde{\cH}[\rho]$, we can write $\tilde{h} = h_1 \cdot h_2$ where
\begin{align*}
\|\tilde{h}\|_{2,n}^2 = \frac{1}{n}\sum_{i=1}^n h_1(w_i)^2 h_2(w_i)^2 \le \frac{4B^2}{n}\sum_{i=1}^n h_1(w_i)^2 = 4B^2\|h_1(w_{1:n})\|_{2,n}^2 \le 4B^2 \rho^2,
\end{align*}
where the first inequality is by \Cref{lem:cent_bound} and second by definition of $\tilde \cH[\rho]$. Again by \Cref{lem:cent_bound}, $\max_i \sup_{\tilde h \in \tilde \cH[\rho]} \le 4B^2$. It follows by \Cref{lem:rademacher_concentration} that the following holds with probability $1 - \delta$ 
\begin{align*}
\bZ[\rho,w_{1:n}] &:= \sup_{h_1 \in \Fcent: \|h_1(w_{1:n})\|_{2,n} \le \rho}\sup_{h_2 \in \Gcent(\gamma)}\frac{1}{n}\sum_{i=1}^n \epsvar_i \tilde{h}(w_i)\\
&=\sup_{\tilde{\cH}[\rho] }\frac{1}{n}\sum_{i=1}^n \epsvar_i \tilde{h}(w_i)  \\
&\lesssim \Exp\left[\sup_{\tilde{\cH}[\rho] }\frac{1}{n}\sum_{i=1}^n \epsvar_i \tilde{h}(w_i)\right] + B\rho\sqrt{\log(1/\delta)/n} + \frac{B^2}{n}.
\end{align*}
Applying \Cref{lem:offset_master} with the classes
\begin{align*}
\bbV = \cH[w_{1:n}], \quad \bu_i = (\epsvar_i, i), \quad \Phi := \{\bu_i \mapsto \epsvar_i h_2(w_i): h_2 \in \Gcent(\gamma) \}
\end{align*}
and constants
\begin{align*}
c_1 \lesssim 1, \quad c_2 \lesssim B, \quad c_2 \lesssim B^2, \sigma = 1, \quad \tau  \lesssim 1,
\end{align*}
we conclude
\begin{align}
T_1' &\le \sup_{w_{1:n}}\Exp_{\epsvar_{1:n}}\left[\sup_{h_1 \in \Fcent, h_2 \in \Gcent(\gamma) }\frac{16c}{n}\sum_{i=1}^n \epsvar_i h_1(w_i)h_2(w_i) - \frac{1}{2}\|h_1(w_i)\|_{2,n}^2\right] \lesssim \updelta_n^2 + \frac{B^2}{n}, \label{eq:T1_intermed}
\end{align}
where
\begin{align*}
\updelta_n^2 := \inf\left\{\rho^2: \sup_{w_{1:n}}\Exp[\bZ[\rho,w_{1:n}]] \le \frac{\rho^2}{2}\right\}.
\end{align*}
Note that 
\begin{align*}
\sup_{w_{1:n}}\Exp[\bZ[\rho,w_{1:n}]] := \sup_{w_{1:n}}\Raden(\Fcent[\rho,w_{1:n}] \odot \Gcent(\gamma)[w_{1:n}] ),
\end{align*}
so that
\begin{align*}
\updelta_n^2 &= \inf\{\rho^2: \sup_{w_{1:n}}\Raden(\Fcent[\rho,w_{1:n}] \odot \Gcent(\gamma)[w_{1:n}] )\le \frac{\rho^2}{2}\}\\
&:= \delcrossn^2(\Fcent;\Gcent(\gamma)), \tag{\Cref{defn:cross_crit_rad}}
\end{align*}
so that by \Cref{eq:T1_intermed},
\begin{align}
T_1 \le T_1' &\lesssim\delcrossn^2(\Fcent;\Gcent(\gamma)) + \frac{B^2}{n}. \label{eq:T1_final}
\end{align}
\paragraph{Bounding $T_2$.}

This second term can be bounded by \Cref{lem:sasha_lem} and is at most
\begin{align}
T_2 \lesssim \frac{B^2}{n} + \delnst(\Fcent,4B)^2\lesssim \frac{B^2}{n} + \delnst(\Fcent,B)^2 \label{eq:T2_final}
\end{align}
where the first inequality also uses \Cref{lem:cent_bound} to bound $\sup_{w}|h(w)| \le 4B$ for $h \in \Fcent(r)$, and the second uses \Cref{lem:star_shaped} to remove the factor of $4$. Hence, with probability $1 - \delta/4$, the following inequality holds for any fixed $r > 0$:

\paragraph{Concluding the proof}
Combining \Cref{eq:T1_final,eq:T2_final,eq:Psi_last_line} gives that with probability $1-\delta$,
\begin{align*}
\Psi(r,\gamma) &\lesssim  T_1 +T_2  + \frac{(1+B)\log(1/\delta)}{n}\\
&\lesssim  \frac{B^2 + (1+B)\log(1/\delta)}{n} +  \delnst(\Fcent,B)^2 +  \delcrossn^2(\Fcent;\Gcent(\gamma))\\
&\lesssim  \frac{(1+B^2)\log(1/\delta)}{n} +  \delnst(\Fcent,B)^2 +  \delcrossn^2(\Fcent;\Gcent(\gamma)).
\end{align*}
Hence, if for a sufficiently large constant $c'$, we take
\begin{align*}
r^2 := c'\left(\frac{(1+B^2)\log(1/\delta)}{n} +  \delnst(\Fcent,B)^2 +  \delcrossn^2(\Fcent;\Gcent(\gamma))\right),
\end{align*}
then $\Psi(r,\gamma)  \le \frac{r^2}{2}$ with probability at least $1 - \delta$. Therefore by \Cref{eq:Term_Two_Thing}, we conclude that with probability $1 - \delta$,
\begin{align*}
\Term_2 \lesssim \frac{(1+B^2)\log(1/\delta)}{n} +  \delnst(\Fcent,B)^2 +  \delcrossn^2(\Fcent;\Gcent(\gamma)). 
\end{align*}
Combining with the bound on $\Term_1$ due to \Cref{eq:term_one_bound} concludes the proof.

\subsection{A modification of \Cref{prop:main_reg} }
For sharper rates with finite function classes (\Cref{sec:finite_function_classes}), we modify \Cref{prop:main_reg} as follows.
\begin{restatable}[Population-Localized Cross-Critical Radius]{definition}{pop_local}\label{defn:pop_local} 
\begin{align}
     \delcrossnbar(\Fcent;\cH) := \inf\left\{r: \Exp_{\bw_{1:n}}\Raden((\Fcent(r) \odot \cH) [\bw_{1:n}] )\le \frac{r^2}{2}\right\}.
     \end{align}
\end{restatable}


\begin{proposition}\label{prop:main_reg_mod} Suppose that $(\cF,\cG)$ satisfy $\gamma$-conditional completeness. Then, whenever $\Rzero[\ghatn;\fhatn] \le \gamma$, the following holds with probability at least $1 - \delta$,
    \begin{align*}
    \Rzero[\fhatn] &\lesssim \delcrossnbar^2(\Fcent;\Gcent(\gamma))^2+ \delnst(\Fcent,B)^2 +  \rhonst( \Fcent,\sigma)^2 +  \frac{(\sigma^2 + B^2)\log(1/\delta)}{n}. 
    \end{align*}
    \end{proposition}

\paragraph{Modification to obtain dependence on $\delcrossnbar$. } To obtain a dependence on $\delcrossnbar$, we change our bound the term $T_1$ above to use \Cref{lem:random_design_master_lemma} instead of \Cref{lem:offset_master}. The details are very similar. 
\end{proof}

\subsection{Proof of \Cref{lem:reg_decomp}}\label{sec:prop_main_reg_supporting}

This section establishes the generalized excess-risk decomposition which forms the basis of the argument in the previous section, and which decouples - via conditional-completeness - the recovery of $f \in \cF$ with conflation by $g \in \cG$.

Recall $\bw = (\bx,\by)$ and for $f \in \cF$, define 
\begin{align*}
h_f(\bw) := (f- \fst)(\bx) - \upbeta_f(\by),
\end{align*}
and note that $h_f \in \Fcent$.  Then, for any $f \in \cF$, $g \in \cF$, and $g_0 \in \cG$, we have
\begin{align*}
&\Lossn(f,g) - \Lossn(\fst,g_0) \\
&= \frac{1}{n} \sum_{i=1}^n (f(\xvar_i)+ g_0(\yvar_i) - \bz_i)^2 - (\fst(\xvar_i)+ g_0(\yvar_i) - \bz_i)^2\\
&= \frac{1}{n} \sum_{i=1}^n ((f - \fst)(\xvar_i) - \sigma \xivar_i + (g - \gst ) (\yvar_i))^2 - (- \sigma \xivar_i + (g_0 - \gst)(\yvar_i))^2\\
&= \frac{1}{n} \sum_{i=1}^n ((f - \fst)(\xvar_i) - \upbeta_f(\yvar_i)  - \sigma \xivar_i + (g - \gst + \upbeta_f) (\yvar_i))^2 - (- \sigma \xivar_i + (g_0 - \gst)(\yvar_i))^2\\
&= \frac{1}{n} \sum_{i=1}^n ((f - \fst)(\xvar_i) - \upbeta(\yvar_i)  - \sigma \xivar_i + (g - \gst + \upbeta_f) (\yvar_i))^2 - (- \sigma \xivar_i + (g - \gst + \upbeta_f)(\yvar_i))^2\\
&\quad + \frac{1}{n}\sum_{i=1}^n (- \sigma \xivar_i + (g - \gst + \upbeta_f)(\yvar_i))^2 - (- \sigma \xivar_i + (g_0 - \gst)(\yvar_i))^2\\
&= \frac{1}{n} \sum_{i=1}^n ((f - \fst)(\xvar_i) - \upbeta_f(\yvar_i))^2   - 2\sigma \cdot \frac{1}{n}\sum_{i=1}^n\xivar_i ((f - \fst)(\xvar_i) - \upbeta_f(\yvar_i)) \\
&\qquad + \frac{2}{n}\sum_{i=1}^n ((f - \fst)(\xvar_i) - \upbeta_f(\yvar_i))\cdot(g - \gst + \upbeta_f)(\yvar_i))\\
&\quad + \underbrace{\frac{1}{n}\sum_{i=1}^n (- \sigma \xivar_i + (g - \gst + \upbeta_f)(\yvar_i))^2 - (- \sigma \xivar_i + (g_0 - \gst)(\yvar_i))^2}_{\mathrm{Remainder}(g_0;f,g)}
\end{align*}
Applying the definition of $h_f$, the above admits the more compact form
\begin{align*}
\Lossn(f,g) - \Lossn(\fst,g_0) &:= \frac{1}{n} \|h_f(\bw_{1:n})\|_{2,n}^2  - 2\sigma \cdot \frac{1}{n}\sum_{i=1}^n\xivar_i h_f(\bw_i) + \frac{2}{n}\sum_{i=1}^n h_f(\bw_i) \cdot(g - \gst + \upbeta_f)(\yvar_i))\\
&\quad +  \mathrm{Remainder}(g_0;f,g). 
\end{align*}
By $\gamma$-conditional completeness, we have that if $\Rtrain[f,g] \le \gamma^2$, then we may select $\tilde{g} = g+ \upbeta_f \in \cG$ so that $\mathrm{Remainder}(g_0;f,g) = 0$. Similarly, if we now consider $\hat{f}_n,\hat{g}_n$ to be empirical risk minimizers of $\Lossn(f,g)$, it must hold that $\Lossn(\fhatn,\ghatn) - \inf_{g_0 \in \cG}\Lossn(\fst,g_0) \le 0$. Thus,
\begin{align*}
0 &\ge \Lossn(f,g) - \inf_{g_0 \in \cG}\Lossn(\fst,g_0) \\
&\ge \Lossn(f,g) - \Lossn(\fst,\tilde g)\\
&\ge \frac{1}{n} \|h_{\fhatn}(\bw_{1:n})\|_{2,n}^2  - 2\sigma \cdot \frac{1}{n}\sum_{i=1}^n\xivar_i h_{\fhatn}(\bw_i) + \frac{2}{n}\sum_{i=1}^n h_{\fhatn}(\bw_i) \cdot(g - \gst + \upbeta_{\fhatn})(\yvar_i))\\
&= \frac{1}{n} \|h_{\fhatn}(\bw_{1:n})\|_{2,n}^2  - 2\sigma \cdot \frac{1}{n}\sum_{i=1}^n\xivar_i h_{\fhatn}(\bw_i) + \frac{2}{n}\sum_{i=1}^n h_{\fhatn}(\bw_i) \cdot(g - \gst + \upbeta_{\fhatn})(\yvar_i))
\end{align*}
Note that $\Exp[h_{\fhatn}(\bw)^2]$ is precisely equal to $\Rtrain[\fhatn]$. Adding and substracting an $\eta$ multiple of this term for $\eta$ tunable,
\begin{align*}
\eta\Rtrain[\fhatn] &\ge \frac{1}{n} \|h_{\fhatn}(\bw_{1:n})\|_{2,n}^2 - \eta\Exp[h_{\fhatn}(\bw)^2]  - 2\sigma \cdot \frac{1}{n}\sum_{i=1}^n\xivar_i h_{\fhatn}(\bw_i) \\
&+ \frac{2}{n}\sum_{i=1}^n h_{\fhatn}(\bw_i) \cdot(g - \gst + \upbeta_{\fhatn})(\yvar_i)).
\end{align*}
Rearranging,
\begin{align*}
\Rtrain[\fhatn] &\le \eta^{-1}\left(\eta\Exp[h_{\fhatn}(\bw)^2]  - \frac{1}{n} \|h_{\fhatn}(\bw_{1:n})\|_{2,n}^2 -  2\sigma \cdot \frac{1}{n}\sum_{i=1}^n\xivar_i h_{\fhatn}(\bw_i)\right)\\
&\quad+ \eta^{-1}\left(  - \frac{1}{2n}\sum_{i=1}^n h_{\fhatn}(\bw_i) \cdot(g - \gst + \upbeta_{\fhatn})(\yvar_i))\right)\\
&= \eta^{-1}\left(2\eta\Exp[h_{\fhatn}(\bw)^2]  - \frac{1}{n} \|h_{\fhatn}(\bw_{1:n})\|_{2,n}^2 -  2\sigma \cdot \frac{1}{n}\sum_{i=1}^n\xivar_i h_{\fhatn}(\bw_i)\right)\\
&\quad+ \eta^{-1}\left(- \frac{\eta}{n} \Exp[h_{\fhatn}(\bw)^2]  - \frac{1}{2n}\sum_{i=1}^n h_{\fhatn}(\bw_i) \cdot(g - \gst + \upbeta_{\fhatn})(\yvar_i))\right).
\end{align*}
Finally, note that $h_{\fhatn} \in \Fcent$. As established above, $g -\upbeta_{\fhatn} \in \cG$ by conditional completeness, so $g - \gst + \upbeta_{\fhatn} \in \Gcent$. In fact, the condition $\Rtrain(\fhatn,\ghatn) \le \gamma^2$ implies via \Cref{lem:excess_decomp} that $\Exp[(g -\upbeta_{\fhatn} - \gst)^2] \le \gamma^2$, so that $g - \gst + \upbeta_{\fhatn} \in \Gcent(\gamma)$.
 Thus, we may pass to a supremum on the right-hand side equations:
\begin{align*}
\Rtrain[\fhatn] &\le  \eta^{-1}\sup_{h_1 \in \Fcent} \left(2\eta\Exp[h_{1}(\bw)^2]  - \frac{1}{n} \|h_{1}(\bw_{1:n})\|_{2,n}^2 -  2\sigma \cdot \frac{1}{n}\sum_{i=1}^n\xivar_i h_{1}(\bw_i)\right)\\
&\quad+ \eta^{-1}\sup_{h_1 \in \Fcent, h_2 \in \Gcent}\left(- \frac{\eta}{n} \Exp[h_1(\bw)^2]  - \frac{1}{2n}\sum_{i=1}^n h_1(\bw_i) \cdot h_2(\bw_i)\right).
\end{align*}
Selecting $\eta = 1/8$ conclues.
\newcommand{\Gcomp}{\mathscr{G}}
\newcommand{\Rcomp}{\mathscr{R}}
\newcommand{\Gcomphat}{\hat{\Gcomp}}
\newcommand{\Rcomphat}{\hat{\Rcomp}}
\subsection{Proof of \Cref{prop:dudley_holder}}\label{sec:prop:dudley_holder}

\newcommand{\nrmpn}[2]{\|#1\|_{#2,n}}
\newcommand{\nrmtwon}[1]{\|#1\|_{2,n}}
\newcommand{\nrminfn}[1]{\|#1\|_{\infty,n}}

This section establishes the H\"older-style inequality for Rademacher complexities of product classes. For completeness, we begin by reproducing a standard bound on the Rademacher complexity of finite function classes.
\begin{lemma}\label{lem:finite_expectation_upper_bound} Let $\bbV \subset \R^n$ be a finite set. Then, 
\begin{align*}
 \Rcomp_n(\bbV) \le \rad_2(\bbV)\min\{1,\sqrt{2 \log|\bbV|/n}\}
\end{align*}
The same bounds also hold for $\Gcomphat_n,\Gcomp_n$, and more generally, whenever the variables $\epsvar_i$ in the definition of the Rademacher complexities are replaced by arbitrary $1$-subGaussian  variables.\footnote{Recall a variable $\epsvar$ is $1$-subGaussian if, for all $\lambda \ge 0$, $\log \Exp[\exp(\lambda \epsvar)] \le \lambda^2/2$.} 
\end{lemma}
\begin{proof}
Let us bound $\Rcomphat_n$, with $\epsvar_i$ replaced by arbitrary $1$-subGaussian random variables. 
Recall the definition of a $1$-subGaussian variable $\epsvar$: $\log \Exp[\exp(\lambda \epsvar)] \le \frac{1}{2}\lambda^2$ (it is standard that Gaussian random variables and Rademacher variables satisfy this inequality). By Taylor expanding $\log \Exp[\exp(\lambda \epsvar)] = \log (1+\sum_{i\ge 1} (\lambda\Exp[\epsvar^i]/i!) $, it follows that $\Exp[\epsvar] =0$, and $\Exp[\epsvar^2] \le 1$.  Hence, by Cauchy-Schwartz,
\begin{align*}
 \Rcomp_n(\bbV) = \Exp[\sup_{v \in \bbV}\frac{1}{n}\sum_{i=1}^n \epsvar_i v_i] \le \sup_{v \in \bbV}\|v\|_{2,n} \cdot \sqrt{\frac{1}{n}\Exp\sum_{i=1}^n \epsvar_i^2} \le \rad_2(\bbV).
\end{align*}
The second bound is a consequence of standard sub-Gaussian maximal inequality (see, e.g. \citet[Theorem 2.5]{} in Lugosi) and the fact that  $\frac{1}{n}\sum_{i=1}^n \epsvar_i v_i$ is $\frac{1}{n}\|v\|_{2,n}^2$-subGaussian (e.g.,  the discussion in \citet[Chapter 2.3]{boucheron2013concentration}).
\end{proof}

\newcommand{\radinf}{\rad_{\infty}}
\newcommand{\radtwon}{R_{2}}

    We now turn to the proof of \Cref{prop:dudley_holder}. We first state two useful lemmas. The first is a direct consequence of H\"older's inequality and the fact that $\|\cdot\|_{p,n} \le \|\cdot\|_{p',n}$ for $p' \ge p $.
    \begin{lemma}[Variant of H\"older's inequality]\label{lem:holder_odot} For any $p,q \ge 2$ satisfying  $1/p+1/q \le 1/2$, 
    \begin{align}
    \forall v,u \in \R^n, \quad \|v\odot u\|_{2,n} \le \|v\|_{p,n}\cdot\|u\|_{q,n} \label{eq:Q_holder}
    \end{align}
    \end{lemma}
    The second bounds the Rademacher complexity of Hadamard products in terms of a finite cover. It is stated with a factor of $2$ for convenience when applied below.
    \begin{lemma}\label{claim:Dcomp_reduce_to_discrete} Let $\bbV,\bbU \subset \R^n$, $p,q \ge 2$ satisfy $1/p+1/q \le 1/2$, $\updelta_1,\updelta_2 \ge 0$, and let $\bbV'$ be a $2\updelta_1$-net of $\bbV$ in $\|\cdot\|_{p,n}$ and $\bbU'$ an $2\updelta_2$-net of $\bbU$ in $\|\cdot\|_{q,n}$. Then, 
    \begin{align*}
    \Rcomp_n(\bbV \odot \bbU) \le \Rcomp_n(\bbV' \odot \bbU') + 2(\rad_p(\bbV)\updelta_1  + \rad_q(\bbU)\updelta_2),
    \end{align*}
    The above bound also holds for the $\Gcomp_n$, and more generaly, any analogous complexity using suprema over $1$-subGaussian random variables. 
    \end{lemma}
    \begin{proof}[Proof of \Cref{claim:Dcomp_reduce_to_discrete}]
    Observe that, for any $(v,u) \in \bbV \times \bbU$, there exists a $(v',u') \in \bbV' \times \bbU'$ with $\|v - v'\|_{p,n} \le \updelta_1$ and $\|u - u'\|_{q,n} \le \updelta_2$. Hence, by \Cref{eq:Q_holder} followed by \Cref{eq:del_bounds_thing_veps},
    \begin{align*}
    \|v \odot u - (v')\odot(u')\|_{2,n} &\le  \|v' \odot (v - v')\|_{2,n} + \|v \odot (u-u)'\|_{2,n} \\
    &\le 2(\rad_q(\bbU)\updelta_1 + \rad_p(\bbV)\updelta_2)
    \end{align*}
    Hence,
    \begin{align*}
    \Rcomp_n(\bbV \odot \bbU) &= \Exp\sup_{(v,u) \in \bbV \times \bbU} \frac{1}{n}\sum_{i=1}^n \epsvar_i u_i v_i\\
    &\le \Exp\sup_{(v,u) \in \bbV_{j_1} \times \bbU_{k_1}} \frac{1}{n}\sum_{i=1}^n \epsvar_i u_i v_i  - \Exp\sup_{(v,u) \in \bbV_{1} \times \bbU_{j_1}}\inf_{(v',u') \in \bbV \times \bbU} \frac{1}{n}\sum_{i=1}^n \epsvar_i (u_i v_i - u_i' v_i')\\
    &\le \Rcomp_n(\bbV_{j_1} \odot \bbU_{k_1})  - \Exp\sup_{w:\|w\|_{2,n} \le 2(\rad_q(\bbU)\updelta_1 + \rad_p(\bbV)\updelta_2)} \frac{1}{n}\sum_{i=1}^n \epsvar_i w_i\\
    &\le \Rcomp_n(\bbV_{j_1} \odot \bbU_{k_1}) + 2(\rad_q(\bbU)\updelta_1 + \rad_p(\bbV)\updelta_2) \tag{\Cref{lem:finite_expectation_upper_bound}},
    \end{align*}
    as needed.
    \end{proof}

        We now turn to the proof of the main result of this section.
        \begin{proof}[Proof of \Cref{prop:dudley_holder}] 
         Recall that $\rad_p(\bbV)$ and $\rad_q(\bbU)$ denote the radii of $\bbV$ and $\bbU$ in the $\|\cdot\|_{p,n}$ and $\|\cdot\|_{q,n}$ norms, respectively, assuming $\zero \in \bbU \cap \bbV$. The only properties of Rademacher variables we use are those assumed by \Cref{lem:finite_expectation_upper_bound}, i.e. $1$-subGaussianity, so our bound holds for Gaussian complexity and other subGaussian ensembles.

          We begin with the classical construction of Dudley's integral. Fix $\updelta_1 \le \rad_p(\bbV),\updelta_2 \le \rad_q(\bbU)$ 
         \begin{align*}
         j_1 := \sup\{j:2^{-j}\rad_p(\bbV) \ge \updelta_1\}, \quad k_1 := \sup\{k:2^-k \rad_q(\bbU) \ge \updelta_2\}
         \end{align*}

         For each $j \in [j_1]\cup\{0\}$, let $\bbV_j$ denote a minimal $2^{-j}\rad_p(\bbV)$ covering of $\bbV$ in $\|\cdot\|_{p,n}$. Note that since $\zero \in \bbV$, we can take $\bbV_0 = \{\bzero\}$, so $|\bbV_0| = 0$.  Define $\pi_{j_1}(v;\bbV) = v$ for $v \in \bbV_{j_1}$, and recursively set $\pi_{j-1}(v;\bbV) \in \argmin_{v' \in \bbV_{j-1}} \|v' - \pi_j(v;\bbV)\|_{p,n}$. Set $\Delta_0(v;\bbV) = \pi_0(v;\bbV)$, and for $j \ge 1$, set $\Delta_j(v;\bbV) := \pi_{j}(v;\bbV) - \pi_{j-1}(v;\bbV)$ for $j \in [j_1]$. 
        Repeat the construction to construct $\bbV_k$,  projection $\pi_k(\bbU)$, and remainders $\Delta_k(u;\bbU)$ analogously, but replacing $\|\cdot\|_{p,n}$ the its conjugate $\|\cdot\|_{q,n}$. The for all $u \in \bbU_{j_1}, v \in \bbU_{k_1}$.
        \begin{align}
        v = \sum_{j=0}^{j_{1}}\Delta_j(v;\bbV), \quad
        u = \sum_{k=0}^{k_1}\Delta_k(u;\bbU). \label{eq:delta_decomp}
        \end{align}

        Lastly, as a shorthand, set
        \begin{align*}
        \veps_{j}(\bbV) := \rad_p(\bbV)2^{-j}, \quad \veps_{k}(\bbU) = 2^{-k}\rad_q(\bbU),
        \end{align*}
        noting that 
        \begin{align}\label{eq:del_bounds_thing_veps}
        \updelta_1 \le \veps_{j_1}(\bbV) \le 2\updelta_1, \quad \updelta_2\le  \veps_{k_1}(\bbU)\le 2\updelta_2
        \end{align}

        By \Cref{claim:Dcomp_reduce_to_discrete}, it suffices to bound $\Rcomphat_n(\bbV_{j_1}\odot \bbU_{k_1})$. For, $(u,v) \in \bbV_{j_1}\times \bbU_{k_1}$, 
        \begin{align*}
        \Rcomphat_n(\bbV_{j_1}\odot \bbU_{k_1}) &= \sup_{(u,v) \in \bbV_{j_1}\times \bbU_{k_1}} \frac{1}{n}\sum_{i=1}^n \epsvar_i u\odot v \\
        &= \sup_{(u,v) \in \bbV_{j_1}\times \bbU_{k_1}}\frac{1}{n}\sum_{i=1}^n \sum_{j=0}^{j_1}\sum_{k=0}^{k_1} \epsvar_i \Delta_j(v;\bbU) \odot \Delta_k(u;\bbU)\\
        &\le \sum_{j=0}^{j_1}\sum_{k=0}^{k_1} \sup_{(u,v) \in \bbV_{j_1}\times \bbU_{k_1}}\frac{1}{n}\sum_{i=1}^n  \epsvar_i \Delta_j(v;\bbU) \odot\Delta_k(u;\bbU)\\
        &\le \sum_{j=0}^{j_1}\sum_{k=1}^{k_0} \Rcomphat_n(\bbW_{j,k})
        \end{align*}
        where we define $\bbW_{j,k} := \{\Delta_j(v;\bbU) \odot\Delta_k(u;\bbU): v \in \bbV_{j_1},u \in \bbU_{k_1}\}$. Taking expectations yields
        \begin{align}
        \Rcomp_n(\bbV_{j_1}\odot \bbU_{k_1}) \le \sum_{j=0}^{j_1}\sum_{k=0}^{k_1} \Rcomp_n(\bbW_{j,k}) \label{eq:radcomp_decomp}
        \end{align}

        From our construnction, we can bound
        \begin{align*}
        \log |\bbW_{j,k}| \le \log( |\bbV_j||\bbV_{j-1}||\bbU_k||\bbU_{k-1}|) &\le \log( |\bbV_j|^2|\bbU_k|^2) \le \log(2|\bbV_j|^2|\bbU_k|^2) \\
        &= 2(\mnum_p(\bbV;\eps_j(\bbV)) + \mnum_q(\bbU;\eps_k(\bbU)))
        \end{align*}
        where we use $1 \le |\bbV_{j-1}| \le |\bbV_{j}| = \mnum_p(\bbV;\eps_j(\bbV))$, and similarly for the sets $\bbU_k$. Moreover, \Cref{eq:Q_holder} 
        \begin{align*}
        \rad_2(\bbW_{j,k}) &= \sup\{\|\Delta_j(v;\bbU) \odot\Delta_k(u;\bbU)\|_{2,n}: v \in \bbV_{j_1},u \in \bbU_{k_1}\}\\
        &\le \sup_{v \in \bbV_{j}} \|\Delta_j(v;\bbU)\|_{p,n} \cdot \sup_{u \in \bbU_{k}}\|\Delta_k(u;\bbU)\|_{q,n}\\
        &\le \eps_j(\bbV)\eps_k(\bbU)
        \end{align*}
        Thus, \Cref{lem:finite_expectation_upper_bound} yields
        \begin{align*}
        \Rcomp_n(\bbW_{j,k}) &\le \eps_j(\bbV)\eps_k(\bbU)\sqrt{\frac{2}{n}} \cdot\sqrt{( 2(\mnum_p(\bbV;\eps_j(\bbV)) + \mnum_q(\bbU;\eps_k(\bbU)))}\\
        &\le \frac{2}{\sqrt{n}}\left(\eps_j(\bbV)\eps_k(\bbU)\sqrt{\mnum_p(\bbV;\eps_j(\bbV)} + \eps_j(\bbV)\eps_k(\bbU)\sqrt{\mnum_q(\bbU;\eps_k(\bbU))}\right)\\
        &\le \frac{2}{\sqrt{n}}\left(\rad_q(\bbU)2^{-k}\eps_j(\bbV)\sqrt{\mnum_p(\bbV;\eps_j(\bbV)} + \rad_p(\bbV)2^{-j}\eps_k(\bbU)\sqrt{\mnum_q(\bbU;\eps_k(\bbU))}\right)\\
        \end{align*}
        Hence, \Cref{eq:radcomp_decomp} and evaluating convergent sums yields
        \begin{align*}
        \Rcomp_n(\bbV_{j_1}\odot \bbU_{k_1}) &\le \frac{2}{\sqrt{n}}\sum_{j=0}^{j_1}\sum_{k=0}^{k_1}  \left(\rad_q(\bbU)2^{-k}\eps_j(\bbV)\sqrt{\mnum_p(\bbV;\eps_j(\bbV)} + \rad_p(\bbV)2^{-j}\eps_k(\bbU)\sqrt{\mnum_q(\bbU;\eps_k(\bbU))}\right)\\
        &\le  \frac{4\rad_q(\bbU)}{\sqrt{n}}\sum_{j=1}^{j_1}\eps_j(\bbV)\sqrt{\mnum_p(\bbV;\eps_j(\bbV))} +\frac{4\rad_p(\bbV)}{\sqrt{n}}\sum_{k=1}^{k_1}\eps_k(\bbU)\sqrt{\mnum_q(\bbU;\eps_k(\bbU))}
        \end{align*}
        where in the second-to-last line, we use that we have $\mnum_p(\bbV;\eps_0(\bbU)) = \log |\bbV_0| = 0$. 

        To simplify, we invoke the following claim. 
        \begin{claim}[Sum-to-Integral Coversion]\label{claim:to_integral} Let $\phi$ be a non-increasing function, and let $\eps_j = 2^{-j}R$ for some $R > 0$. Then, for $j_a \le j_b$,
        \begin{align*}
        \sum_{j=  j_a}^{j_b}\eps_j\phi(\eps_j) \le 2\int_{\eps_{j_b +1}}^{\eps_{j_a}} \phi(\eps)\rmd \eps = \int_{\eps_{j_b}}^{2\eps_{j_a}} \phi(\eps/2)\rmd \eps.
        \end{align*} 
        \end{claim}
        \begin{proof} The first inequality follows since $\phi$ is non-increasing, and the second line uses a change of variables
        \begin{align*}
        \sum_{j=  j_a}^{j_b}\eps_j\phi(\eps_j) &\le \sum_{j=  j_a}^{j_b}\frac{\eps_j}{\eps_j - \eps_{j+1}} \int_{\eps_{j+1}}^{\eps_j}\phi(\eps)\rmd \eps = 2\sum_{j=  j_a}^{j_b}\int_{\eps_{j+1}}^{\eps_j}\phi(\eps)\rmd \eps \le 2\int_{\eps_{j_b +1}}^{\eps_{j_a}} \phi(\eps)\rmd \eps\\
        &= \int_{2\eps_{j_b +1}}^{2\eps_{j_a}} \phi(\eps/2)\rmd \eps = \int_{\eps_{j_b}}^{2\eps_{j_a}} \phi(\eps/2)\rmd \eps.
        \end{align*}
        \end{proof}

        In particular, since metric entropies are non-increasing in their scale factors, 
        \begin{align*}
        \sum_{j=1}^{j_1}\eps_j(\bbV)\sqrt{\mnum_p(\bbV;\eps_j(\bbV))} &\le \sum_{j=1}^{j_1}\frac{\eps_j(\bbV)}{\eps_{j}(\bbV) - \eps_{j+1}(\bbV)}\int_{\eps_{j+1}(\bbV)}^{\eps_j(\bbV)}\sqrt{\mnum_p(\bbV;\veps)}\rmd \veps\\
        &= 2\sum_{j=1}^{j_1}\int_{\eps_{j+1}(\bbV)}^{\eps_j(\bbV)}\sqrt{\mnum_p(\bbV;\veps)}\rmd \veps = 2\int_{\eps_{j_1+1}(\bbV)}^{\eps_1(\bbV)}\sqrt{\mnum_p(\bbV;\veps)}\rmd \veps\\
        &= 2\int_{\eps_{j_1}(\bbV)/2}^{\rad_p(\bbV)/2}\sqrt{\mnum_p(\bbV;\veps)}\rmd \veps \\
        &= \int_{\eps_{j_1}(\bbV)}^{\rad_p(\bbV)}\sqrt{\mnum_p(\bbV;\veps/2)}\rmd \veps \le \int_{\updelta_1}^{\rad_p(\bbV)}\sqrt{\mnum_p(\bbV;\veps/2)}\rmd \veps,
        \end{align*}
        where the last inequality uses \Cref{eq:del_bounds_thing_veps}.

        Invoking a similar bound for the analogus $\bbU$-term, we conclude
        \begin{align*}
        \Rcomp_n(\bbV_{j_1}\odot \bbU_{k_1}) &\le  \frac{4\rad_q(\bbU)}{\sqrt{n}}\int_{\updelta_1}^{\rad_p(\bbV)}\sqrt{\mnum_p(\bbV;\veps/2)}\rmd \veps +\frac{4\rad_p(\bbV)}{\sqrt{n}}\int_{\updelta_2}^{\rad_q(\bbU)}\sqrt{\mnum_q(\bbU;\veps/2)}\rmd \veps
        \end{align*}
        Combining with \Cref{claim:Dcomp_reduce_to_discrete} and taking the infinum over valid $\updelta_1,\updelta_2$,
        \begin{align*}
        \Rcomp_n(\bbV\odot \bbU) &\le  \rad_q(\bbU)\underbrace{\inf_{\updelta_1 \le \rad_p(\bbV)}\left(2\updelta_1 + \tfrac{4}{\sqrt{n}}\int_{\updelta_1}^{\rad_p(\bbV)}\sqrt{\mnum_p(\bbV;\veps/2)}\rmd \veps\right)}_{\Dudfun_{n,p}(\bbV)} \\
        &\qquad+ \rad_p(\bbV) \underbrace{\inf_{\updelta_2 \le \rad_q(\bbU)} \left(2\updelta_2 + \tfrac{4}{\sqrt{n}}\int_{\updelta_2}^{\rad_q(\bbU)}\sqrt{\mnum_q(\bbU;\veps/2)}\rmd \veps\right)}_{\Dudfun_{n,p}(\bbU)}
        \end{align*}
    \end{proof}

\subsection{Derivation of \Cref{lem:Dudley_ub} from \Cref{prop:dudley_holder}}\label{sec:lem:Dudley_ub}
We consider the Rademacher complexity, as \Cref{prop:dudley_holder} guarantees the same holds of the Gaussian complexity. Let $\bbU = \{w \mapsto (1,1,\dots,1) \in \R^n\}$. Applying  \Cref{prop:dudley_holder} with the square-H\"older conjugates $p = 2$ and $q = \infty$. As the construction of $\bbU$ ensures $\bbV \odot \bbU = \bbV$, this yields
\begin{align*}
\Raden(\bbV) = \Raden(\bbV\odot \bbU) \le \rad_\infty(\bbU)\Dudfun_{n,2}(\bbV) + \rad_2(\bbV)\Dudfun_{n,\infty}(\bbU).
\end{align*}
Notice that $\rad_\infty(\bbU) = \|(1,1,\dots,1)\|_{\infty} = 1$. Moreover,  the covering number of $\bbU$ is $1$, so its log-covering numbers are zero. Thus, the integral in $\Dudfun_{n,\infty}(\bbU)$ vanishes. This concludes the demonstration that
\begin{align}
\Raden(\bbV) \le \Dudfun_{n,2}(\bbV) \label{eq:Raden_Dudfun}
\end{align}
As a consequence,
\begin{align*}
\delnst(\cH,c) &:= \inf\left\{r: \Raden(\cH[r,w_{1:n}]) \le \frac{r^2}{2c}\right\}\\
&\le \inf\left\{r: \Dudfun_{n,2}(\cH[r,w_{1:n}]) \le \frac{r^2}{2c}\right\} \tag{\Cref{eq:Raden_Dudfun}}\\
&:= \deldud(\cH,c).
\end{align*}

\section{Technical Tools}\label{app:tech_tools}
This section enumerates the accompanying technical results applied in the proofs in \Cref{sec:proof_main_props}. Whereas \Cref{sec:proof_main_props} highlights conceptually novel arguments, this section massages more standard material into the most convenient form for adoption in the prior section. 

The results in this section are stated at the following level of generality: Throughout, let $\cH: \cW \to \cR$ denote a class of functions, and $P$ be a measure over $\cW$, with $\Var$ and $\Exp$ its corresponding expectation and variance functionals with respect to $P$. We say $\cH$ contains zero if the function $h_0(w) \equiv 0$ lies in $\cH$. Many definitions results below involve star-hulls and convex-hulls. 

\begin{definition} Let $\bbV \subset \R^n$. We let $\conv(\bbV)$ denote its convex hull and $\starhull(\bbV) := \{t\cdot v : t \in [0,1], v \in \bbV\}$. Similarly, for a function class $\cH:[0,1] \to \R$, we let $\starhull(\cH):=\{t \cdot h, t \in [0,1], h \in \cH\}$, \emph{convex hull} as $\conv(\cH)$ as the minimal convex set containing $\cH$.
\end{definition}

\subsection{Basic Empirical Process Results}

\paragraph{Properties of Rademacher and Gaussian complexities.} We recall a couple standard facts about the Rademacher complexity. First is that Rademacher complexity is invariant under the convex hull operation, and also under the star-hull operation if the set contains zero.

\begin{lemma}[Convex Hulls]\label{lem:rad_conv_hull} If $\bbV' \subset \bbV$, $\Raden(\bbV') \le \Raden(\bbV)$. Moreover, $\Raden(\bbV) = \Raden(\conv(\bbV))$, and if in addition, $\zero \in \bbV$, $\Raden(\bbV) = \Raden(\starhull(\bbV))$.
\end{lemma}
\begin{proof}Recall $\Raden(\bbV) := \frac{1}{n}\Exp_{\epsvar}\sup_{v \in \bbV}\sum_{i=1}^n \epsvar_iv_i$. It is then clear that if $\bbV' \subset \bbV$, then $\Raden(\bbV') \le \Raden(\bbV)$. Then is establishes the first point. The second follows because the maximum of a linear function $v \mapsto \sum_{i=1}^n \epsvar_iv_i$ occurs on the extreme points of $\bbV$, which are the same as those of $\conv(\bbV)$. Lastly, if $\bbV$ contains $\zero$, $\conv(\bbV) \supset \starhull(\bbV) \supset \bbV$.
\end{proof}
Next, we state a classical Lipschitz contraction for Rademacher complexity.
\begin{lemma}[Rademacher Contraction, Lemma 29 in \cite{rakhlinnotes} ]\label{lem:rad_lipschitz} Let $\phi$ be any $L$-Lipschitz function, and given $\bbV \subset \R^n$, let $\phi(\bbV) := \{\phi(v): v \in \bbV\}$. Then, $\Raden(\phi(\bbV)) \le L\Raden(\bbV)$. 
\end{lemma}


The following lemma is standard (see, e.g. \citet[Chapter 14]{wainwright2019high} or, examine the proof of \citet[Lemma 30]{rakhlinnotes}).
\begin{lemma}\label{lem:star_shaped} Let $\cH$ be star-shaped. Then $\cH(cr) \subset c\cH(r)$ for all $c \ge 1$. Hence, for $c \ge 1$, it holds that $\delnst(\cH,cB) \le c \delnst(\cH,B)$, and similarly $\rhonst(\cH,c\sigma) \le c \rhonst(\cH,\sigma)$.
\end{lemma}
The following lemma shows a similar property for the Dudley functional, this time without the constraint that $\cH$ is star-shaped.
\begin{lemma}\label{lem:dudley_star_prop} For any class $\cH$ and $c \ge 1$, it holds that $\Dudfun_{n,2}(\cH[r,w_{1:n}]) \le c^{-1}\Dudfun_{n,2}(\cH[cr,w_{1:n}])$. Thus, for any $c \ge 1$ and $B > 0$, $\deldud(\cH,cB) \le c\deldud(\cH,B)$. 
\end{lemma}
\begin{proof} 
We then have, recalling \Cref{defn:dudfunc} and using $\rad_2(\cH[r,w_{1:n}])) = r$ that
\begin{align*}
\Dudfun_{n,2}(\cH[r,w_{1:n}]) &:= \inf_{\updelta \le r}\left(2\updelta + \frac{4}{\sqrt{n}}\int_{\updelta}^{r}\sqrt{\mnum_2(\bbV;\veps/2)}\rmd \veps\right)\\
&= \inf_{\updelta \le r}\left(2\frac{c\updelta}{c} + \frac{4}{c\sqrt{n}}\int_{c\updelta}^{cr}\sqrt{\mnum_2(\bbV; \veps/2c)}\rmd \veps\right)\\
&= \frac{1}{c}\inf_{\updelta \le cr}\left(2\updelta + \frac{4}{\sqrt{n}}\int_{\updelta}^{cr}\sqrt{\mnum_2(\bbV; \veps/2c)}\rmd \veps\right)\\
&\overset{(i)}{\ge} \frac{1}{c}\inf_{\updelta \le cr}\left(2\updelta + \frac{4}{\sqrt{n}}\int_{\updelta}^{cr}\sqrt{\mnum_2(\bbV; \veps/2)}\rmd \veps\right)\\
&= \frac{1}{c}\Dudfun_{n,2}(\cH[cr,w_{1:n}]),
\end{align*}
where in $(i)$ we use anti-monotonicity of covering numbers $\mnum_2(\bbV; \veps/2c) \le \mnum_2(\bbV; \veps/2)$ for $c \ge 1$. Recall from \Cref{defn:dudfunc} the definition
\begin{align*}
\deldud(\cH,cB) &:= \inf\left\{r: \sup_{w_{1:n}}\Dudfun_{n,2}(\cH[r,w_{1:n}]) \le \frac{r^2}{2cB} \right\}\\
&= \inf\left\{cr: \sup_{w_{1:n}}\Dudfun_{n,2}(\cH[cr,w_{1:n}]) \le \frac{cr^2}{2B} \right\}\\
&= c\inf\left\{r: c^{-1}\sup_{w_{1:n}}\Dudfun_{n,2}(\cH[cr,w_{1:n}]) \le \frac{r^2}{2B} \right\}\\
&\le c\inf\left\{r: \sup_{w_{1:n}}\Dudfun_{n,2}(\cH[r,w_{1:n}]) \le \frac{r^2}{2B} \right\}\\\
&= c \deldud(\cH,cB),
\end{align*}
where the inequality above follows from the first part of the lemma.
\end{proof}

\begin{lemma}\label{lem:dud_translation} Let $\cH:\cW \to \R$ be an arbitrary class of functions, and let $h_0:\cW \to \R$ be arbitrary. Then, the class $\cH + h_0 := \{h + h_0: h \in \cH\}$ satisfies, for all $n \in \N$,  $c > 0$, $q \ge 1$, and $w_{1:n} \in \cW^n$ the equalities,
\begin{align*}
\deldud(\cH + h_0,c) = \deldud(\cH,c), \quad \Dudfun_{n,q}((\cH + h_0)[w_{1:n}]) = \Dudfun_{n,q}(\cH[w_{1:n}]).
\end{align*}
In particular, recalling the notation of \Cref{sec:proof_thm_main_guarantee}, for any $c > 0$, $\deldud(\Fcent,c) = \deldud(\Fbet,c)$ and $\sup_{w_{1:n}}\Dudfun_{n,\infty}((\Gcentil)[w_{1:n}])^2 = \sup_{w_{1:n}}\Dudfun_{n,\infty}(\cG[w_{1:n}])^2$.
\end{lemma}
\begin{proof} The proof is immediate from the fact that translation by a single element leaves the covering numbers, and hence metric entropies, unchanged. 
\end{proof}

\subsection{Deviation Inequalities for Empirical Processes} The following is a standard maximal inequality for empirical processes.
\begin{lemma}[Empirical Process Inequality, Theorem 2.3 in \cite{bousquet2002bennett}]\label{lem:saras_lemma} Let $\cH: \cW \to [-B,B]$ and let $P$ be a measure on $\cW$ such that $\sup_{h \in \cH}|\Exp_{\bw \sim P}[h(\bw)]| = 0$. Let $\bZ := \frac{1}{n}\sup_{h \in \cH}\sum_{i=1}^n h(\bw_i)$, and let $r^2 := \sup_{h \in \cH}\Exp_{\bw \sim P}[h(\bw)^2]$. Then, for any choice of parameter $\epsilon > 0$.
\begin{align*}
\Pr_{\bw_{1:n} \sim P}\left[ \bZ \ge (1+\epsilon)\Exp[\bZ]  \ge r \sqrt{\frac{2\log(1/\delta)}{n}}  + \left(\frac{1}{\epsilon} + \frac{1}{3} \right)\frac{B}{n} \log(1/\delta)\right] \le \delta.
\end{align*}
\end{lemma}
By examining the proof of  Theorem 2.3  in \cite{bousquet2002bennett} from Theorem 2.1 in that same work, one can check that the concusion of  \Cref{lem:saras_lemma} holds verbtaim in the folllowing more general setup: the class of  functions $\tilde\cH: \cW \times [n] \to [-B,B]$ are index-dependent,  the process is $\bZ := \frac{1}{n}\sup_{\tilde h \in \tilde \cH}\sum_{i=1}^n \tilde h(\bw_i, i)$, and where we define $r^2 := \frac{1}{n}\Exp[\sum_{i=1}^n \tilde h(\bw_i,i)^2]$ as the average variance. A special case of this generalization applies to Rademacher processes $\bZ := \frac{1}{n}\sup_{h \in \cH}\sum_{i=1}^n \epsvar_i h(w_i)$, where $\epsvar_i$ plays the roll of the random variable $\bw_i$, and where $\tilde{h}(\epsvar_i,i) = \epsvar_i h(w_i)$. 
\begin{lemma}\label{lem:rademacher_concentration}Let $\cH: \cW \to [-B,B]$.  Fix any $w_{1:n} \in \cW$, and let $\bZ := \frac{1}{n}\sup_{h \in \cH}\sum_{i=1}^n \epsvar_{i} h(w_i)$, and let $r^2 := \sup_{h \in \cH}\|h(w_{1:n})\|_{2,n}^2$. Then, for any choice of parameter $\epsilon > 0$.
\begin{align*}
\Pr_{\bw_{1:n} \sim P}\left[ \bZ \ge (1+\epsilon)\Exp[\bZ]  \ge r \sqrt{\frac{2\log(1/\delta)}{n}}  + \left(\frac{1}{\epsilon} + \frac{1}{3} \right)\frac{B}{n} \log(1/\delta)\right] \le \delta.
\end{align*}
\end{lemma}
We shall also need a related lemma that bounds deviations in terms of Rademacher complexity.
\begin{lemma}[Uniform Convergence, Theorem 2.1 in \cite{bartlett2005local}]\label{lem:rad_uniform_convergence} Let $\cH: \cW \to [-B,B]$ be a family of uniformly bounded functions with $|h|\le B$ and $\sup_{h \in \cH}\Var[h^2] \le r^2$ and let $P$ be a measure over $\cW$. Then, with probability at least $1 - \delta$, any 
\begin{align*}
\sup_{h \in \cH}\frac{1}{n}\sum_{n=1}^n \Exp_{\bw \sim P}[h(\bw)] - h(\bw_i)  \le \left(6\Exp_{\bw_{1:n}\sim P}[\Raden(\cH[\bw_{1:n}])] + r\sqrt{\frac{2 \log(1/\delta)}{n}} + \frac{11 B\log(1/\delta)}{n}\right)
\end{align*}
In particular, if $\cH = \cH(r)$ for some $r$, then the above holds for $\nu^2 = r^2$.
\end{lemma}

\subsection{Fixed-Design Guarantees}
This section concerns various measures of complexity for a function class when its arguments (``design points'') $w_{1:n}$ are treated as deterministic. We begin with the following general lemma, which abstracts away the function class $\cH[w_{1:n}]$ evaluated on the design points with a set $\bbV \subset \R^{n}$. This lemma measure ``offset complexities'', were a mean zero process involving $v \in \bbV$ is offset by norms $-\|v\|^2$. This lemma implies important consequences of this lemma for Gaussian and Rademacher compelxities.
\begin{lemma}[Fixed-Design Master Lemma]\label{lem:offset_master} Let $\bbV \subset \R^n$ be a containing $\zero \in \R^n$, with  $\bbV[r] := \{v \in \bbV: \|v\|_{2,n} \le r\}$, and  let $\Phi:\cU \to \R$ be an arbitrary function classes (possibly even of cardinality one). Let $\bu_1\dots,\bu_n$ be a random variables taking values in $\cU$, and define the processes
\begin{align*}
\bZ(r) &:=  \sup_{v \in \bbV[r]}\sup_{\phi \in \Phi}  \frac{1}{n}\sum_{i=1}^n v_i \phi(\bu_i) \tag{localized maximal process}\\
\mathbf{Y}(\tau) &:= \sup_{v \in \bbV}\sup_{\phi \in \Phi} \left\{\frac{2\tau}{n}\sum_{i=1}^n v_i \phi(\bu_i) - \frac{1}{2}\|v\|_{2,n}^2\right\}. \tag{offset maximal process}
\end{align*} 
Lastly, define a modification of $\bY$ which replaces the offset by $\|v\|_{2,n}^2$ with the offset by $r$:
\begin{align*}
\tilde{\bY}(r;\tau):= \sup_{v \in \bbV[r]}\sup_{\phi \in \Phi} \left\{\frac{2\tau}{n}\sum_{i=1}^n v_i \phi(\bu_i)\right\} - \frac{r^2}{2} = 2\tau\bZ(r) - \frac{r^2}{2}.
\end{align*}
Then, the following are true.
\begin{itemize}
    \item[(a)] With probability one,
    \begin{align*}
    \bY(\tau) = \sup_{r > 0} \tilde{\bY}(r;\tau).
    \end{align*}
    \item[(b)] With probability one, 
    \begin{align*}
    \bY(\tau) \le \inf\left\{r^2: \tilde{\bY}(r;\tau) \le \frac{r^2}{2}\right\}
    \end{align*}
    \item[(c)]
    Suppose that, for any choice of $r > 0$, $\bZ(r)$ satisfies the following concentration inequality with parameters $c_1 \ge 1$ and $c_2,c_3 > 0$:
\begin{align}
\Pr\left[ \mathbf{Z}(r) \le c_1 \Exp[\mathbf{Z}(\tau)] + c_2r \sqrt{\frac{ \log(1/\delta)}{n}}  + \frac{c_3\log(1/\delta)}{n}\right] \le \delta \label{eq:Z_prob},
\end{align}
\end{itemize}
Then, for any $\tau \ge 1$ and $\sigma > 0$, the following holds probability $1-\delta$, the following holds 
\begin{align*}
\mathbf{Y}(\sigma\tau) \lesssim  c_1^2\tau^2\bar\updelta_n(\sigma)^2 + \frac{(\tau^2\sigma^2c_2^2 +  (c_3/c_2)^2)\log(1/\delta)}{n},
\end{align*}
where 
\begin{align*}
\bar\updelta_{n}(\sigma) := \inf\left\{r: \Exp[\mathbf{Z}(r)] \le \frac{r^2}{2\sigma}\right\}.
\end{align*}
Thus, by itegrating,
\begin{align*}
\Exp \mathbf{Y}(\sigma\tau) \lesssim  c_1^2\tau^2\bar\updelta_n(\sigma)^2 + \frac{(\tau^2\sigma^2(c_2^2 + 1) + c_3)}{n},
\end{align*}
\end{lemma}
\begin{proof} We prove the lemma in parts. First, however, we observe that we may assume without loss of generality that $\bbV$ is star-shaped. Note that $\starhull(\bbV[r]) = \starhull(\bbV)[r]$.
Then, by the same logic as in the proof of \Cref{lem:rad_conv_hull}, the inclusion $0 \in \bbV[r]$ implies $\conv(\bbV[r]) \supset \starhull(\bbV[r]) = \starhull(\bbV)[r] \supset \bbV[r]$, which establishes
\begin{align*}
\bZ(r) &:=  \sup_{v \in \bbV[r]}\sup_{\phi \in \Phi}  \frac{1}{n}\sum_{i=1}^n v_i \phi(\bu_i) =  \sup_{v \in \conv(\bbV[r])}\sup_{\phi \in \Phi}  \frac{1}{n}\sum_{i=1}^n v_i \phi(\bu_i)\\
&\ge  \sup_{v \in \starhull(\bbV)[r]}\sup_{\phi \in \Phi}  \frac{1}{n}\sum_{i=1}^n v_i \phi(\bu_i) \ge\sup_{v \in \bbV[r]}\sup_{\phi \in \Phi}  \frac{1}{n}\sum_{i=1}^n v_i \phi(\bu_i) = \bZ(r),
\end{align*}
so 
\begin{align*}
\bZ(r) =  \sup_{v \in \starhull(\bbV)[r]}\sup_{\phi \in \Phi}  \frac{1}{n}\sum_{i=1}^n v_i \phi(\bu_i) .
\end{align*}
Similarly, one can show that 
\begin{align*}
\mathbf{Y}(\tau) &:= \sup_{v \in \starhull(\bbV)}\sup_{\phi \in \Phi} \left\{\frac{2\tau}{n}\sum_{i=1}^n v_i \phi(\bu_i) - \frac{1}{2}\|v\|_{2,n}^2\right\}, \quad \tilde{\bY}(r;\tau):= \sup_{v \in \starhull(\bbV)[r]}\sup_{\phi \in \Phi} \left\{\frac{2\tau}{n}\sum_{i=1}^n v_i \phi(\bu_i)\right\} - \frac{r^2}{2}.
\end{align*}
Hence, we can apply the entire lemma to $\starhull(\bbV)$, and then convert back to $\bbV$ by the above reduction.

\paragraph{Part (a).} As $r^2 \ge \|v\|_{2,n}^2$ for all $v \in \bbV[r]$, it is immediate that $\bY(\tau) \ge \sup_{r > 0} \tilde{\bY}(r;\tau)$. We prove the other direction. Fix an arbitrary $\epsilon > 0$ and suppose that that $v \in \bbV$ and $\phi \in \Phi$ satisfy
\begin{align*}
\frac{2\tau}{n}\sum_{i=1}^n v_i \phi(\bu_i) - \frac{1}{2}\|v\|_{2,n}^2 \ge \bY(\tau) - \epsilon.
\end{align*}
Letting $r := \|v\|_{2,n}$, it holds that 
\begin{align*}
\bY(\tau) - \epsilon \le  \frac{2\tau}{n}\sum_{i=1}^n v_i \phi(\bu_i) - \frac{r^2}{2} \le \tilde{\bY}(r;\tau) \le \sup_{r > 0}\tilde{\bY}(r;\tau).
\end{align*}
As $\epsilon$ was arbitrary, $\bY(\tau) \le \tilde{\bY}(r;\tau)$. 

\paragraph{Part (b).} Suppose that $r$ satisfies
\begin{align*}
\tilde{\bY}(r;\tau) \le 0.
\end{align*}
Fix $v \in \bbV$. Since $\bbV$ is star-shaped,  either $v \in \bbV[r]$, or $\alpha v \in \bbV[r]$ for  $\alpha = r/\|v\|_{2,n} < 1$. In the first case, 
\begin{align*}
\sup_{\phi \in \Phi} \frac{2\tau}{n}\sum_{i=1}^n v_i \phi(\bu_i) - \frac{1}{2}\|v\|_{2,n}^2 &\le \sup_{\phi \in \Phi} \frac{2\tau}{n}\sum_{i=1}^n v_i \phi(\bu_i)   \\
&= \sup_{\phi \in \Phi} \frac{2\tau}{n}\sum_{i=1}^n v_i \phi(\bu_i)  - r^2/2 + r^2/2 \\
&\le \underbrace{\tilde{\bY}(r;\tau)}_{\le \frac{r^2}{2}} + \frac{r^2}{2} \le \frac{r^2}{2}.
\end{align*}
In the second case, recalling $\alpha = r/\|v\|_{2,n} < 1$, 
\begin{align*}
\sup_{\phi \in \Phi} \frac{2\tau}{n}\sum_{i=1}^n v_i \phi(\bu_i)  - \frac{1}{2}\|v\|_{2,n}^2 &= \frac{1}{\alpha}\left(\sup_{\phi \in \Phi} \frac{2\tau}{n}\sum_{i=1}^n \alpha v_i \phi(\bu_i)  - \frac{1}{2\alpha} r^2\right)\\
&\le \frac{1}{\alpha}\left(\sup_{\phi \in \Phi} \frac{2\tau}{n}\sum_{i=1}^n \alpha v_i \phi(\bu_i)  - \frac{1}{2}r^2\right) + \frac{r^2}{2\alpha}(1- \frac{1}{\alpha})\\
&\le \frac{r^2}{2\alpha}+\frac{r^2}{2\alpha}(1- \frac{1}{\alpha}) = \frac{r^2}{2}\left(2\alpha^{-1} - \alpha^{-2}\right) \le \frac{r^2}{2}
\end{align*}
where the second inequality uses  $\max_{x}(2x - x^2) \le 1$
This concludes the proof of part (b).

\paragraph{Part (c).} Applying the AM-GM inequality twice to \Cref{eq:Z_prob}, the following holds with probability $1 - \delta$, 
\begin{align*}
 \mathbf{Z}(r) &\le c_1 \Exp[\mathbf{Z}(r)] + \frac{r^2}{4\tau \sigma} +  \frac{(2 \tau \sigma c_2^2 + c_3)\log(1/\delta)}{n}\\
 &= c_1 \Exp[\mathbf{Z}(r)] + \frac{r^2}{4\tau \sigma} +  \frac{(2 \tau^2 \sigma^2 c_2^2 + 
 \tau \sigma c_3)\log(1/\delta)}{n \cdot \tau \sigma}\\
 &\le c_1 \Exp[\mathbf{Z}(r)] + \frac{r^2}{4\tau \sigma} +  \frac{(3\tau^2 \sigma^2  c_2^2 + c_3^2/2c_2^2)\log(1/\delta)}{n\cdot\tau \sigma}.
\end{align*}
Consequently, setting $\alpha := 8c_1 \tau \ge 1$, 
\begin{align*}
\tilde{\bY}(r;\sigma \tau) = 2\sigma \tau\bZ(r) - \frac{r^2}{2} &\le  2c_1\sigma  \tau \Exp[\mathbf{Z}(r)] - \frac{r^2}{4} +  \frac{(3\tau^2 \sigma^2  c_2^2 + c_3^2/2c_2^2)\log(1/\delta)}{n}\\
&=  \frac{\sigma \alpha}{4} \left(\Exp[\mathbf{Z}(r)] - \frac{r^2}{\alpha 2\sigma}\right) + \frac{(3\tau^2 \sigma^2  c_2^2 + c_3^2/2c_2^2)\log(1/\delta)}{n} - \frac{r^2}{8}\\
\end{align*}
To conclude it suffices to show that the above expression is non-positive for the choice
\begin{align*}
r^2 :=  \max\left\{\frac{(24\tau^2 \sigma^2  c_2^2 + 4(c_3/c_2)^2)\log(1/\delta)}{n}, \alpha^2 (\bar{\updelta}_n(\sigma)^2+\epsilon)\right\}.
\end{align*}
Note that this choice makes the second term in the previous display vanishes, so 
\begin{align*}
\tilde{\bY}(r;\sigma \tau) &\le \frac{\sigma \alpha}{4} \left(\Exp[\mathbf{Z}(r)] - \frac{r^2}{2\alpha \sigma}\right)\\
&= \frac{\sigma \alpha}{4} \left(\Exp[\mathbf{Z}(\alpha \bar{\updelta}_n(\sigma)^2)] - \frac{\alpha \bar{\updelta}_n(\sigma)^2}{2\sigma}\right)\\
&\overset{(i)}{\le} \frac{\sigma \alpha}{4} \left(\alpha \Exp[\mathbf{Z}(\bar{\updelta}_n(\bbV,\Phi,\sigma)^2)] - \frac{\alpha \bar{\updelta}_n(\sigma)^2}{2\sigma}\right) \\
&= \frac{\sigma \alpha^2}{4} \left(\Exp[\mathbf{Z}(\bar{\updelta}_n(\sigma)^2)] - \frac{\bar{\updelta}_n(\sigma)^2}{2\sigma}\right) \le 0 \le \frac{r^2}{2}
\end{align*}
where $(i)$ uses that $\bbV$ is star-shaped, so $\Exp[\mathbf{Z}(\mu r)] \le \mu \Exp[\mathbf{Z}(r)]$ for any $\mu \ge 1$. The proof now follows by subsituting in $\alpha = 8c_1\tau$.
\end{proof}

\paragraph{Consequences of the Master Lemma.} Our first consequence is for Gaussian complexities.
\begin{lemma}[Offset Gaussian Complexity Bound]\label{lem:gauss_offset} Let $\cH: \cW \to \R$ be a function class containing the zero function. Fix any $\delta \in (0,1)$ and $ \sigma > 0$ and $\tau \ge 1$. Then, there exists a constant $c > 0$ such that
\begin{align*}
\sup_{w_{1:n}}\Pr_{\xivar_{1:n}}\left[\sup_{h \in \cH}   \frac{2\sigma \tau}{n}\sum_{i=1}^n \xivar_ih(w_i)  - \frac{1}{2n}\|h(w_{1:n})\|_{2,n}^2 > c\tau^2\left(\frac{\sigma^2\log(1/\delta)}{n} + \rhonst(\cH,\sigma)^2\right)\right] \le \delta,
\end{align*}
where above $\xivar_{1:n}$ are i.i.d. standard Normal. 
\end{lemma}
\begin{proof} Recall the set $\cH[r,w_{1:n}] = \{h \in \cH: \|h\|_{2,n}^2 \le r^2\}$. Then, the random variable
\begin{align*}
\bZ(r) := \frac{1}{n}\sup_{h \in \cH[r;w_{1:n}]}\sum_{i=1}^n \xivar_ih(w_i) 
\end{align*}
satisfies, by Gaussian-Lipschitz concentration (e.g. \citet[Chapter 2.3]{boucheron2013concentration} or \citet[Chapter 3]{wainwright2019high}),
\begin{align}
\Pr[\bZ(r_0) \ge \Exp_{\bxi} \bZ(r_0) + r_0\sqrt{2\log(1/\delta)/n} ] \le \delta. \label{eq:event_Zrnot}
\end{align}
The bound now follows from \Cref{lem:offset_master}, where $\bu_i = \bx_i$ and $\Phi$ is a singleton consisting of the identity function. 
\end{proof}
We establish a similar guarantee for Rademacher variables. 
\begin{lemma}[Offset Rademacher Complexity Bound]\label{lem:Rad_offset} Let $\cH: \cW \to [-B,B]$ be a function class containing zero, and let $\epsvar_{1:n}$ be i.i.d. Rademacher random variables. Then, for any $\delta \in (0,1/2)$, $\sigma > 0$ and $\tau \ge 1$, the following holds with probability $1 - \delta$
\begin{align*}
\sup_{h \in \cH}   \frac{2\sigma \tau}{n}\sum_{i=1}^n \epsvar_ih(w_i)  - \frac{1}{2n}\sum_{i=1}^n h(w_i)^2 \lesssim \frac{(\tau^2 \sigma^2 + B^2) \log(1/\delta)}{n} + \tau^2\delnst(\cH,\sigma)^2
\end{align*}
In particular, by integrating, 
\begin{align*}
\Exp_{\epsvar_{1:n}}\left[\sup_{h \in \cH}   \frac{2\sigma \tau}{n}\sum_{i=1}^n \epsvar_ih(w_i)  - \frac{1}{2n}\sum_{i=1}^n h(w_i)^2\right] \lesssim \frac{\tau^2\sigma^2 + B^2}{n} + \tau^2 \delnst(\cH,\sigma)^2
\end{align*}
\end{lemma}
\begin{proof}
Recall the set $\cH[r,w_{1:n}] = \{h \in \cH: \|h\|_{2,n}^2 \le r^2\}$ and let
\begin{align*}
\bZ(r) := \frac{1}{n}\sup_{h \in \cH[r;w_{1:n}]}\sum_{i=1}^n \xivar_ih(w_i) 
\end{align*}
By \Cref{lem:rademacher_concentration}, for some universal constant $c' > 0$,
\begin{align*}
\Pr\left[\tilde\bZ(r_0) > c'\left(\Exp[\tilde\bZ(r_0)] + r_0\sqrt{ \log(1/\delta)/n} + \frac{B\log(1/\delta)}{n} \right)\right] \le \delta.
\end{align*}
The bound now follows from \Cref{lem:offset_master}.
\end{proof}

\subsection{Random-Design Complexities}

The following is an analogue of \Cref{lem:offset_master} for random design. It's proof is nearly identical, with the key difference between that localization occurs based on the empirical $\cL_2$-norm $\Exp[h(\bw)]^2$ and not $\|h(w_{1:n})\|_{2,n}^2$. \footnote{This remark is under the identification $\bbV := \cH[w_{1:n}]$. We further not that \Cref{lem:random_design_master_lemma} implies \Cref{lem:offset_master} by choosing  the measure $P$ to be a dirac-delta. However, to avoid confusion of the subtle differences in localization, we state these two lemmas separately. }
\begin{lemma}[Random-Design Master Lemma]\label{lem:random_design_master_lemma} Let $P$ be a measure over random variables $(\bu,\bw)$, and let $\Phi:\cU \to \R$ and $\cH: \cW \to \R$ be function classes, and recall $\cH(r) := \{h \in \cH: \Exp_{\bw \sim P}h(\bw)^2 \le r^2\}$.
For $(\bu_1,\bw_1),\dots,(\bu_n,\bw_n) \iidsim P$, define the processes
\begin{align*}
\bZ(r) &:= \sup_{h \in \cH(r)}\sup_{\phi \in \Phi}\frac{1}{n}\sum_{i=1}^n h(\bw_i)\phi(\bu_i)\\
\bY(\tau) &:= \sup_{h \in \cH}\sup_{\phi \in \Phi}\frac{1}{n}\sum_{i=1}^n h(\bw_i)\phi(\bu_i) -  \Exp[\bh(\bw)^2]\\
\tilde{\bY}(r;\tau)&:= 2\tau\bZ(r) - \frac{r^2}{2}.
\end{align*}
Then, the conclusions of the fixed-design master lemma \Cref{lem:offset_master} hold verbatim with the above definitions.
\end{lemma}

Next, we establish two lemmas which give control on the complexities of relevant random-design (i.e. $\bw_{1:n} \sim P$)quantities involving quadratic terms such as $\Exp[h(\bw)^2]$. 
\begin{lemma}[Quadratic Loss Symmetrization]\label{lem:sasha_lem} Let $\cH: \cW \to [-B,B]$ be a  function class containing zero, let $P$ be a distribution over $\cW$, and let $\eta > 0$ be arbitrary. Consider the (very similar) terms
\begin{align*}
T_1(\eta) &:= \Exp_{\bw_{1:n} \iidsim P}\left[\sup_{h \in \cH}\left\{\|h(\bw_{1:n})\|_{2,n}^2 - (1+\eta)\Exp_{\bw' \sim P}[h(\bw')^2]\right\} \right]\\
T_2(\eta) &:= \Exp_{\bw_{1:n} \iidsim P}\left[\sup_{h \in \cH}\left\{\Exp_{\bw' \sim P}[h(\bw')^2]- (1+\eta)\|h(\bw_{1:n})\|_{2,n}^2 \right\} \right]\\
T_3(\eta) &:= \Exp_{\bw_{1:n},\bw_{1:n}' \iidsim P}\left[\sup_{h \in \cH}\left\{\|h(\bw_{1:n})\|_{2,n}^2- (1+\eta)\|h(\bw_{1:n}')\|_{2,n}^2 \right\} \right],
\end{align*}
as well as the term 
\begin{align*}
T_4(\eta) := \Exp_{\bw_{1:n}}\left[\sup_{h \in \cH}\Exp_{\epsvar_{1:n}}\left[\frac{B(1+\eta)}{n}\sum_{i=1}^n \epsvar_i h(\bw_i)\right] - \frac{1}{2}\Exp[\|h(\bw_{1:n})\|_{2,n}^2]\right]
\end{align*}
Then,
\begin{align*}
\max\{T_1(\eta), T_2(\eta),T_3(\eta),T_4(\eta)\} \lesssim  \frac{\eta(1+\eta^{-1})^{-2}}{n}\left(\frac{B^2}{n} + \delnst(\cH,B)^2\right).
\end{align*}
\end{lemma}
\begin{proof} By \citet[Lemma 14]{liang2015learning} (modifying the constant of $4B$ to $2B$ to account for the fact that we consider the uncentered $h \in \cH$, and not centered $h - \hst \in \cH$, and reparameterizing $\eta \gets \eta/2$), it holds that
\begin{align*}
T_2(\eta) \le \frac{\eta}{2n}\Exp_{\bw_{1:n} \sim P}\left[\sum_{i=1}^n  \frac{2B(2+\eta)}{\eta}\epsvar_i h(\bw_i) - h(\bw_i)^2 \right]
\end{align*}
The same argument can be modified to show that $T_1(\eta)$ satisfies the same upper bound, as $T_1(\eta)$ satisfies the same intermediate inequality obtained via Jensen's inequality (the third line of the proof in \citet[Lemma 14]{liang2015learning}), and the same argument extends to $T_3(\eta)$ because this expresssion is precisely the consequence of applying Jensen's inequality. Thus, 
\begin{align*}
\max\left\{T_1(\eta),T_2(\eta),T_3(\eta)\right\} &\le \frac{\eta}{2n}\Exp_{\bw_{1:n} \sim P}\left[\sum_{i=1}^n  \frac{2B(1+2\eta)}{\eta}\epsvar_i h(\bw_i) - h(\bw_i)^2 \right]\\
&\le \frac{2\eta}{n}\Exp_{\bw_{1:n} \sim P}\left[\sum_{i=1}^n 4B(2+\eta^{-1})\epsvar_i h(\bw_i) - \frac{1}{2} h(\bw_i)^2 \right]\\
&\overset{(i)}{\lesssim} \frac{\eta(2+\eta^{-1})^2}{n}\left(\frac{B^2}{n} + \delnst(\cH,B)^2\right),\\
&\lesssim \frac{\eta(1+\eta^{-1})^2}{n}\left(\frac{B^2}{n} + \delnst(\cH,B)^2\right),
\end{align*}
where the inequality $(i)$ is by \Cref{lem:Rad_offset} with $\sigma = B$, and $\tau = 2(1+\eta^{-1})$.  The bound on $T_4(\eta)$ follows from a similar application of \Cref{lem:Rad_offset}.
\end{proof}

\begin{lemma}[Quadratic Lower Bound]\label{lem:quad_lb} Let $\cH:\cW \to [-B,B]$ be a function class containing zero, and let $P$ be a measure over $\cW$. Then, there is a universal constant $c> 0$ such that for any $\delta > 0$, it holds that
\begin{align*}
P\left[\sup_{h \in \cH} \|h\|_{\cL_2(P)}^2 - 2\|h(\bw_{1:n})\|_{2,n}^2 \le c\rquad(\cH,\delta)^2\right] \le 1-\delta,
\end{align*}
where
\begin{align*}
\rquad(\cH,\delta)^2 := \left(\delnst(\cH,B)^2  + \frac{ B^2 \log(1/\delta)}{n} \right).
\end{align*}
\end{lemma}
\begin{proof} In view of \Cref{lem:rad_conv_hull}, the fact that $\cH$ contains zero means we may assume without loss of generality that $\cH$ is star-shaped (indeed, apply the lemma to $\starhull(\cH)$, and note that $\delnst(\cH,B) = \delnst(\star(\cH),B)$).
Introduce the class of function $\tilde \cH_r := \{\Exp[h^2] - h^2: h \in \cH(r)\}$ (here, we use subscript $r$ to distinguish from the standard localization notation). Then, $\tilde \cH_r: \cW \to [-B^2,B^2]$, and $\Exp[\tilde h(\bw)^2] \le B^2 r^2$ for $\tilde h \in \tilde \cH_r$. Note that
\begin{align*}
\sup_{h \in \tilde \cH_r} \frac{1}{n}\sum_{i=1}^n\tilde{h}(\bw_i) = \sup_{h \in \cH(r)}\|h\|_{\cL_2(P)}^2 - \|h(\bw_{1:n})\|_{2,n}^2
\end{align*}
Hence, by \Cref{lem:saras_lemma} and AM-GM, the following holds with probability $1 - \delta$ and for a universal constant $c > 0$ and any $\tau \ge 0$:
\begin{align*}
&\sup_{h \in \cH(r)}\|h\|_{\cL_2(P)}^2 - \|h(\bw_{1:n})\|_{2,n}^2 \\
&\le c\left(\Exp\left[\sup_{h \in \tilde \cH_r}\|h\|_{\cL_2(P)}^2 - \|h(\bw_{1:n})\|_{2,n}^2\right] + \tau r^2 + \frac{(\tau^{-1}+1)B^2 \log(1/\delta)}{n} \right)\\
&\le c\left(\Exp\left[\sup_{h \in \cH(r)}(1-\tau)\|h\|_{\cL_2(P)}^2 - \|h(\bw_{1:n})\|_{2,n}^2\right] + 2\tau r^2 + \frac{(\tau^{-1}+1)B^2 \log(1/\delta)}{n} \right),
\end{align*}
where the second inequality uses $\|h\|_{\cL_2(P)}^2 \le r^2$. Let $\tau \le 1/2$ and let $\eta$ be such that $(1-\tau) = \frac{1}{1+\eta}$. By \Cref{lem:sasha_lem},
\begin{align*}
\Exp\left[\sup_{h \in \cH(r)}(1-\tau)\|h\|_{\cL_2(P)}^2 - \|h(\bw_{1:n})\|_{2,n}^2\right] &= \frac{1}{1+\eta}\Exp\left[\sup_{h \in\cH(r)}\|h\|_{\cL_2(P)}^2 - (1+\eta)\|h(\bw_{1:n})\|_{2,n}^2\right] \\
&\lesssim \frac{\eta(1+\eta^{-1})^{-2}}{(1+\eta)}\left(\frac{B^2}{n} + \delnst(\cH,B)\right)\\
&= (1-\eta^{-1})\left(\frac{B^2}{n} + \delnst(\cH,B)^2\right)\\
&\lesssim \frac{1}{\tau}\left(\frac{B^2}{n} + \delnst(\cH,B)^2\right),
\end{align*}
where in the last line, we use that $\eta = 1 - (1-\tau)^{-1} \gtrsim \tau$ for $\tau \le 1/2$. In sum, there is a universal constant $c'$ such that, for all $\tau \le 1/2$, the following holds with probability $1- \delta$:
\begin{align*}
\sup_{h \in \cH(r)}\|h\|_{\cL_2(P)}^2 - \|h(\bw_{1:n})\|_{2,n}^2 &\le c'\left(\frac{1}{\tau}\left(\frac{B^2}{n} + \delnst(\cH,B)\right)^2 + \tau r^2 + \frac{(\tau^{-1}+1)B^2 \log(1/\delta)}{n} \right). 
\end{align*}
By making $\tau$ a sufficiently small universal constant, we can ensure that there is a universal constants $c'',c'''$ such that, whenever 
\begin{align*}
r^2 = c''\left( \delnst(\cH,B)^2 +  \frac{B^2 \log(1/\delta)}{n} \right)\lesssim \rquad(\cH,\delta).
\end{align*}
we have that with probability $1 - \delta$,
\begin{align}
\sup_{h \in \cH(r)}\|h\|_{\cL_2(P)}^2 - \|h(\bw_{1:n})\|_{2,n}^2 \le \frac{r^2}{2}. \label{eq:on_H_r}
\end{align}
We claim that in fact, with probability $1- \delta$, it holds that the above holds for all $h \in \cH$, that is
\begin{align*}
\sup_{h \in \cH}\|h\|_{\cL_2(P)}^2 - 2\|h(\bw_{1:n})\|_{2,n}^2 \le \frac{r^2}{2}
\end{align*}
Indeed, it suffices to check \Cref{eq:on_H_r} implies the inequality for $h \notin \cH(r)$. Since $\cH$ is star-shaped, there exists some $\alpha$ such that $\alpha h \in \cH(r)$ and in fact $\|\alpha h\|_{\cL_2(P)}^2 = r^2$. Then, on \Cref{eq:on_H_r}
\begin{align*}
\|\alpha h\|_{\cL_2(P)}^2 - \|\alpha h(\bw_{1:n})\|_{2,n}^2 \le \frac{r^2}{2} = \frac{\|\alpha h\|_{\cL_2(P)}^2}{2}
\end{align*}
so by rearranging
\begin{align*}
\| h\|_{\cL_2(P)}^2 \le  2\|h(\bw_{1:n})\|_{2,n}^2.
\end{align*}

\end{proof}

\section{Rates for finite function classes.}
\label{sec:finite_function_classes}

In this section, we establish sharper bounds 
for finite function classes $\cH_1$ and $\cH_2$. Because errors for finite function classes already attain the $\BigOh{1/n}$ parametric rate, we require an additional assumption to achieve improvement. Specifically, we need a hypercontractivity condition which states that higher moments of $\Exp[h(\bw)^q]$ for $h \in \cH$ are controller by lower order moments $\Exp[h(\bw^s)]$ for $s < q$. This is the first notion of hypercontractivity, defined below.
\begin{definition}[\Hycon] We say a class $\cH \subset \{\cW \to \R\}$ satisfies $(\kappa, \Pr, s, q)$-\hycon{} if, for all $h \in \cH$, $\Exp[|h(\bw)|^q]^{1/q} \le \kappa \Exp[h(\bw)^s]^{1/s}$.
\end{definition}
We achieve even faster rates under a stronger variant of hypercontractivity, defined below.
\begin{definition}[subGaussian \Hycon] We say a class $\cH \subset \{\cW \to \R\}$ satisfies $(\kappa,\Pr)$-subGaussian \hycon{} if, for all $h \in \cH$, $h(\bw) - \Exp[h(\bw)]$ is $\kappa^2 \Exp[h(\bw)^2]$ subGaussian. \footnote{This is equivalent to $\BigOh{\kappa}$-\hycon{} in the $\cL_2 \to \psi_2$ norms, where $\psi_2$ is the subGaussian (Orlicz) norm (see e.g. \citet[Exercise 2.18]{boucheron2013concentration}.) } 
\end{definition}
Under the various hypercontractivity assumptions, we attain the following bound, which is the formal statement of  \Cref{thm:finite_class_informal}.
\begin{theorem}\label{thm:finite_class_hyper} Let $\cF: \cX \to [-1,1]$ and $\cG:\cY \to [-1,1]$ be finite function classes, and suppose $1 \le d_1 \le d_2$ satisfy $ \log |\cF| \le d_1$ and $\log |\cG| \le d_2$. Define the class\footnote{note that $\Fcent := \starhull(\tilde \cF)$} $\tilde \cF := \{f - \upbeta_f - \fst: f \in \cF\}$, and casing on the hypercontractivity assumptions with parameter $\kappa$, define
\begin{align*}
\phi_n(d_1,d_2) := 
 \begin{cases}\left(\frac{d_2}{n}\right)^{\frac{2}{q_2}} +  \left(\frac{d_1}{d_2}\right)^{\frac{1}{q_1}} & \text{\normalfont general } \frac{1}{q_1}+\frac{1}{q_2} = 2, ~ \tilde{\cF} \text{ satisfies $(\kappa, \Ptrain, 2,q_1)$ \hycon}\\
\left(\frac{d_2}{n}\right)^{\frac{1}{2}} +  \left(\frac{d_1}{d_2}\right)^{\frac{1}{4}} & ~\tilde{\cF} \text{ satisfies $(\kappa, \Ptrain, 2,4)$ \hycon}\\
\frac{d_2}{n} \cdot  (d_1 + \log n) & \tilde{\cF} \text{ satisfies $(\kappa,\Ptrain)$-subGaussian \hycon}
\end{cases}
\end{align*}
Then, as long as $\sigma^2 \lesssim 1$, for any $\delta \in (e^{-10d_2},e^{-1})$, the following hold simultaneously with probability at least $1- \delta$:
\begin{align*}
 \Rtrain[\ghatn;\fhatn] &\le \Rtrain(\fhatn,\ghatn)  \lesssim \frac{d_2}{n}\\
\Rzero[\fhatn] &\lesssim  \kappa^2\phi_n(d_1,d_2) \cdot \frac{ d_2}{n} + \frac{d_1}{n}  +  \frac{\log(1/\delta)}{n}\\
\Rtest(\fhatn,\ghatn) &\lesssim \nu_{x,y}\cdot \frac{d_1 + \log(1/\delta)}{n} + (\nu_y + \nu_{x,y} \cdot \kappa^2\phi_n(d_1,d_2))\cdot\frac{d_2}{n}
\end{align*}
\end{theorem}
Notice that, as promised by \Cref{thm:finite_class_informal} $\phi_n(d_1,d_2)$  tends to $0$ as $n \to \infty$ and as the ratio of the class complexities $d_1/d_2$ tends to $0$, such that with high probability. Moreover, under subGaussian hypercontractivity, $\lim_{n\to \infty} \phi_n(d_1,d_2)$ for any $d_1,d_2$ fixed.

\begin{remark}[Extension to Parametric Classes]\label{rem:ext_to_par} Up to logarithmic factors in $n$, the above bound can be extended easily extended to infinite-cardinality ``parametric'' function classes (that is, function classes whose metric entropies scale as logarithmic in the scale $\epsilon$). The guarantee of \Cref{thm:finite_class_hyper} also holds under the more general assumption that $\cF$ and $\cG$ are contained in the respective convex hulls of function classes $\tilde \cF$ and $\tilde \cG$, where $\log|\tilde \cF| \le d_1$ and $\log |\tilde \cG| \le d_2$. This includes, for example, many natural linear classes.
\end{remark}

\subsection{Proof of \Cref{thm:finite_class_hyper}}
\newcommand{\delcrossngentil}{\tilde{\updelta}_{n,\mathtt{crs},\Pr}}
We start with localization for finite classes.
 \begin{lemma}[Localization for Finite Classes]\label{lem:finite_localization} Let $\cH$ be a finite function class uniformly bounded by $1$, and let $d = \log |\cH|$. Then, for any probability measure $\Pr$ over $\cW$,
    \begin{align*}
    \delnst^2(\cH,B) \lesssim \frac{B^2 d}{n}, \quad  \rhonst(\cH,\sigma) \lesssim \frac{d\sigma^2}{n}
    \end{align*}
    \end{lemma}
    \begin{proof} [Proof of \Cref{lem:finite_localization}]
    From \Cref{lem:rad_uniform_convergence} and finiteness of $\cH$, we have for any $w_{1:n} \in \cW^n$ that
    \begin{align*}
    \Raden(\cH[r;w_{1:n}]) \le \rad_2(\cH[r;w_{1:n}]\sqrt{2d/n} \le r\sqrt{2d/n}.
    \end{align*}
    It then follows that $\delnst^2(\cH,B) = \sup\{r^2: \Raden(\cH[r;w_{1:n}]) \le \frac{r^2}{2B} \} \lesssim \frac{B^2d}{n}$. The bound on $\rhonst$ similarly yields $\delnst^2(\cH) = \sup\{r^2: \Raden(\cH[r;w_{1:n}]) \le \frac{r^2}{2\sigma } \} \lesssim \frac{\sigma^2 d}{n}$. 
    \end{proof}

    We continue with a generic bound on the following cross-critical radius.
The next proposition is proved in \Cref{sec:prop:rad_bound_finite_hyper} below.
    \begin{proposition}\label{prop:rad_bound_finite_hyper} For $i \in \{1,2\}$, let $ \cH_i\subset \{\cW \to [-1,1]\}$  be finite function classes with $d_i = \log|\cH_i|$.   Assume for simplicity that $d_1 \le d_2$, and let $\gamma^2 \ge d_2/n$. Finally, let $\Pr$ be a distribution of $\cW$. Define the shorthand
    \begin{align*} 
    \updelta_n(\gamma) := \inf\left\{r: \Exp_{\bw_{1:n}}[\Raden((\cH_1(r) \odot \cH_2(\gamma))[\bw_{1:n}])] \le \frac{r^2}{2} \right\}.
    \end{align*}
    Then, it holds that
    \begin{itemize}
    \item Let $1/q_1 + 1/q_2 = 1/2$ be square H\"older conjugates.  If $\cH_1$ satisfies $(\kappa, \Pr, 2,q_1)$ \hycon,
    \begin{align*}
    &\updelta_n(\gamma)^2 \lesssim \kappa^2 \gamma^{4/q_2} \frac{d_2}{n} + \gamma^{2/q_2}\sqrt{\frac{d_2}{n}}\left(\frac{d_1}{n}\right)^{1/q_1}.
    \end{align*}
    In particular, if $\gamma^2 \simeq d_2/n$, 
    \begin{align*}
    \updelta_n(\gamma)^2 &\lesssim \kappa^2 \begin{cases}\left(\frac{d_2}{n}\right)^{1+\frac{2}{q_2}} + \frac{d_2}{n} \cdot \left(\frac{d_1}{d_2}\right)^{\frac{1}{q_1}}& \text{\normalfont general } \frac{1}{q_1}+\frac{1}{q_2} = 2\\
    \left(\frac{d_2}{n}\right)^{\frac{3}{2}} + \frac{d_2}{n} \cdot \left(\frac{d_1}{d_2}\right)^{\frac{1}{4}} & q_1 = q_2 = 4.\end{cases}\\
    \end{align*}
    \item  If $\cH_1$ satsfies $(\kappa,\Pr)$-subGaussian \hycon, 
    \begin{align*}
    \updelta_n(\gamma)^2 \lesssim \frac{\kappa^2\gamma^2 d_2 }{n}  \cdot (d_1 + \log n).
    \end{align*}
    In particular, if $\gamma^2 \simeq d_2/n$, the above scales as $\kappa^2 (d_2/n)^2 \cdot  (d_1 + \log n)$.
    \end{itemize}
    \end{proposition}
Next, we recall standard localization bounds for finite function classes.
       \begin{proof}[Proof of \Cref{thm:finite_class_hyper}]  As $d_2 \ge d_1$, $\log|\cF + \cG| \le \log(|\cF| + |\cG|) \lesssim d_2$. Taking $B = 1$, $\sigma^2 \lesssim 1$, and $\delta = e^{-d_2}$, \Cref{lem:finite_localization} and \Cref{prop:generic_sum_regret} allow us to bound
    \begin{align*}
     \Rtrain[\ghatn;\fhatn] &\le \Rtrain(\fhatn,\ghatn)  \lesssim \frac{d_2}{n} \text{ w.p. } 1- e^{-d_2}
    \end{align*}
    By the same token, applying \Cref{prop:main_reg} and \Cref{lem:finite_localization}, and making similar simplifications ($B=1,\sigma^2 \lesssim 1$), the following holds with probabilty $1-\delta$
    \begin{align*}
    \Rzero[\fhatn] &\lesssim \delcrossnbar(\gamma)^2 + \frac{d_1 + \log(1/\delta)}{n}
    \end{align*}
    Bounding $\delcrossnbar(\gamma)^2$ by \Cref{prop:rad_bound_finite_hyper}, on the same event we have
    \begin{align*}
    \Rzero[\fhatn] &\lesssim \frac{\kappa^2\phi_n(d_1,d_2) d_2}{n} + \frac{d_1 + \log(1/\delta)}{n},
    \end{align*}
    where we recall $\phi_n(d_1,d_2)$ defined in the theorem statement.
    When both events hold, \Cref{lem:excess_decomp} entails
    \begin{align*} 
    \Rtest(\fhatn,\ghatn) &\lesssim \nu_{x,y}\left(\frac{\kappa^2\phi_n(d_1,d_2) d_2}{n}  + \frac{d_1}{n}  +  \frac{\log(1/\delta)}{n}\right) + \nu_y\frac{d_2}{n}.
    \end{align*}
    \end{proof}
\subsection{Proof of \Cref{prop:rad_bound_finite_hyper}}\label{sec:prop:rad_bound_finite_hyper}
\newcommand{\radpq}{\rad_{P,q}}
    Define the radius of a class $\cH:\cW \to \R$ be a class, and let $P$ be a measure over $\cW$. Define
    \begin{align*}
    \radbar_{q}(\cH) := \sup_{h \in \cH}\Exp[|h(\bw)|^q]^{1/q}, \quad \rad_{q}(\cH[w_{1:n}]) := \sup_{h \in \cH}\|h[w_{1:n}]\|_{q,n} 
    \end{align*}
    \paragraph{Part 1. Bounds on the empircal norms.} The next lemma bounds the magnitude of the empirical $q$-norm radius. 
    \begin{lemma}\label{lem:Bernstein_q_norm} Let $\cH \subset \{\cW \to [-1,1]\}$ be a finite class, and take $\delta \in (0,1)$. With probability at least $1 - \delta$, 
    \begin{align*}
    \rad_q(\cH[\bw_{1:n}])\le 2\radbar_q(\cH) + \left(\frac{\log |\cH|/\delta}{n}\right)^{\frac{1}{q}} \le 1 - \delta.
    \end{align*}
    In particular, if $\radbar_2(\cH) \le r$ and $\cH$ satisfes $(\kappa, \Pr,2,q)$ hypercontractivity, then with probability $1- \delta$
    \begin{align*}
     \rad_q(\cH[\bw_{1:n}])\le 2\kappa r + \left(\frac{\log 1/\delta}{n}\right)^{\frac{1}{q}}.
    \end{align*}
    \end{lemma}
    \begin{proof} We observe that $\rad_q(\cH[\bw_{1:n}])^q\le  \radbar_q(\cH)^q + \Exp\sup_{h \in \cH}\frac{1}{n}\sum_{i=1}^n |h(\bw_i)|^q - \Exp[|h(\bw_i)^q|]$. As $\sup_h|h| \le 1$, $ \sup_{h \in \cH} \Exp[|h(\bw_i)|^{2q}] \le \sup_{h \in \cH} \Exp[|h(\bw_i)|^{q}] = \radbar_q(\cH)^q$. By Bernstein's inequality and a union bound, with  probability at least $1- \delta$,
    \begin{align*}
    \Exp\sup_{h \in \cH}\frac{1}{n}\sum_{i=1}^n |h(\bw_i)|^q - \Exp[|h(\bw_i)|^q] &\le \sqrt{\frac{2\radbar_q(\cH)^q \log(|\cH|/\delta)}{n}} + \frac{\log|\cH|/\delta}{3n}\\
    &\le \radbar_q(\cH)^q  +  \frac{\log|\cH|/\delta}{n}. \tag{AM-GM, and $\frac{1}{3}+\frac{1}{2} \le 1$}
    \end{align*}
    Hence, via the previous two displays, with probability $1-\delta$ it holds that
    \begin{align*}
    \rad_q(\cH[\bw_{1:n}])^q \le 2\radbar_q(\cH)^q + \frac{\log|\cH|/\delta}{n}.
    \end{align*}
    Taking the $q$-th root and using  of $(x+y)^{1/q} \le y^{1/q} + x^{1/q}$ for $q \ge 1$  and $x,y \ge 0$ concludes the proof. 
    \end{proof}
     When $\cH$ satisfies $(\kappa,\Pr)$-subGaussian hypercontractivity, we can improve this bound. 
    \begin{lemma}\label{lem:Gauss_hyper_con} Suppose $\cH\subset \{\cW \to [-1,1]\}$ be a finite class which satisfies $(\kappa,\Pr)$-subGaussian hypercontractivity. Then,
    \begin{align*}
    \Pr[\rad_{\infty}(\cH[\bw_{1:n}]) \le \kappa\sqrt{\log |\cH| + \log(n/\delta)} \cdot \radbar_2(\cH)] \ge 1-\delta.
    \end{align*}
    \end{lemma}
    \begin{proof} By Gaussian concentration and $(\kappa,\Pr)$-suBgaussian hypercontractivity, for any $h \in \cH$, $i \in [n]$ and $\delta > 0$, we have
    \begin{align*}
    \Pr\left[ |h(\bw)_i| \ge \kappa \radbar_2(\cH) \log(1/\delta) \right] \le  \Pr\left[ |h(\bw)_i| \ge \kappa \Exp[h(\bw_i)^2]^{1/2} \log(1/\delta) \right] \le \delta
    \end{align*}
    Union bounding over $h \in \cH$ and $i \in [n]$ concluds the proof.
    \end{proof}
	
	\paragraph{Part 2. Controlling the Rademacher Complexities} Next, we turn to bounding the Rademacher complexity in terms of empirical radii.
    \begin{lemma}\label{lem:first_finite_rad_bound}  Let $1/q_1 + 1/q_2 \le 1/2$ be squared H\"older conjugates. Then, 
    \begin{align*}
    \Radenp(\cH_1(r)\odot \cH_2(\gamma)) &\le \sqrt{2 (d_1 +d_2)}\Exp \left[\rad_{q_1}( \cH_{1,r}[\bw_{1:n}])\cdot\rad_{q_2}( \cH_{2,\gamma}[\bw_{1:n}])\right]
    \end{align*}
    \end{lemma}
    \begin{proof} This is a direct consequence of  \Cref{lem:finite_expectation_upper_bound}, and the fact that $\log|\cH_1(r) \odot \cH_2(\gamma)| \le \log |\cH_1| + \log|\cH_2| = d_1 + d_2$. 
    \end{proof}

    \paragraph{Part 3. Conclusion the proof.} We can now conclude.
    \begin{proof}[\Cref{prop:rad_bound_finite_hyper}]  Let us start with the case that $\cH_1$ satisfies $(\kappa, \Pr, 2, q_1)$ hypercontractivity.  Uing boundedness of $|h| \le 1$ and $q_2 \le 1$ we get 
    \begin{align*}
    \radbar_{q_2}(\cH_{2,\gamma})^{q_2} = \sup_{h \in \cH_{2,\gamma}}\Exp[|h(\bw)|^{q_2}] \le \sup_{h \in \cH_{2,\gamma}}\Exp[|h(\bw)|^{2}] \le \gamma^2,
    \end{align*} 
    so $\radbar_{q_2}(\cH_{2,\gamma}) \le \gamma^{2/q_2}$. Hence, as $d_2/n \le \gamma^2$ by assumption,
    \begin{align}
    \radbar_{q_2}(\cH_{2,\gamma}) + (\frac{d_2}{n})^{\frac{1}{q_2}} &\le \gamma^{2/q_2} + (d_2/n)^{1/q_2} \le2 \gamma^{2/q_2} \label{eq:rad_bound_h2}.
    \end{align}
    Next, by \Cref{lem:Bernstein_q_norm}, hypercontractivity of $ \cH_1$, and the above bound implies that, for all $\delta > 0$,  both 
    \begin{align*}\rad_{q_1}(\cH_{1,r}[\bw_{1:n}])\le 2\kappa_1 r + \left(\frac{d_1 + \log 1/\delta}{n}\right)^{\frac{1}{q_1}}
    \end{align*} and 
    \begin{align*}
    \rad_{q_2}(\cH_{1,r}[\bw_{1:n}]) &\le 2\gamma^{2/q_2} + \left(\frac{d_2 + \log 1/\delta}{n}\right)^{\frac{1}{q_2}},
    \end{align*}
    hold with probability at least $1 - \delta$.  Taking the product, integrating the tail over $\delta$, and invoking \Cref{eq:rad_bound_h2} implies
    \begin{align*}
    \Exp \left[\rad_{q_1}( \cH_{1,r}[\bw_{1:n}])\cdot\rad_{q_2}( \cH_{2,\gamma}[\bw_{1:n}])\right] &\lesssim \gamma^{2/q_2} \left(r \cdot \kappa   + \left(\frac{d_1}{n}\right)^{1/q_1}\right).
    \end{align*}
    Note the resulting constant does not depend on $q_1$ or $q_2$, as $0 \le 1/q_1,1/q_2 \le 1$ are both bounded. 
    Consequently, \Cref{lem:first_finite_rad_bound} and $d_1 \le d_2$ entails
    \begin{align*}
    \Radenp(\bar\cH_1(r)\odot \bar\cH_2(\gamma)) \lesssim \sqrt{\frac{d_2}{n}}\gamma^{2/q_2} \left(r \cdot \kappa   + \left(\frac{d_1}{n}\right)^{1/q_1}\right),
    \end{align*}
    Thus, 
    \begin{align*}
    &\inf \left\{r^2:\Radenp(\bar\cH_1(r)\odot \bar\cH_2(\gamma)) \le \frac{r^2}{2}\right\} \lesssim  \kappa^2 \gamma^{4/q_2} \frac{d_2}{n} + \gamma^{2/q_2}\sqrt{\frac{d_2}{n}}\left(\frac{d_1}{n}\right)^{1/q_1}.
    \end{align*}
    Next, let's consider the subGaussian hypercontractive case. Here, we replace \Cref{lem:Bernstein_q_norm} with \Cref{lem:Gauss_hyper_con}. A similar computation yields 
    \begin{align*}
    \Exp \left[\rad_{\infty}( \cH_{1,r}[\bw_{1:n}])\cdot\rad_{2}( \cH_{2,\gamma}[\bw_{1:n}])\right] \lesssim \kappa r \gamma\sqrt{d_1 + \log n}.
    \end{align*}
    Hence,  \Cref{lem:first_finite_rad_bound} (with $d_1 \le d_2$) gives
    \begin{align*}
    \Radenp(\bar\cH_1(r)\odot \bar\cH_2(\gamma)) \lesssim \sqrt{\frac{d_2}{n}} \kappa r \gamma\sqrt{d_1 + \log n}.
    \end{align*}
    We may then conclude
    \begin{align*}
    &\inf \left\{r^2:\Radenp(\bar\cH_1(r)\odot \bar\cH_2(\gamma)) \le \frac{r^2}{2}\right\} \lesssim   \frac{\kappa^2\gamma^2 d_2 }{n} \cdot \kappa (d_1 + \log n).
    \end{align*}

    \end{proof}

\newcommand{\Xiglobn}{\Xi_{n,\mathtt{glob}}}
\newcommand{\Xilocn}{\Xi_{n,\mathtt{loc}}}

\newcommand{\deldudbar}{\bar{\updelta}_{n,\Dudfun}}
\newcommand{\Dudfunbar}[1][q]{\bar{\Dudfun}}
\section{Formal Guarantees for Nonparametric Classes (formal statement of \Cref{thm:nonpar})}\label{sec:instantiating_the_rates}
In this section,  give a formal statements of our main results. After giving further preliminaries in, we state give an formal statement of \Cref{thm:nonpar}, \Cref{thm:nonpar_formal}, and a formal version of \cref{thm:double_ml}, \Cref{thm:double_ml_formal}. This is done in  \Cref{sec:rates_instant}, which explicitly defines the $\rate$ functions. Below that, we derive these two results from a yet-more-general bound, \Cref{thm:deldudbar}, which replaces the dependence of Dudley integrals on the centered classes in \Cref{sec:analysis_overview} with terms dependending only on the complexit of $\cF$, $\cG$, and, optionally, on a class $\upbeta_{\cF}$ of biases. The remaind of the section is dedicated to proofs. 

\paragraph{Further Preliminaries} Recall the definition of the normalized $q$-norms: For $v  = (v_1,\dots,v_n) \in \R^n$ and $q \in [1,\infty)$, we have
\begin{align*}\|v\|_{q,n} = \left(\frac{1}{n}\sum_{i=1}^n |v_i|^q\right)^{1/q}, \quad \|v\|_{\infty,n} = \|v\|_{\infty} = \max_{ i \in [n]}|v_i|.
\end{align*} 
We now define the radii and metric entropies in these norms, with a definition that expands upon \Cref{defn:metric_entropy}.
\begin{definition}[$q$-norms, radii, and metric entropies]\label{defn:entropies}
 Given a subset $\bbV \subset \R^n$ and $q \in [1,\infty]$, define the the \emph{radius} $\rad_q(\bbV) = \max_{v \in \bbV}\|v\|_{q,n}$, define the \emph{covering number} $\covnum(\bbV,\|\cdot\|_{q,n},\veps)$ as the cardinality of minimal-cardinality $\veps$-cover of $\bbV$ in the norm $\|\cdot\|_{q,n}$, and define the \emph{metric entropy} $\mnum_q(\bbV,\veps)= \log \covnum(\bbV,\|\cdot\|_{q,n},\veps)$ as the logarithmic of the covering number. For a function class $\cH:\cW \to \R$, we define its $q$-norm \emph{metric entropy } as 
\begin{align*}
\mnum_q(\cH, \veps):= \sup_{n\in \N}\sup_{w_{1:n} \in \cW^{n}} \mnum_q(\cH[w_{1:n}],\veps).
\end{align*}
\end{definition}

\subsection{Instantiating the Rates}\label{sec:rates_instant}

Throughout, we make the following mild compactness assumption, which holds whenever \Cref{thm:nonpar,thm:nonpar_formal} is non-vacuous.
\begin{assumption}\label{asm:F_compact} For all $\veps > 0$, $\mnum_{2}(\cF,\veps) < \infty$.
\end{assumption}
We also introduce a strictly optional second assumption, but codifies a way in which $\cF$ is ``simpler'' than $\cG$, and enables further simplifications when it holds.
\begin{assumption}\label{asm:F_simpler} For all $\veps > 0$, $\mnum_{\infty}(\cF,\veps) \le\mnum_{\infty}(\cG,\veps)$.
\end{assumption}

Lastly, we define the class 
\begin{align}
\upbeta_{\cF} := \{\upbeta_f : f \in \cF\}.
\end{align}

Next, we formally define families of function classes we all \emph{entropy families}, which are characterized by upper bounds of their metric entropies. These entropy families formally capture the entropy rates depicted in \Cref{thm:nonpar}.
\begin{restatable}[Entropy Families]{definition}{entclasses}\label{defn:entclasses} Let $\cH:\cW \to \R$, and let $q \in [1,\infty]$, and let $\btau = (\tau_0,\tau_1,\tau_2) \in \R^3$ denote a vector of parameters. We say that $\cH \in \entclass_q(p,\btau;R)$ if $\rad_q(\cH) \le R$ and either
\begin{itemize}
\item  $p = 0$, and for all $\veps > 0$, $\mnum_q(\cH) \le \tau_0 + \tau_1\log(\tau_2/\veps)$ or
\item $p > 0$, and for all $\veps > 0$, $\mnum_q(\cH,\veps) \le \tau_0 +  \tau_1\veps^{-p}$. 
\end{itemize}
Notice that the sets $\entclass_{q}(p,\btau)$ are non-increasing in $q$, and non-decreasing in the coordinates of $\btau$, and (up to constants) non-increasing in $p$. 
\end{restatable}

We now define complexities measures that upper bound the localized and unlocalized Dudley integrals for function classes in a given entropy family.
\begin{restatable}[Key Complexities]{definition}{Xicomps}\label{eq:Xicomps} Let $\btau = (\tau_0,\tau_1,\tau_2) \in \R_{\ge 0}^3$, $p \in [0,\infty)$, and $R,c > 0$. We define the \emph{global complexity term}
\begin{align*}
\Xiglobn(p,\btau, R) &:= \frac{R\tau_0}{\sqrt{n}} +  \begin{cases} R\sqrt{\frac{\tau_1\log(e + \tau_2/R)}{n}} & p = 0\\
 \frac{R^{1-p/2}}{1-p/2}\sqrt{\frac{\tau_1}{ n}} & p \in (0,2)
    \\
    \sqrt{\frac{\tau_1}{n}}\log(e + R\sqrt{n/\tau_1}) & p = 2 \\
    (\frac{\tau_1}{n})^{\frac{1}{p}}(p/2-1)^{-\frac{2}{p}} & p > 2
    \end{cases} =  \BigOhTil{1} \cdot \begin{cases} n^{-\frac{1}{2}} & p \le 2\\
    n^{-\frac{1}{p}} & p > 2
\end{cases}
\end{align*}
and the \emph{local complexity term}
\begin{align*}
    \Xilocn( p,\btau,c) &:= \frac{c^2(1+\tau_0)^2}{n} + \begin{cases} \frac{c^2\tau_1\log(e+\sqrt{n}\tau_2/c)}{n} & p = 0\\
\left(\frac{c^2}{(1-p/2)^2} \cdot \frac{\tau_1}{ n }\right)^{\frac{2}{2+p}} & p \in (0,2) \\
c\sqrt{\frac{\tau_1\log(e+ c\sqrt{n}/\tau_1)}{n}} & p =2 \\
(p/2-1)^{-\frac{2}{p}}(\frac{\tau_1}{n})^{\frac{1}{p}} & p > 2.
\end{cases} = \BigOhTil{1}\cdot \begin{cases} n^{-\frac{2}{2+p}} & p \le 2\\
    n^{-\frac{1}{p}} & p > 2
\end{cases}
\end{align*}
\end{restatable}
Lastly, we define the \emph{rate functionals}.
    \begin{definition}[Rate Functionals] We define the following rate functionals:
    \begin{align*}
    \rateqn(\cH,c) &:= \inf_{p,\btau,R} \{\Xilocn( p,\btau,c) : \cH \in \entclass_q(p,\btau; R) \}\\
    \ratestn(\cH) &:= \inf_{p,\btau,R} \{\Xiglobn(p,\btau, R) : \cH \in \entclass_\infty(p,\btau; R) \}.
    \end{align*}
    \end{definition}
    That is, $\rateqn(\cH,c)$ is the smallest possibly local complexity term subject to $\cH$ being in the appropriate entropy family, and $\ratestn(\cH)$ is the smallest possible global complexity term, always taken with metric entropy in the $\infty$-norm.
We are now ready to state the formal versions of \Cref{thm:nonpar} and \Cref{thm:double_ml} with explicit rates. 
\begin{theorem}[Formal version of \Cref{thm:nonpar}]\label{thm:nonpar_formal} 
Suppose \Cref{asm:cov_shift,asm:well_spec,asm:conditional_completeness,asm:bounded} hold, as well as \Cref{asm:F_compact}. Let $\sigB := \max\{B,\sigma\}$, let $\nu_1,\nu_2$ be as in \Cref{eq:nu1,eq:nu2}, and let $c_2$ be a sufficiently small universal constant, and suppose that
\begin{align*}
\rateqn[2](\cG,\sigB) + \rateqn[2](\cF,\sigB) + \frac{\sigB^2\log(1/\delta)}{n} \le c_2 \gamma,
\end{align*} 
Then 
\begin{itemize} 
    \item[(a)] If \Cref{asm:F_simpler} holds, i.e. the metric entropies of $\cF$ are less than those of $\cG$, then  probability at least $1 - \delta$, \begin{align*}
\Rtest(\fhatn,\ghatn) \lesssim (\nu_{1}+\nu_2) \rateqn[2](\cF,\sigB) + \nu_1 \ratestn(\cG)^2  \\
+ \nu_2 \,\rateqn[2](\cG,\sigB) + (\nu_1+\nu_2)\cdot\frac{\sigB^2 \log(1/\delta)}{n}, 
\end{align*}
\item[(b)] As a consequence of (a), 
if we upper bound $\nu_1 \le \nu_{x,y}$ and $\nu_2 \le \nu_y$, we obtain (as $\nu_y \le \nu_{x,y}$)
\begin{align*}
\Rtest(\fhatn,\ghatn) \lesssim \nu_{x,y} \left(\rateqn[2](\cF,\sigB) + \ratestn(\cG)^2\right) + \nu_y \,\rateqn[2](\cG,\sigB) + \nu_{x,y}\frac{\sigB^2 \log(1/\delta)}{n}. 
\end{align*}
\item[(c)] More generally, suppose that \Cref{asm:F_simpler} need not hold. Then, recalling the class $\upbeta_{\cF}:= \{\upbeta_f: f \in \cF\}$, we obtain 
\begin{align*}
\Rtest(\fhatn,\ghatn) &\lesssim (\nu_{1}+\nu_2) \rateqn[2](\cF,\sigB) + \nu_1 \left(\ratestn(\cG)^2 + \min\left\{\ratestn(\cF)^2,\ratestn(\upbeta_{\cF})^2\right\}\right) \\
&\quad+ \nu_2 \,\rateqn[2](\cG,\sigB) + (\nu_1+\nu_2)\cdot\frac{\sigB^2 \log(1/\delta)}{n}. 
\end{align*}
\end{itemize}

\end{theorem}

\begin{theorem}[Formal version of \Cref{thm:double_ml}]\label{thm:double_ml_formal} 
Suppose \Cref{asm:well_spec,asm:conditional_completeness,asm:bounded} hold, as well as \Cref{asm:F_compact}. Let $\sigB := \max\{B,\sigma\}$, let $\nu_1,\nu_2$ be as in \Cref{eq:nu1,eq:nu2}, and let $c_2$ be a sufficiently small universal constant, and suppose that
\begin{align*}
\rateqn[2](\cG,\sigB) + \rateqn[2](\cF,\sigB) + \frac{\sigB^2\log(1/\delta)}{n} \le c_2 \gamma,
\end{align*} 
Then 
\begin{itemize} 
    \item[(a)] If \Cref{asm:F_simpler} holds, i.e. the metric entropies of $\cF$ are less than those of $\cG$, then  probability at least $1 - \delta$, \begin{align*}
\Rtest[\fhatn] \lesssim  \rateqn[2](\cF,\sigB) +  \ratestn(\cG)^2 + \frac{\sigB^2 \log(1/\delta)}{n}. 
\end{align*}
\item[(b)] More generally, suppose that \Cref{asm:F_simpler} need not hold. Then, recalling the class $\upbeta_{\cF}:= \{\upbeta_f: f \in \cF\}$, we obtain 
\begin{align*}
\Rtest[\fhatn]\lesssim  \rateqn[2](\cF,\sigB) + \ratestn(\cG)^2 + \min\left\{\ratestn(\cF)^2,\ratestn(\upbeta_{\cF})^2\right\}+\frac{\sigB^2 \log(1/\delta)}{n}. 
\end{align*}
\end{itemize}
\end{theorem}
The proof of both theorems are derived from an intermediate bound, \Cref{thm:deldudbar}, in the following section.

\subsection{Proof of Main Theorems via Intermediate Dudley Bound}
Recall the Dudley functional from \Cref{defn:dudfunc}
We begin by defining an upper bound on the Dudley critical radius $\deldud$ of a function class $\cH$
\begin{definition}[Upper Bound on Dudley Critical Radius] We define
\begin{align*}
\Dudfunbar_{n,2}(\cH,r) := \inf_{\updelta \le r}\left(2\updelta + \frac{4}{\sqrt{n}}\int_{\updelta}^{r}\sqrt{\mnum_2(\cH;\veps/4)}\rmd \veps\right),
\end{align*}
and
\begin{align*}
\deldudbar(\cH,c) := \inf\{r: \Dudfunbar_{n,2}(\cH,r)  \le \frac{r^2}{2c}\}.
\end{align*}
Lastly, we define
\begin{align*}
\Dudfunbar_{n,q}(\cH) := \inf_{\updelta \le B}\left(2\updelta + \frac{4}{\sqrt{n}}\int_{\updelta}^{B}\sqrt{\mnum_q(\cH;\veps/2)}\rmd \veps\right), \quad B:= \sup_{w}|h(w)|.
\end{align*}
\end{definition}

Recall $\nu_1,\nu_2$ from \Cref{eq:nu1,eq:nu2}. We have the following theorems

\begin{theorem}\label{thm:deldudbar} Suppose that \Cref{asm:well_spec,asm:conditional_completeness,asm:bounded,asm:F_compact} hold. 
Recall the class $\upbeta_{\cF} := \{\upbeta_f : f \in \cF$\}. Then, for any $\delta \in (0,1)$, if $\deldudbar(\cF,\sigB)^2 + \deldudbar(\cG,\sigB)^2 +  \frac{\sigB^2\log(1/\delta)}{n} \le c_1 \gamma$.  Then it holds that with   probability at least $1 - \delta$, each hold 
\begin{itemize}
    \item[(a)] \textbf{Upper bound on} $\Rtest[\fhatn]$. It holds that
    \begin{align}
    \Rtest[\fhatn] &\lesssim \deldudbar(\cF,\sigB)^2 + \Dudfunbar_{n,\infty}(\cG)^2 + \min\{\Dudfunbar_{n,\infty}(\cF)^2,\Dudfunbar_{n,\infty}(\upbeta_\cF)^2\} + \frac{\sigB^2\log(1/\delta)}{n}.
    \end{align}
    In particular, if \Cref{asm:F_simpler} also holds, then 
    \begin{align*}
    \Rtest[\fhatn] &\lesssim \deldudbar(\cF,\sigB)^2 + \Dudfunbar_{n,\infty}(\cG)^2 + \frac{\sigB^2\log(1/\delta)}{n}.
    \end{align*}

    \item[(b)] \textbf{Upper bound on} $\Rtrain(\fhatn,\ghatn)$. Suppose in addition that \Cref{asm:cov_shift} holds. Then, it holds that 
    \begin{align*}
    \Rtest(\fhatn,\ghatn) &\lesssim (\nu_{1}+\nu_2)\deldudbar(\cF,\sigB)^2 + \nu_1 \left(\Dudfunbar_{n,\infty}(\cG)^2 + \min\{\Dudfunbar_{n,\infty}(\cF)^2,\Dudfunbar_{n,\infty}(\upbeta_\cF)^2\}\right) \\
    &\qquad + \nu_{2}\cdot\deldudbar(\cG,\sigB)^2 + \frac{(\nu_1+\nu_2)\sigB^2\log(1/\delta)}{n}.
    \end{align*}
    In particular, if \Cref{asm:F_simpler} also holds, then 
    \begin{align*}
    \Rtest(\fhatn,\ghatn) &\lesssim (\nu_{1}+\nu_2)\deldudbar(\cF,\sigB)^2 + \nu_1 (\Dudfunbar_{n,\infty}(\cG)^2 + \nu_{2}\cdot\deldudbar(\cG,\sigB)^2 + \frac{(\nu_1+\nu_2)\sigB^2\log(1/\delta)}{n}.
    \end{align*}
    And in addition, when we upper bound $\nu_1 \le \nu_{x,y}$ and $\nu_2 \le \nu_y$,
    \begin{align*}
    \Rtest(\fhatn,\ghatn) &\lesssim \nu_{x,y}(\deldudbar(\cF,\sigB)^2 + \Dudfunbar_{n,\infty}(\cG)^2) + \nu_{y}\deldudbar(\cG,\sigB)^2 + \frac{\nu_{x,y}\sigB^2\log(1/\delta)}{n}.
    \end{align*}
\end{itemize}
\end{theorem}

\begin{proof}[Proof of \Cref{thm:nonpar_formal,thm:double_ml_formal}] The proof of both theorems is a direct consequence of \Cref{thm:deldudbar}, stated below, and the  following two lemmas to bound the Dudley integrals and critical radii, whose computations are essentially standard but which we prove in the \Cref{sec:lem:comp_dud_bounds}.
\end{proof}
    \begin{lemma}\label{lem:gl_comp} Suppose that $\cH \in \entclass_q(p,\btau,R)$. Then, $\Dudfunbar_{n,q}(\cH) \lesssim \Xiglobn(p,\btau;R)$. Hence,
    \begin{align*}
    \Dudfunbar_{n,\infty}(\cH) \le \ratestn(\cH).
    \end{align*}
    \end{lemma}
     \begin{lemma}\label{lem:dudley_local} Suppose that $\cH \in \entclass_2(p,\btau,R)$, and let $c > 0$ be arbitrary. Then,  $\deldudbar(\cH,c)^2  \lesssim \Xilocn(p,\btau,c)$. Hence, 
    \begin{align*}
    \deldudbar(\cH,c)^2  \le \rateqn[2](\cH, c).
    \end{align*}
    \end{lemma}
\subsubsection{Proof of \Cref{thm:deldudbar}}
We sketch the proof of \Cref{thm:deldudbar}, deferring supporting proofs to \Cref{sec:deldudbar_support}. We focus on part (b) of the theorem, as part (a) follows from similar arguments.
From \Cref{thm:main_guarante}, it suffices to establish the following inequalities for $c \ge 1$:
\begin{equation}\label{eq:deldudbar_bounds_thing}
\begin{aligned}
    &\deldud(\Hsum,c) \lesssim \deldudbar(\cF,c) + \deldudbar(\cG,c)\\
    &\deldud(\Fcent,c) \lesssim \deldudbar(\cF,c)\\
    &\sup_{w_{1:n}}\Dudfun_{n,\infty}(\Gcent[w_{1:n}])\lesssim  \min\{\Dudfunbar_{n,\infty}(\cF),\Dudfunbar_{n,\infty}(\upbeta_\cF)\} + \Dudfunbar_{n,\infty}(\cG).
\end{aligned}
\end{equation}
We first upper bound all relevant ``non-barred'' Dudley integrals in terms of ``barred'' integrals.
\begin{lemma}\label{lem:Dudfun_bar_ub} For any class $\cH$, 
\begin{align*}
\deldudbar(\cH,c) \ge \deldud(\cH,c), \quad \text{and} \quad \Dudfunbar_{n,\infty}(\cH) \ge \sup_{w_{1:n}}\Dudfun_{n,\infty}(\cH[w_{1:n}]).
\end{align*}
\end{lemma}
Next, we give a technical lemma which allows us to relate the Dudley integral/critical radius of class $\cH$ in terms of classes which upper bound its metric entropy.
\begin{lemma}\label{lem:Hsum_bound} Fix $a \ge 1$.
\begin{itemize}
    \item[(a)] Suppose that $\cH,\cH_1,\cH_2$ satisfy the $\ell_2$-metric entropy inequality, for all $\veps > 0$,
\begin{align*}
\mnum_{2}(\cH,\veps) \le \mnum_{2}(\cH_1,\veps/a)+\mnum_{2}(\cH_2,\veps/a).
\end{align*}
Then, it holds that $\deldudbar(\cH,c) \le a\max_{i \in [2]}\deldudbar(\cH_i,2c/a)$. In particular, if $a \ge 2$,
\begin{align*}
\deldudbar(\cH,c) \le a\max_{i \in [2]}\deldudbar(\cH_i,c).
\end{align*}
\item[(b)] Suppose instead $\cH,\cH_1,\cH_2$ satisfy the $\ell_\infty$-metric entropy inequality, for all $\veps > 0$,
\begin{align*}
\mnum_{\infty}(\cH,\veps) \le \mnum_{\infty}(\cH_1,\veps/a)+\mnum_{\infty}(\cH_2,\veps/a).
\end{align*}Then, 
$\Dudfunbar_{n,\infty}(\cH) \le a\left(\sum_{i=1}^n \Dudfunbar_{n,\infty}(\cH_i)\right)$
\end{itemize}
\end{lemma}
To apply these, we require control over the metric entropy of $\upbeta_{\cF}$.
\begin{lemma}\label{lem:beta_F_conv} Let $\upbeta_{\cF} := \{\upbeta_{f} : f \in \cF\}$. Then, as long as $\mnum_2(\cF,\veps')$ is finite for all $\veps'$,
\begin{align*}
\mnum_2(\upbeta_{\cF},\veps) \le \inf_{\veps' < \veps}\mnum_2(\cF,\veps'), \quad \text{and}\quad \mnum_\infty(\upbeta_{\cF},\veps) \le \inf_{\veps' < \veps}\mnum_\infty(\cF,\veps')
\end{align*}
\end{lemma}
To prove \Cref{lem:beta_F_conv}, we require the following qualitative statement, which can be derived from a Glivenko-Cantelli Theorem (e.g. \citet[Theorem 2.8.1]{van1996weak}, with the substitution $\cF \gets (\cH-\cH)^2$).
\begin{proposition}[Uniform Covergence of $\cL_2$ measures]\label{lem:unif_conv_L2} Let $P$ be any measure over $\bW$, let $\cH$ be any class for which $\mnum_2(\cH,\veps)$ is finite for all $\veps$. Then, for all $t > 0$
\begin{align*}
\lim_{n \to \infty}\Pr_{\bw_{1:n}\sim P}\left[\sup_{h,h' \in \cH}\|h(\bw_{1:n})-h'(\bw_{1:n})\|_{2,n} - \Exp_{\bw \sim P}[(h(\bw)-h'(\bw))^2] \ge t \right] = 0. 
\end{align*}
\end{proposition}
\begin{proof}[Proof of \Cref{lem:beta_F_conv}] 
Fix $w_{1:n} = (x_i,y_i)_{1:n} \in \cW^n$ and a slack parameter $t > 0$. Introduce the measure $P$ to be the mixture distribution $P =\frac{1}{n} \sum_{i=1}^n \Ptrain[\bx = \cdot \mid \by = y_i]$. Recall $\cH_2 = \{\upbeta_f: f \in \cF\}$, 
\begin{align*} 
\|\upbeta_{f}[w_{1:n}] - \upbeta_{f'}[w_{1:n}]\|_{2,n}^2 &= \sum_{i=1}^n (\Etrain[f(\bx) - f'(\bx) \mid \by = y_i])^2\\
&\le \sum_{i=1}^n \Etrain[(f(\bx) - f'(\bx))^2 \mid \by = y_i] \\
&= \Exp_{\bx \sim P} \Etrain[(f(\bx) - f'(\bx))^2]
\end{align*}
\Cref{lem:unif_conv_L2} implies that there must exists some number $m \ge n$ and $\tilde x_{1:m} \in \cX^m$ such that, for all $f,f' \in \cF$,
\begin{align*}
\Exp_{\bx \sim P} \Etrain[(f(\bx) - f'(\bx))^2] \le t + \|f[\tilde x_{1:m}]  - f'[\tilde x_{1:m}]\|_{2,m}^2
\end{align*}
Hence, 
\begin{align*}
\mnum_{2}( \cH_2[w_{1:n}], \veps + t)\le \mnum_{2}( \cF[\tilde x_{1:m}], \veps ) \le \mnum_{2},\veps ) .
\end{align*}
As $t,w_{1:n}$ are arbitrary, the first bound follows. To prove the second, we apply a similar argument, bounding
\begin{align*}
\|\upbeta_{f}[w_{1:n}] - \upbeta_{f'}[w_{1:n}]\|_{\infty}^2 &\le \max_{i \in [n]}\Etrain[(f(\bx) - f'(\bx))^2 \mid y = y_i], 
\end{align*}
For each $y_i$, there exists some $m_i$ and a sequence $\tilde x_{1:m_i}^{(i)}$ such that for all $f,f' \in \cF$,
\begin{align*}
 \Etrain[(f(\bx) - f'(\bx))^2 \mid \by = y_i] \le t + \|f[\tilde x_{1:m}]  - f'[\tilde x_{1:m}]\|_{2,m}^2 \le t \max_{j \in [m_i]} |\tilde x_{j}^{(i)}]  - f'[\tilde x_{j}^{(i)}]|_{\infty}^2
\end{align*}
Hence, introduce the finite set $\tilde \cX := \bigcup_{i=1}^n \bigcup_{j=1}^{m_i}\{x_j^{(i)}\}$, we have that for all $f,f' \in \cF$,
\begin{align*}
\max_{i} \Etrain[(f(\bx) - f'(\bx))^2 \mid \by = y_i] \le t+ \max_{\tilde x \in \tilde \cX} |f(\tilde x)  - f'(\tilde x)|^2.
\end{align*}
The bound follows.
\end{proof}

\Cref{eq:deldudbar_bounds_thing} is now an immediate consequence of \Cref{lem:Dudfun_bar_ub} and the following lemma.

\begin{lemma}\label{lem:all_the_deldudbar_bouds} The following bounds hold:
\begin{itemize}
    \item[(a)] $\deldudbar(\Hsum,c) \le 2\max\{\deldudbar(\cF,c),\deldudbar(\cG,c)\}$
    \item[(b)] $\deldudbar(\Fcent,c) \le 4\deldudbar(\cF,c)$.
    \item[(c)]$\Dudfunbar_{n,\infty}(\Gcent) \le  \min\{\Dudfunbar_{n,\infty}(\cF),\Dudfunbar_{n,\infty}(\upbeta_\cF)\} + \Dudfunbar_{n,\infty}(\cG)$.
\end{itemize}
\end{lemma}
\begin{proof}[Proof of \Cref{lem:all_the_deldudbar_bouds}] For all points, we apply \Cref{lem:Hsum_bound}. For (a), we can verify by the triangle inequality that for any $\cH_1, \cH_2$, 
\begin{align}
\mnum_2(\cH_1 + \cH_2,\veps) \le \mnum_2(\cH_1 ,\veps/2) + \mnum_2(\cH_2,\veps/2), \label{eq:mnum_sum}
\end{align}
which yields part (a) when specializing to $\cH_1 = \cF$, $\cH_2 = \cG$, and applying \Cref{lem:Hsum_bound}. For part (b), we observe that $\Fcent \subset (\cF - \fst) - \upbeta_{\cF}$. Thus, \Cref{fact:sub_cover}, followed by \Cref{eq:mnum_sum} and finally \Cref{lem:beta_F_conv} imply
\begin{align*}
\mnum_2(\Fcent,\veps) &\le \mnum_2(\Fcent - \upbeta_{\cF} ,\veps/2) \tag{\Cref{fact:sub_cover}}\\
&\le \mnum_2(\cF - \fst ,\veps/4) + \mnum_2( \upbeta_{\cF} ,\veps/4) \tag{\Cref{eq:mnum_sum}}\\
&= \mnum_2(\cF  ,\veps/4) + \mnum_2( \upbeta_{\cF} ,\veps/4)\\
&\le \mnum_2(\Fcent ,\veps/4) + \inf_{b > 1}\mnum_2( \cF ,\veps/4b) \tag{\Cref{lem:beta_F_conv}}\\
&\le  \inf_{b > 1}\mnum_2(\Fcent ,\veps/4b) + \mnum_2( \cF ,\veps/4b).
\end{align*}
The result now follows from \Cref{lem:Hsum_bound} with $a \gets 4b$, and taking $b \to 1$.

The proof of part (c) is similar:
\begin{align*}
\mnum_2(\Fcent,\veps) &= \mnum_2(\cG - \gst + \upbeta_{\cF} ,\veps) \tag{\Cref{fact:sub_cover}}\\
&\le \mnum_2(\cG - \gst ,\veps/2) + \mnum_2( \upbeta_{\cF} ,\veps/2) \tag{\Cref{eq:mnum_sum}}\\
&= \mnum_2(\cG  ,\veps/2) + \mnum_2( \upbeta_{\cF} ,\veps/2)\\
&\le \mnum_2(\Fcent ,\veps/2) + \inf_{b > 1}\mnum_2( \cF ,\veps/2b) \tag{\Cref{lem:beta_F_conv}}\\
&\le  \inf_{b > 1}\mnum_2(\Fcent ,\veps/4b) + \mnum_2( \cF ,\veps/2b),
\end{align*}
and follows from similar steps as part (c). 
\end{proof}

\subsubsection{Proofs of supporting Lemmas for \Cref{thm:deldudbar}}\label{sec:deldudbar_support}

\begin{fact}\label{fact:sub_cover}[Exercise 4.2.10 in \cite{vershynin2018high}]  For any sets $\bbV' \subset \bbV$, $\veps$-covering number of $\bbV'$ in any norm is at most the $\veps/2$ covering number of $\bbV$. 
\end{fact}
\begin{proof}[Proof of \Cref{lem:Dudfun_bar_ub}]  The proof of $\Dudfunbar_{n,\infty}(\cH) \ge \sup_{w_{1:n}}\Dudfun_{n,\infty}(\cH[w_{1:n}]).
$ is straightforward. To check $\deldudbar(\cH,c) \ge \deldud(\cH,c)$, we invoke \Cref{fact:sub_cover}:
\begin{align}
\sup_{w_{1:n}}\mnum_q(\cH[r,w_{1:n}];\veps/2) \le \sup_{w_{1:n}}\mnum_q(\cH[w_{1:n}];\veps/4) \le \mnum_q(\cH,\veps/4). \label{eq:cov_num_bounds_subset}
\end{align}
Thus,
\begin{align*}
\sup_{w_{1:n}}\Dudfun_{n,2}(\cH[r,w_{1:n}]) &:= \sup_{w_{1:n}}\inf_{\updelta \le R}\left(2\updelta + \frac{4}{\sqrt{n}}\int_{\updelta}^{R}\sqrt{\mnum_q(\cH[r,w_{1:n}];\veps/2)}\rmd \veps\right), \quad \text{where } R = \rad_2(\cH[r,w_{1:n}])\\
&\le \sup_{w_{1:n}}\inf_{\updelta \le r}\left(2\updelta + \frac{4}{\sqrt{n}}\int_{\updelta}^{r}\sqrt{\mnum_q(\cH[r,w_{1:n}];\veps/2)}\rmd \veps\right) \tag{$\rad_2(\cH[r,w_{1:n}]\le r$ by localization}\\
&\le\inf_{\updelta \le r}\left(2\updelta + \frac{4}{\sqrt{n}}\int_{\updelta}^{r} \sup_{w_{1:n}}\sqrt{\mnum_q(\cH[r,w_{1:n}];\veps/2)}\rmd \veps\right) \\
&\le\inf_{\updelta \le r}\left(2\updelta + \frac{4}{\sqrt{n}}\int_{\updelta}^{r} \sqrt{\mnum_q(\cH;\veps/4)}\rmd \veps\right) \tag{\Cref{eq:cov_num_bounds_subset}}\\
&:= \Dudfunbar_{n,2}(\cH,r). 
\end{align*}
Hence, 
\begin{align*}
\deldud(\cH,c) &:= \inf\left\{r: \sup_{w_{1:n}}\Dudfun_{n,2}(\cH[r,w_{1:n}]) \le \frac{r^2}{2c} \right\}\\
&\le \inf\left\{r: \Dudfunbar_{n,2}(\cH,r)  \le \frac{r^2}{2c}\right\} = \deldudbar(\cH,c).\tag*\qedhere
\end{align*}
\end{proof}
\begin{lemma}[Concavity of $\Dudfunbar$]\label{lem:Dudfunbar_concave} For all $a \ge 1$, $\Dudfunbar_{n,2}(\cH,ar) \le a\Dudfunbar_{n,2}(\cH,r)$. Hence, if $\Dudfunbar_{n,2}(\cH,r) \le \frac{r^2}{2c}$, then $\Dudfunbar_{n,2}(\cH,r') \le \frac{(r')^2}{2c}$ for $r' \ge r$. Moreover, if $b \ge 1$, $\deldudbar(\cH,br) \le b \deldudbar(\cH,r)$.
\end{lemma}
\begin{proof}[Proof of \Cref{lem:Dudfunbar_concave}]
\begin{align*}
 \Dudfunbar_{n,2}(\cH,ar) &= \inf_{\updelta \le ar} \left(2\updelta + \frac{4}{\sqrt{n}}\int_{\updelta}^{ar} \sqrt{\mnum_q(\cH;\veps/4)}\rmd \veps\right)\\
 &= \inf_{\updelta \le ar} \left(2\updelta + \frac{4a}{\sqrt{n}}\int_{\updelta/a}^{r} \sqrt{\mnum_q(\cH;a\veps/4)}\rmd \veps\right)\\
 &\le \inf_{\updelta \le ar} \left(2\updelta + \frac{4a}{\sqrt{n}}\int_{\updelta/a}^{r} \sqrt{\mnum_q(\cH;\veps/4)}\rmd \veps\right)\\
 &\le \inf_{\updelta \le r} \left(2a\updelta + \frac{4a}{\sqrt{n}}\int_{\updelta}^{r} \sqrt{\mnum_q(\cH;\veps/4)}\rmd \veps\right) = a\Dudfunbar_{n,2}(\cH,r).
 \end{align*}
 The rest of the result follows similarly to \Cref{lem:star_shaped}.
\end{proof}

\begin{proof}[Proof of \Cref{lem:Hsum_bound}] We prove part (a), the calculate for part (b) is near-identical.
\begin{align*}
\Dudfunbar_{n,2}(\cH,r)  &\le  \inf_{\updelta \le r} \left(2\updelta + \frac{4}{\sqrt{n}}\int_{\updelta}^{r} \sqrt{\mnum_q(\cH;\veps/4)}\rmd \veps\right) \\
&\le  \inf_{\updelta \le r} \left(2\updelta + \frac{4}{\sqrt{n}}\sum_{i=1}^2\int_{\updelta}^{r} \sqrt{\mnum_q(\cH_i;\veps/4a)}\rmd \veps\right) \tag{by assumption}\\
&\le  \inf_{\updelta \le r} \left(2\updelta + \frac{4a}{\sqrt{n}}\sum_{i=1}^2\int_{\updelta/a}^{r/a} \sqrt{\mnum_q(\cH_i;\veps/4)}\rmd \veps\right) \tag{$\veps \gets \veps/a$}\\
&=  a\inf_{\updelta \le r/a} \left(2\updelta + \frac{4}{\sqrt{n}}\sum_{i=1}^2\int_{\updelta}^{r/a} \sqrt{\mnum_q(\cH_i;\veps/4)}\rmd \veps\right) \tag{$\updelta \gets a\updelta$}\\
&=  a\inf_{\updelta_1,\updelta_2 \le r/a} \left(2\max\{\updelta_1,\updelta_2\} + \frac{4}{\sqrt{n}}\sum_{i=1}^2\int_{\max\{\updelta_1,\updelta_2\}}^{r/a} \sqrt{\mnum_q(\cH_i;\veps/4)}\rmd \veps\right) \\
&\le  a\inf_{\updelta_1,\updelta_2 \le r} \sum_{i=1}^2\left(2\updelta_i + \frac{4}{\sqrt{n}}\int_{\updelta_i}^{r/a} \sqrt{\mnum_q(\cH_i;\veps/4)}\rmd \veps\right) \\
&=  a\sum_{i=1}^2\Dudfunbar_{n,2}\left(\cH_i,\frac{r}{a}\right).
\end{align*}
Consequently, 
\begin{align*}
\deldudbar(\cH,c) &:= \inf\left\{r: \Dudfunbar_{n,2}(\cH,r)  \le \frac{r^2}{2c}\right\} \\
&\le \inf\left\{r: \Dudfunbar_{n,2}(\cH_1,\frac{r}{a}) + \Dudfunbar_{n,2}(\cH_2,\frac{r}{a})   \le \frac{r^2}{2ac}\right\} \\
&= a\inf\left\{r: \Dudfunbar_{n,2}(\cH_1,r) + \Dudfunbar_{n,2}(\cH_2,r)   \le \frac{a r^2}{2c}\right\}\\
&= a\inf\left\{r: \max_{i \in [2]}\Dudfunbar_{n,2}(\cH_1,r)  \le \frac{ar^2}{4c}\right\}.
\end{align*}
By \Cref{lem:Dudfunbar_concave}, the above is at most $a\max_{i \in [2]}\inf\{r: \Dudfunbar_{n,2}(\cH_1,r)  \le \frac{r^2}{2c}\} \le a\max_{i \in [2]}\deldudbar(\cH_i,2c/a)$.
\end{proof}

\subsection{Proof of Dudley Bounds (\Cref{lem:gl_comp,lem:dudley_local})}\label{sec:lem:comp_dud_bounds}

\begin{proof}[Proof of \Cref{lem:gl_comp}] 
    
        From the definition of the Dudley functional, \Cref{defn:dudfunc}, it is clear that the additive $\tau_0$-term in the metric entric bound contributes at most an additive $\frac{R\tau_0}{\sqrt{n}}$ term to the integral. Consequently, let handle what is left over, assuming throughout that $\tau_0 = 0$. For a class $\cH$, let $\phi(\veps)$ be the upper bound on $\mnum_{q}(\cH,\veps)$ prescribed by \Cref{defn:entclasses} (again, setting $\tau_0 = 0$). Then, from \Cref{defn:dudfunc}, so that
        \begin{align*}
    \Dudfunbar_{n,q}(\cH) &\le \inf_{\updelta \le R}\left(2\updelta + \frac{4}{\sqrt{n}}\int_{\updelta}^{R}\sqrt{\phi(\veps/2)}\rmd \veps\right) \lesssim \inf_{\updelta \le R}\left(\updelta + \frac{1}{\sqrt{n}}\int_{\updelta}^{R}\sqrt{\phi(\veps)}\rmd \veps\right),
    \end{align*}
    where last inequality uses similar changes of variables as in \Cref{lem:Hsum_bound}.
    When $p < 2$, we take $\updelta = 0$ and attain
    \begin{align*}
    \Dudfunbar_{n,q}(\cH) \lesssim  R\frac{1}{\sqrt{n}}\int_{0}^{R}\sqrt{\phi(\veps)}\rmd \veps \lesssim R\frac{\tau_0}{\sqrt{n}} +  R\sqrt{\frac{\tau_1}{n}}\cdot \begin{cases} \sqrt{\log(\tau_2 R)}  & p = 0\\
    \frac{1}{p/2-1}R^{1-p/2} & p \in (0,2)
    \end{cases}.
    \end{align*}
    Next, for the case $p = 2$, we pick $\updelta = R \sqrt{\tau_1/n}$ to get
    \begin{align*}
    \inf_{\updelta \le R} \updelta + \frac{1}{\sqrt{n}}\int_{\updelta}^{R}\sqrt{\phi(\veps)} \le R\frac{\tau_0}{\sqrt{n}}  + \inf_{\updelta \le R}  \updelta + \sqrt{\tau_1/n}\log(R/\updelta) \lesssim R\frac{\tau_0}{\sqrt{n}}  + \sqrt{\tau_1/n}\log(e + R\sqrt{n/\tau_1}).
    \end{align*}
    Finally, consider $p > 2$. Then, 
    \begin{align*}
    \inf_{\updelta \le R} \updelta + \frac{1}{\sqrt{n}}\int_{\updelta}^{R}\sqrt{\phi(\veps)} \le R\frac{\tau_0}{\sqrt{n}} +  \inf_{\updelta \le R}  \updelta + \frac{1}{p/2-1}\sqrt{\tau_1/n}\updelta^{1-p/2} 
    \end{align*}
    Taking $\updelta =  \frac{1}{p/2-1}\sqrt{\tau_1}\updelta^{1-p/2}$, we choose $\updelta = \tau_1^{1/p}(p/2-1)^{-2/p}$ yielding 
    \begin{align*}
    \inf_{\updelta \le R} \updelta + \frac{1}{\sqrt{n}}\int_{\updelta}^{R}\sqrt{\phi(\veps)} \lesssim R\frac{\tau_0}{\sqrt{n}}  + (\tau_1/n)^{1/p}(p/2-1)^{-2/p}.
    \end{align*}
    This concludes the proof. 
    \end{proof}
   

\begin{proof}[Proof of \Cref{lem:dudley_local}] Modifying the computation in \Cref{lem:gl_comp} implies
    \begin{align*}
    \Dudfunbar_{n,2}(\cH, r) \lesssim   \Xiglobn(p,\btau;r).
    \end{align*}
    . 
    For $p = 0$, we use that that for $r \ge c/\sqrt{n}$,
    \begin{align*}
     \Xiglobn(p,\btau ;r) = r\left(\frac{\tau_0 + \sqrt{\tau_1\log(e+\tau_2/r)}}{\sqrt{n}}\right)  \le r\left(\frac{(e+ \tau_0) + \sqrt{\tau_1\log(e+\sqrt{n}\tau_2/c)}}{\sqrt{n}}\right),
    \end{align*}
    so that 
    \begin{align*}
    \deldudbar(\cH, c) =  \inf \left\{r^2: \Dudfunbar_{n,2}(\cH, r) \le \frac{r^2}{2c}\right\} \lesssim c^2 \left(\frac{\tau_0^2 + \tau_1\log(e+\sqrt{n}\tau_2/c)}{n}\right) := \Xilocn(0,\btau;c)
    \end{align*}
    For $p \in (0,2)$,
    \begin{align*}
     \Xiglobn(p,{\btau};r) = \frac{r\tau_0 }{\sqrt{n}} + \frac{r^{1-p/2}}{1-p/2}\sqrt{\frac{\tau_1}{ n}} 
    \end{align*}
    Going forward, let $c_0$ a universal constant for which $\Dudfunbar_{n,2}(\cH, r) \le  c_0\Xiglobn(p,{\btau};r)$. Thus, one can check that
    \begin{align*}
    \inf \left\{r^2: \Dudfunbar_{n,2}(\cH,r) \le \frac{r^2}{2c}\right\} \lesssim \max\{r_1^2,r_2^2\},
    \end{align*}
    where $r_1$ and $r_2$ balance the following equations
    \begin{align*}
    \frac{r_1(1+\tau_0)}{\sqrt{n}} = r_1^2/4c_0c, \quad \frac{r_2^{1-p/2}}{1-p/2}\sqrt{\frac{\tau_1}{ n}} = r_2^2/4c_0c.
    \end{align*}
    Solving yields
    \begin{align*}
    r_1^2 = \frac{16c_0^2c^2(1+\tau_0)^2}{n}, \quad r_2^2 = \left(\frac{4c_0 c}{1-p/2}\sqrt{\frac{\tau_1}{ n}}\right)^{\frac{2}{1+p/2}} =  \left(\frac{1}{(1-p/2)^2} \cdot \frac{16c_0^2 c^2\tau_1}{ n }\right)^{\frac{2}{2+p}}
    \end{align*}
    which gives
    \begin{align*}
    \max\{r_1^2,r_2^2\} \lesssim \frac{c^2(1+\tau_0)^2}{n} + \left(\frac{c^2}{(1-p/2)^2} \cdot \frac{\tau_1}{ n }\right)^{\frac{2}{2+p}} := \Xilocn(p,\btau;c), \quad p \in (0,2).
    \end{align*} 

    For $p = 2$, we note that 
    \begin{align*}
    \Xiglobn(p,{\btau},\tau_1;r) \lesssim \left(\frac{r \tau_0 + \sqrt{\tau_1\log(e+R\sqrt{n}/\tau_1)}}{\sqrt{n}}\right).
    \end{align*}
    So similarly, 
    \begin{align*}
    \inf \left\{r^2:  \Dudfunbar_{n,2}(\cH,r) \le \frac{r^2}{2c}\right\} \lesssim \max\{r_1^2,r_3^2\},
    \end{align*}
    where $r_1$ is as above, and where $r_3$ is the smallest term satisfying
    \begin{align*}
    \frac{\sqrt{\tau_1\log(e+r_3\sqrt{n}/\tau_1)}}{\sqrt{n}} \le r_3^2/4c_0 c
    \end{align*}
    As $\frac{\sqrt{\tau_1\log(e+r_3\sqrt{n}/\tau_1)}}{\sqrt{n}}  \le r_3$ by Jensen's inequality, we can take $r_3 \le 4c_0 c$. Thus, it suffices that $r_3$ satisfy
    \begin{align*}
    r_3 \le 4c_0 c\frac{\sqrt{\tau_1\log(e+4c_0 c\sqrt{n}/\tau_1)}}{\sqrt{n}},
    \end{align*}
    so that suppressing the universal constant $c_0$,
    \begin{align*}
    r_1^2 + r_3^2 \lesssim \frac{c^2(1+\tau_0^2)}{n} + c\frac{\sqrt{\tau_1\log(e+ c\sqrt{n}/\tau_1)}}{\sqrt{n}} = \Xilocn(2,\btau;c).
    \end{align*}

    For $p > 2$, repeating the same arguments as above, we have
    \begin{align*}
    \inf \left\{r^2: \sup_{w_{1:n}} \Dudfunbar_{n,2}(\cH,r) \le \frac{r^2}{2c}\right\} \lesssim \max\{r_1^2,r_4^2\},
    \end{align*}
    where $r_1$ is as above and $r_4$ satisfies
    \begin{align*}
    r_4^2/4c_0 c = (\frac{\tau_1}{n})^{1/p}(p/2-1)^{-\frac{2}{p}}.
    \end{align*}
    so that $r_4^2 \lesssim c(\frac{\tau_1}{n})^{1/p}(p/2-1)^{-\frac{2}{p}}$. Following similar steps concludes the proof.
\end{proof}

\section{Experiment Details}
\label{app:exp}
This section describes various experiment details regarding model architectures and training hyperparameters.

\subsection{Regression}
\label{sub_app:regression}
\begin{figure}[t]
    \centering
    \includegraphics[width=0.3\linewidth]{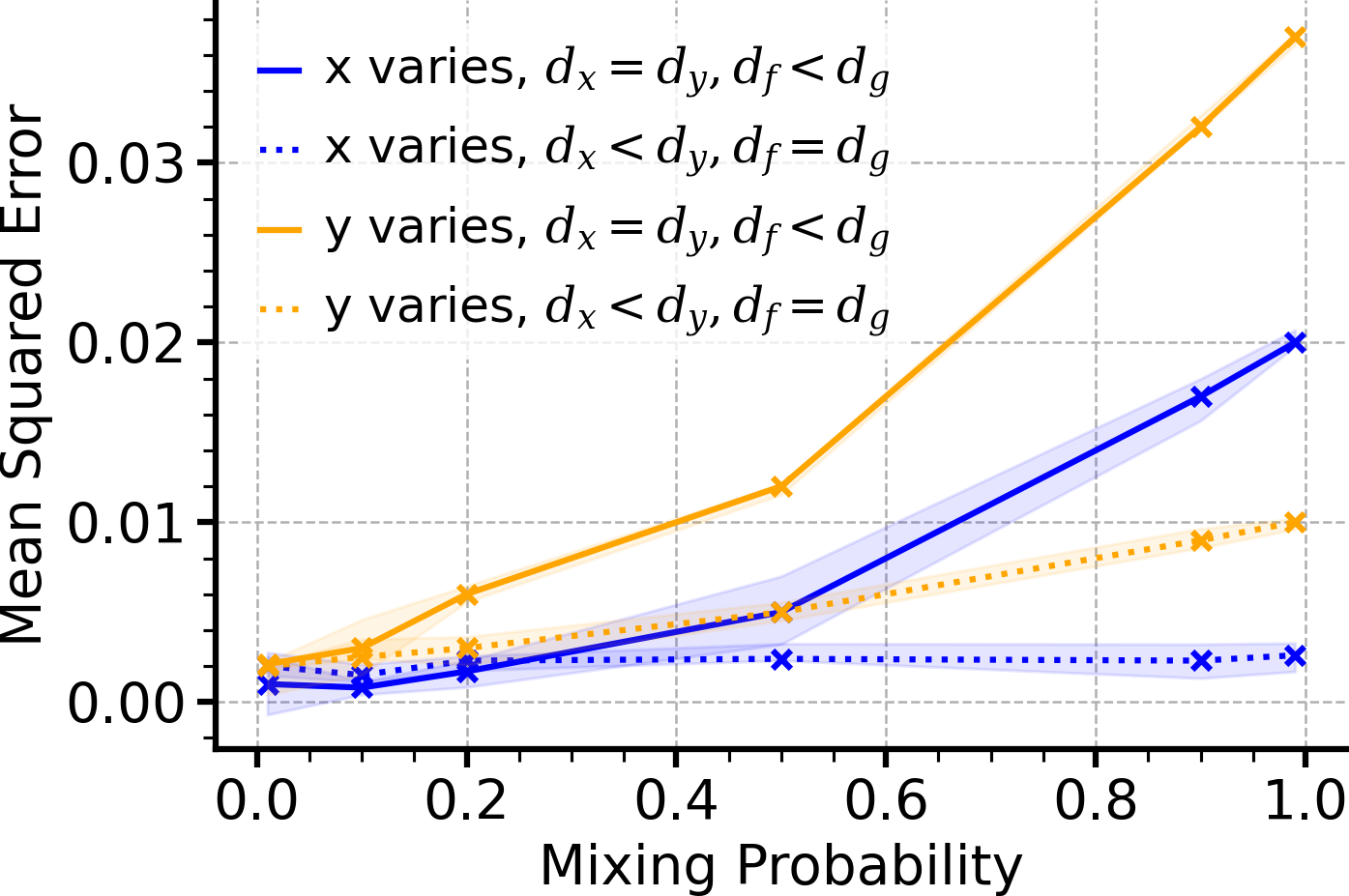}
    \includegraphics[width=0.3\linewidth]{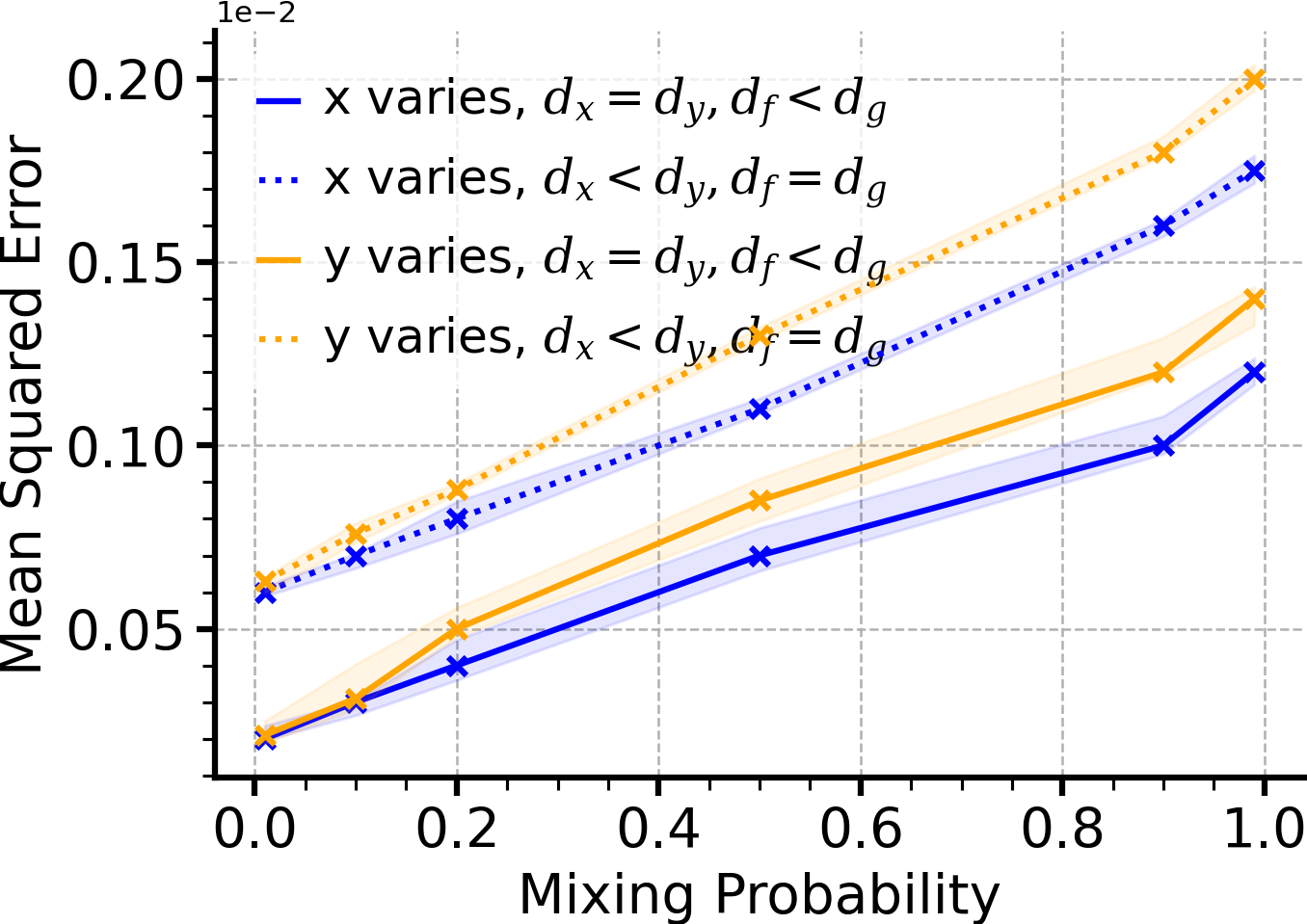}
    \vspace{-1em}
    \caption{\textbf{Mean Squared Error of predictors} of form $\hat{z} = f_\theta(x) + g_\theta(y)$ (Left) and $\hat{z} = h_\theta(x,y)$ (Right) on shifted distributions, where we hold the mixing probability of one of $\{\bx,\by\}$ fixed, and vary the other's probability. MSE for both predictors declines less with shifts in $p_x$ than those in $p_y$.}
    \label{fig:regression}
    \vspace{-1em}
\end{figure}
\paragraph{Analyzing \textit{simple} feature} To make $\bx$ simpler, we either have $d_x < d_y$ or $d_f < d_g$. We compare the generalization error (i.e., the difference between test and train mean squared error) for the auxiliary task of predicting $f_{\star}(\bx)$ and $g_{\star}(\by)$. We train $f_\phi(\bx)$ to predict $f_{\star}(\bx)$ and $g_{\phi}(\by)$ to predict $g_{\star}(\by)$. $f_\phi$ incurs a generalization error of $2.5 \times 10^{-5}$. When $d_x < d_y$, $g_\phi$ incurs a generalization error of $4.9 \times 10^{-5}$. When $d_f < d_g$, $g_\phi$ incurs a generalization error of $7.8 \times 10^{-5}$. In both cases, we observe that the auxiliary task of predicting $f_{\star}(\bx)$ has less generalization error.

\paragraph{Testing resiliency of predictors} We independently train two predictors $\hat{z} = f_\theta(x) + g_\theta(y)$ and $\hat{z} = h_\theta(x,y)$ (using concatenated features $(\bx, \by)$) to minimize mean-square error (MSE) under a training distribution with $p_x = p_y = .01$. We then measure the MSE for both predictors on shifted distributions, where we hold the mixing probability of one of $\{\bx,\by\}$ fixed, and vary the other's probability in the range $\{0.1, 0.2, 0.5, 0.9, 0.99\}$. \Cref{fig:regression} shows that MSE for both predictors declines less with shift in $p_x$ than with those in $p_y$, thereby corroborating our theoretical expectations.

\paragraph{Implementation details} We use $d_x=3$ and $d_y$ is either $3$ (when $d_x=d_y$) or $6$ (when $d_x<d_y$). Both $f(x)$ and $g(y)$ are 2-layered Multi-layered perceptions (MLP) with ReLU activation. While $f$ has a hidden dimension $d_f = 32$, $g$ has a hidden dimension of either $d_g = 32$ (when $d_f=d_g$) or $d_g=4096$ (when $d_f<d_g$). We parameterize $f_\theta(x)$ and $f_\phi(x)$ with 2-layered MLPs having ReLU activation and same hidden dimension as $f$. We parameterize $g_\theta(y)$ and $g_\phi(y)$ with 2-layered MLPs having ReLU activation and same hidden dimension as $g$. We parameterize $h_\theta(x,y)$ with a 2-layered MLP having ReLU activation and same hidden dimension as $g$. We collect $100k$ data points with $p_x=p_y=0.01$ for training $h_\theta$. We train $h_\theta$ with Adam optimizer~\citep{kingma2014adam} for $100$ epochs using a learning rate of $0.001$ and a batch size of $50$. During test, we increase $p_x$ and $p_y$ to $\{0.1, 0.2, 0.5, 0.9, 0.99\}$ and evaluate the model using $1000$ data points. 

\subsection{Binary Classification with Waterbird dataset}
\label{sub_app:waterbird}
\paragraph{Setup} Our task is to classify images of birds as waterbirds or landbirds, against a background of either land or water. These images have two high-level features: \birdtype{} and \background{}. We first empirically determine which feature is \textit{simple}. We then empirically test if the classifier is more resilient to distribution shift in the simple feature. 

\paragraph{Determining \textit{simple} feature} To determine which feature is simple, we learn two classifiers, $f_{\text{bird}}$ predicting \birdtype{} and $f_{\text{back}}$ predicting \background{}. We train them on standard training set of waterbird dataset and test them on a sampled test set that has the same distribution as the training set. While both $f_{\text{bird}}$ and $f_{\text{back}}$ have a training accuracy of $1.0$, $f_{\text{bird}}$ has a test accuracy of $0.88$ and $f_{\text{back}}$ has a test accuracy of $0.96$. Since $f_{\text{back}}$ has lower generalization error, we consider \background{} as simple.

\paragraph{Testing classifier resiliency} We test resiliency of $f_{\text{bird}}$ as we shift distribution of one feature while keeping the distribution of the other feature fixed. Specifically, we vary the proportion (in percentages) of images with waterbird (resp. land background) in test set while keeping the proportion of images with land background (resp. waterbird) fixed. We show the results in Figure~\ref{fig:comb_fig}. We observe that test accuracy of $f_{\text{bird}}$ varies less when we shift distribution of \background{} while keeping the distribution of \birdtype{} fixed.

\paragraph{Implementation details} We parameterize $f_\text{bird}$ and $f_\text{back}$ with ResNet50 model~\cite{he2016deep} and train them with stochastic gradient descent for $300$ epochs using a learning rate of $0.001$, batch size of $128$, momentum of $0.9$ and l2 regularization of $0.0001$. We took these hyperparameters and architectural choices from \citet{sagawa2019distributionally} which introduced the Waterbird dataset. We used the github repo \href{https://github.com/kohpangwei/group_DRO}{https://github.com/kohpangwei/group\_DRO} for running our experiments.

\newcommand{\landtype}{$\mathtt{landtype}$}
\newcommand{\region}{$\mathtt{region}$}
\newcommand{\zoo}{$\mathtt{zoo}$}
\newcommand{\africa}{$\mathtt{Africa}$}

\subsection{Multi-class Classification with FMoW}
\label{sub_app:fmow}
\paragraph{Setup} Our task is to predict \landtype{} from images with top down satellite view. These images also contain information about their geographical region (Africa, the Americas, Oceania, Asia, or Europe). Hence, they have two high-level features: \landtype{} and \region{}. We first empirically determine which feature is \textit{simple}. We then empirically test if the classifier is more resilient to distribution shift in the simple feature.
\begin{figure}[t]
    \centering
    \includegraphics[width=0.44\linewidth]{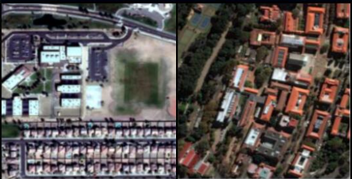}
    \hspace{2em}
    \includegraphics[width=0.3\linewidth]{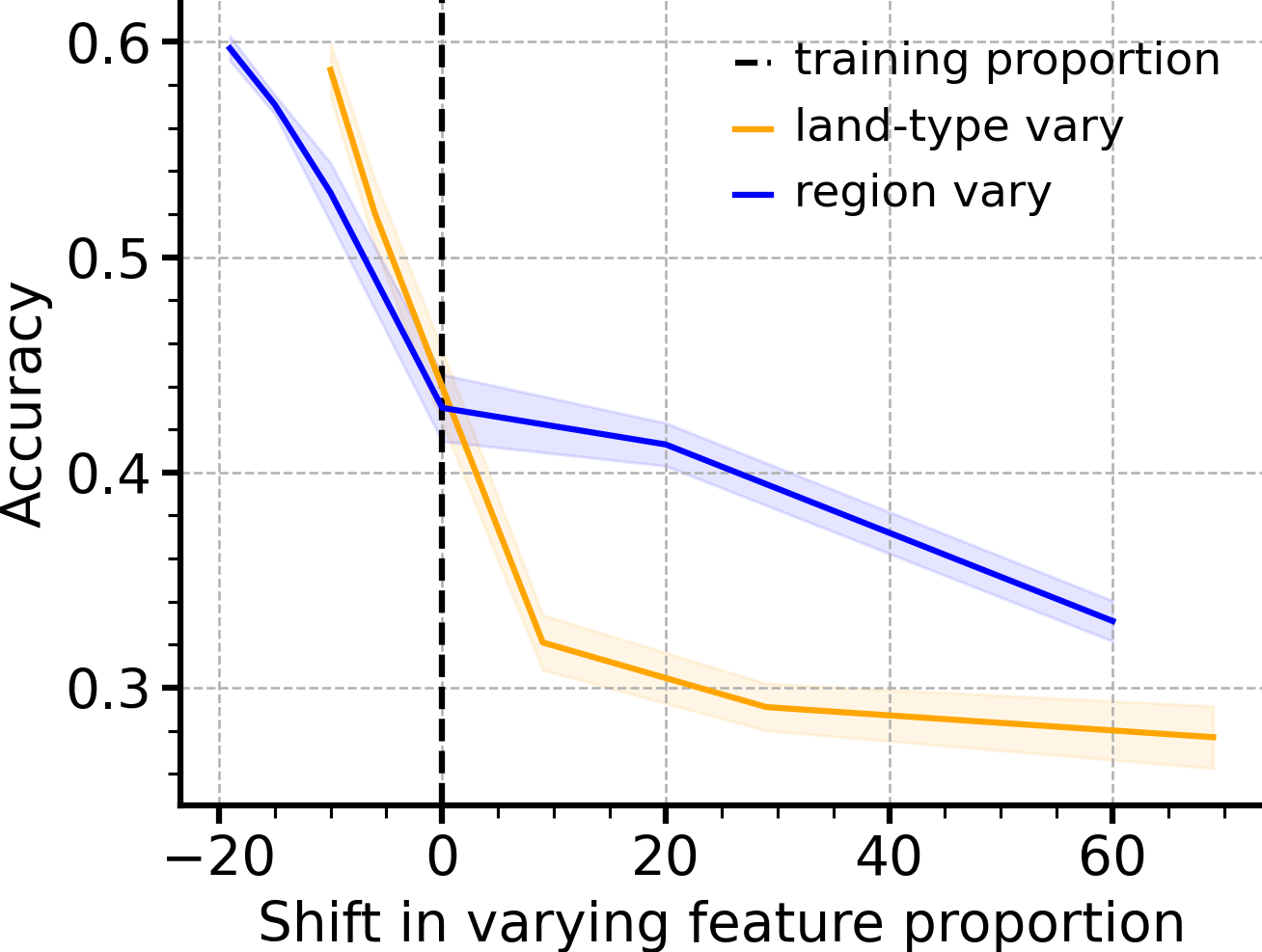}
    \caption{(Left) \textbf{Visualization of FMoW dataset}. (Right) \textbf{Test accuracy of \landtype{} classifier} as we shift distribution of \region{} (resp. \landtype{}) while keeping the distribution of \landtype{} (resp. \region{}) fixed. The results again show that the classifier is more robust to distribution shift in simple feature \region{}.}
    \label{fig:fmow}
    \vspace{-1em}
\end{figure}
\paragraph{Determining \textit{simple} feature} To determine which feature is simple, we learn two classifiers, $f_{\text{land}}$ predicting \landtype{} and $f_{\text{geo}}$ predicting \region{}. We train and test them on standard training and test set of FMoW dataset. While $f_{\text{land}}$ and $f_{\text{geo}}$ have a training accuracy (i.e. number of correct predictions/ number of datapoints) of $0.71$ and $0.74$, they have a test accuracy of $0.59$ and $0.69$ respectively. Since $f_{\text{geo}}$ has lower generalization error, we consider \region{} as simple.

\paragraph{Testing classifier resiliency} We test resiliency of $f_{\text{land}}$ as we shift distribution of one feature while keeping the distribution of the other feature fixed. To shift the distribution of \landtype{}, we vary the proportion of images with labelled as zoo (a \landtype{}). Similarly, to shift the distribution of \region{}, we vary the proportion of images from \africa{} region. We show the results in \Cref{fig:fmow}. We observe that test accuracy of $f_{\text{land}}$ varies less (i.e. $f_{\text{land}}$ is more resilient) when we shift distribution of \region{} while keeping the distribution of \landtype{} fixed.

\paragraph{Implementation details} We parameterize $f_\text{land}$ and $f_\text{geo}$ with DenseNet121 model~\citep{huang2017densely} and train them with Adam optimizer~\citep{kingma2014adam} for $60$ epochs using a learning rate of $0.0001$ and batch size of $32$. We took these hyperparameters and architectural choices from \citet{koh2021wilds}. We used the github repo \href{https://github.com/p-lambda/wilds}{https://github.com/p-lambda/wilds} for running our experiments.

\subsection{Learning logical operators with CelebA}
\label{sub_app:celeba}
\paragraph{Setup} We re-purpose the CelebA dataset to learn logical operators $\ORop$ and $\XORop$ for two attributes. We first empirically determine which attributes are \textit{simpler} than others. We then learn logical operators combining a \textit{simple} attribute and a \textit{complex} attribute. Finally, we empirically test the resilience of logical operators against distribution shifts in simple and complex attributes.
\begin{figure}[t]
    \centering
    \includegraphics[width=0.3\linewidth]{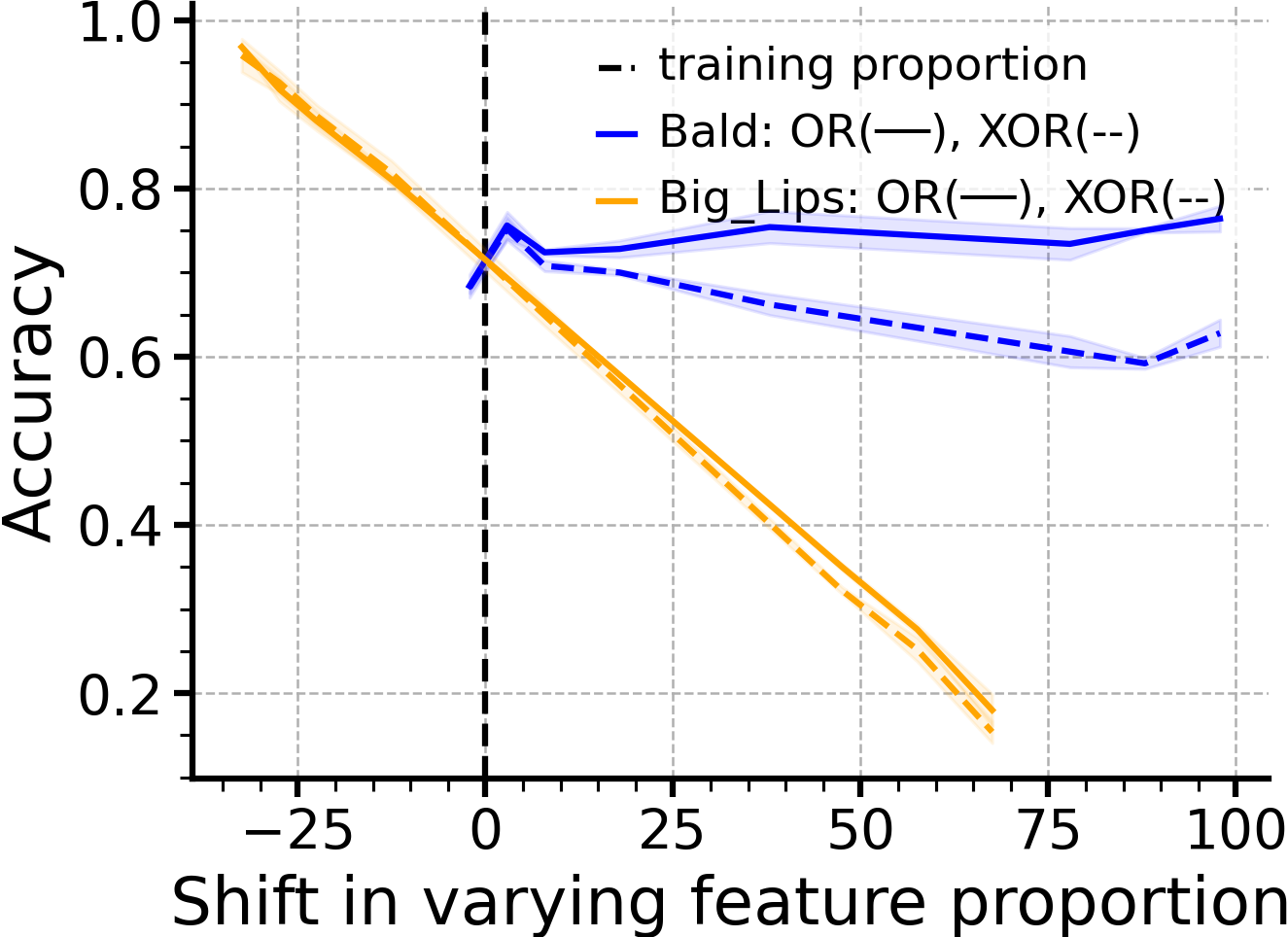}
    \includegraphics[width=0.3\linewidth]{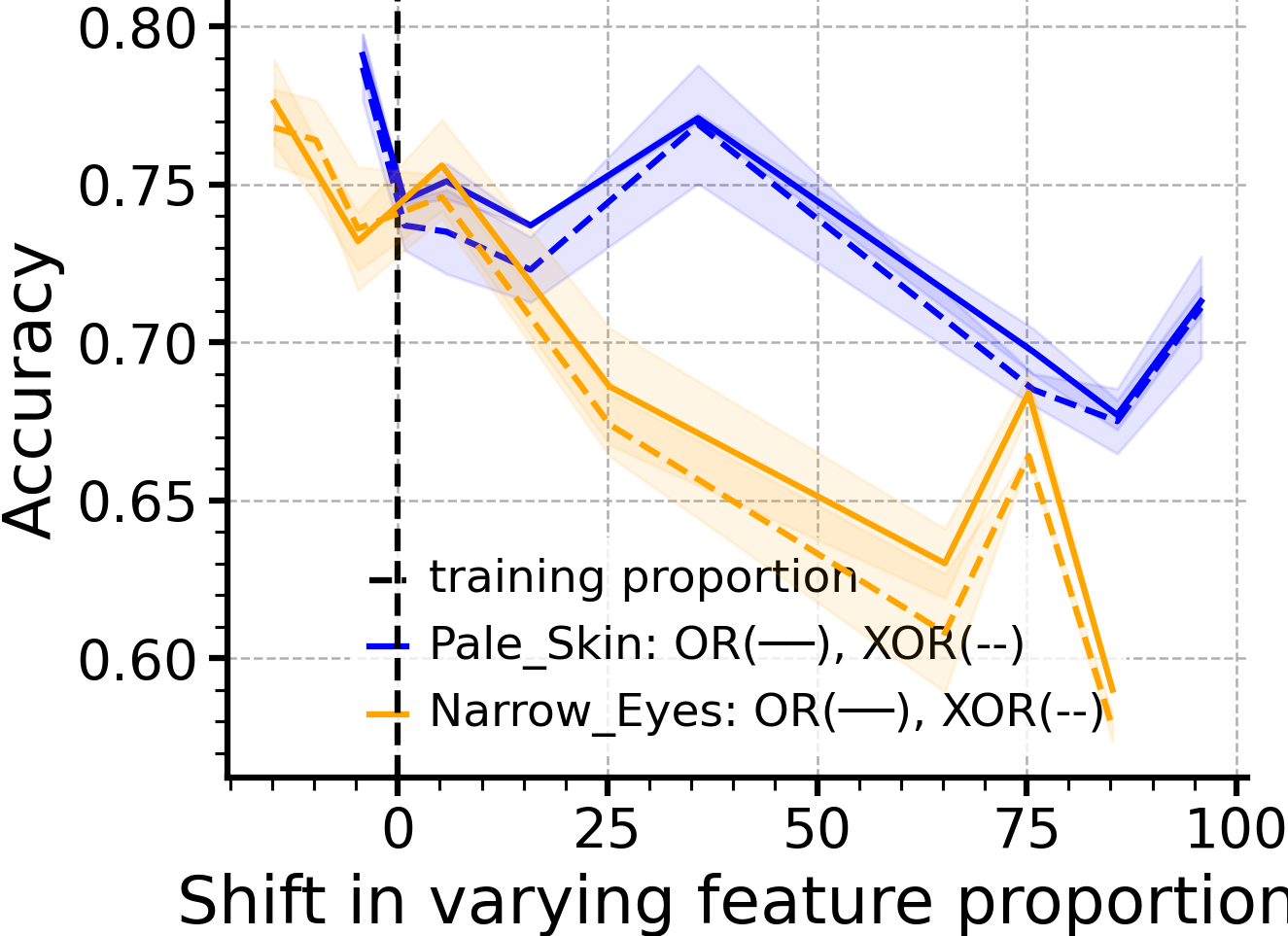}
    \vspace{-1em}
    \caption{\textbf{Test accuracy of logical operators over a pair of (\textit{simple}, \textit{complex}) attribute}. We vary the proportion (in percentages) of images with simple attribute (resp. complex attribute) in test set while keeping the proportion of images with complex attribute (resp. simple attribute) fixed. We use pairs \{$\bald$, $\biglips$\} (Left) and \{$\paleskin$, $\narroweyes$\} (Right).}
    \label{fig:celebA}
    \vspace{-1em}
\end{figure}

\paragraph{Determining \textit{simple} attributes} We first train and test a multi-head binary classifier that detects presence of $40$ different attributes, with one head per attribute, on images from the CelebA ``standard training set'' (CelebA-STS). We select attributes $\{\bald$, $\paleskin\}$ as ``simple'' due to their low generalization error, and $\{\biglips, \narroweyes\}$ as ``complex'' due to their larger generalization error. Table~\ref{tbl:celeba} shows a complete list of training and test accuracy of the multi-head binary classifier for each attribute.

\paragraph{Testing resiliency of logical operators} We first learn logical operators $f_{\ORop}$ and $f_{\XORop}$ over a pair of simple and complex attribute. We use two such pairs in our experiments: $\{\bald, \biglips\}$ and $\{\paleskin, \narroweyes\}$. We represent these logical operators as binary classifiers and train them on CelebA-STS. We get the labels for these logical operators by applying the same logical operation over labels for the simple and the complex attribute. We then test the resiliency of these logical operators by shifting the distribution of the simple attribute and the complex attribute, one at a time. Specifically, we vary the proportion (in percentages) of images with simple attribute (or complex attribute) in test set while keeping the proportion of images with complex attribute (or simple attribute) fixed. We show the results in \Cref{fig:comb_fig} and \Cref{fig:celebA}. We observe that the success rate of the logical operators vary less when we shift the distribution of the simple attributes $\{\bald, \biglips\}$.

\paragraph{Implementation details} We parameterize multi-head binary classifier, predicting presence of $40$ different facial attribute, with MobileNet~\citep{howard2017mobilenets}. We train it with Adam optimizer~\citep{kingma2014adam} for $100$ epochs using a learning rate of $0.001$ and a batch size of $64$. We use the same architecture and the training hyperparameters for learning logical operators $f_\text{OR}$ and $f_\text{XOR}$. We borrow these hyperparameters and the code for running our experiments from the github repo \href{https://github.com/suikei-wang/Facial-Attributes-Classification}{https://github.com/suikei-wang/Facial-Attributes-Classification}.

\begin{table}[t]
\centering
\begin{tabular}{|l|l|l|}
\hline
Attribute             & Train Accuracy             & Test Accuracy              \\ \hline
5\_o\_Clock\_Shadow   & 0.952 $\pm$ 0.003          & 0.945 $\pm$ 0.008          \\
Arched\_Eyebrows      & 0.879 $\pm$ 0.01           & 0.84 $\pm$ 0.006           \\
Attractive            & 0.846 $\pm$ 0.001          & 0.828 $\pm$ 0.003          \\
Bags\_Under\_Eyes     & 0.872 $\pm$ 0.004          & 0.845 $\pm$ 0.006          \\
\textbf{Bald}         & \textbf{0.991 $\pm$ 0.002} & \textbf{0.988 $\pm$ 0.004} \\
Bangs                 & 0.968 $\pm$ 0.006          & 0.96 $\pm$ 0.005           \\
\textbf{Big\_Lips}    & \textbf{0.8 $\pm$ 0.008}   & \textbf{0.716 $\pm$ 0.009} \\
Big\_Nose             & 0.867 $\pm$ 0.002          & 0.839 $\pm$ 0.004          \\
Black\_Hair           & 0.918 $\pm$ 0.007          & 0.896 $\pm$ 0.006          \\
Blond\_Hair           & 0.963 $\pm$ 0.005          & 0.957 $\pm$ 0.003          \\
Blurry                & 0.966 $\pm$ 0.008          & 0.961 $\pm$ 0.005          \\
Brown\_Hair           & 0.886 $\pm$ 0.009          & 0.885 $\pm$ 0.009          \\
Bushy\_Eyebrows       & 0.929 $\pm$ 0.007          & 0.922 $\pm$ 0.002          \\
Chubby                & 0.964 $\pm$ 0.003          & 0.948 $\pm$ 0.005          \\
Double\_Chin          & 0.971 $\pm$ 0.002          & 0.961 $\pm$ 0.006          \\
Eyeglasses            & 0.998 $\pm$ 0.003          & 0.996 $\pm$ 0.006          \\
Goatee                & 0.978 $\pm$ 0.002          & 0.974 $\pm$ 0.003          \\
Gray\_Hair            & 0.984 $\pm$ 0.005          & 0.98 $\pm$ 0.002           \\
Heavy\_Makeup         & 0.939 $\pm$ 0.007          & 0.918 $\pm$ 0.005          \\
High\_Cheekbones      & 0.898 $\pm$ 0.006          & 0.876 $\pm$ 0.004          \\
Male                  & 0.989 $\pm$ 0.001          & 0.979 $\pm$ 0.003          \\
Mouth\_Slightly\_Open & 0.957 $\pm$ 0.004          & 0.936 $\pm$ 0.008          \\
Mustache              & 0.975 $\pm$ 0.005          & 0.97 $\pm$ 0.007           \\
\textbf{Narrow\_Eyes} & \textbf{0.915 $\pm$ 0.002} & \textbf{0.875 $\pm$ 0.004} \\
No\_Beard             & 0.971 $\pm$ 0.009          & 0.96 $\pm$ 0.007           \\
Oval\_Face            & 0.795 $\pm$ 0.002          & 0.758 $\pm$ 0.005          \\
\textbf{Pale\_Skin}   & \textbf{0.97 $\pm$ 0.008}  & \textbf{0.967 $\pm$ 0.005} \\
Pointy\_Nose          & 0.795 $\pm$ 0.011          & 0.774 $\pm$ 0.009          \\
Receding\_Hairline    & 0.952 $\pm$ 0.003          & 0.939 $\pm$ 0.008          \\
Rosy\_Cheeks          & 0.959 $\pm$ 0.004          & 0.95 $\pm$ 0.009           \\
Sideburns             & 0.981 $\pm$ 0.002          & 0.978 $\pm$ 0.003          \\
Smiling               & 0.948 $\pm$ 0.006          & 0.928 $\pm$ 0.005          \\
Straight\_Hair        & 0.856 $\pm$ 0.003          & 0.831 $\pm$ 0.009          \\
Wavy\_Hair            & 0.87 $\pm$ 0.003           & 0.833 $\pm$ 0.004          \\
Wearing\_Earrings     & 0.923 $\pm$ 0.004          & 0.9 $\pm$ 0.002            \\
Wearing\_Hat          & 0.994 $\pm$ 0.009          & 0.989 $\pm$ 0.005          \\
Wearing\_Lipstick     & 0.946 $\pm$ 0.003          & 0.934 $\pm$ 0.007          \\
Wearing\_Necklace     & 0.895 $\pm$ 0.004          & 0.87 $\pm$ 0.006           \\
Wearing\_Necktie      & 0.97 $\pm$ 0.001           & 0.965 $\pm$ 0.002          \\
Young                 & 0.912 $\pm$ 0.002          & 0.877 $\pm$ 0.004          \\ \hline
\end{tabular}
\caption{Train and test accuracy of multi-head binary classifier for each attribute in CelebA dataset. We consider attributes with low generalization error as \textit{simple} and attributes with high generalization error as \textit{complex}. We highlight the selected \textit{simple} (Bald, Pale Skin) and \textit{complex} (Big Lips, Narrow Eyes) features. We report mean and standard error across 4 replicates.}
\label{tbl:celeba}
\end{table}

\subsection{Imitation learning on Robotic pusher arm environment}
\label{sub_app:robotic}
\paragraph{Environment Description} We use Robotic pusher arm environment adapted from \citet{ajay2022distributionally, gupta2018meta} where the goal is to push the red cube to the green circle. When the red cube reaches the green circle, the agent gets a
reward of +1. The state space is 12-dimensional consisting of mass of red cube (1), dampness parameter for each joint (1), joint angles (3) and velocities (3) of
the gripper, COM of the gripper (2) and position of the red cube (2). The green circle’s position is
fixed and at an initial distance of 0.5 from COM of the gripper. The red cube (of size 0.03) is initially
at a distance of 0.1 from COM of the gripper and at an angle $\pi/4$. During training, at beginning for every episode, we sample $m \sim \mathcal{N}(60, 15)$ and $d \sim \mathcal{N}(0.5, 0.1)$. The task horizon is 60 timesteps.

\paragraph{Expert Policy} To obtain expert policy that provides data for imitation learning and for training dynamics models ($p_{\phi_i}(s_{t+1} | s_t, a_t, i)$, $i \in \{d,m\}$), we train a policy $\pi_{\text{exp}}(a|s,m,d)$ with Soft-Actor-Critic~\citep{haarnoja2018soft} for $10e6$ environment steps.

\paragraph{Determining \textit{simple} factor} To determine which of the two is ``simpler'', we measure generalization error on the auxillary task of predicting next-step dynamics where one of $\{m,d\}$ is held fixed, and the other drawn from a certain distribution, fixed across both testing and training. We learn two dynamics model $p_{\phi_m}(s_{t+1}|s_t, a_t, m)$ and $p_{\phi_d}(s_{t+1}|s_t, a_t, d)$ on two separate datasets $\mathcal{D}_{\text{dyn},m}^{train}$ and $\mathcal{D}_{\text{dyn},d}^{train}$. Both $\mathcal{D}_{\text{dyn},m}^{train}$ and $\mathcal{D}_{\text{dyn},d}^{train}$ contain $100$ expert trajectories each, with varying $m$ and $d$ respectively while keeping the other factor fixed. While $m \sim \mathcal{N}(60,15)$ and $d=0.5$ in $\mathcal{D}_{\text{dyn},m}^{train}$, $d \sim \mathcal{N}(0.5,0.1)$ and $m=60$ in $\mathcal{D}_{\text{dyn},d}^{train}$. To evaluate learned dynamics models, we generate $\mathcal{D}_{\text{dyn},m}^{test}$ and $\mathcal{D}_{\text{dyn},d}^{test}$ in same way as their training counterparts. On training datasets, we find mean squared error ($\frac{1}{|\mathcal{D}|} \|s_{t+1} - p_{\phi_i}(s_{t+1}|s_t,a_t,i) \|_2^2$, $i \in \{m,d\}$) of $p_{\phi_m}(s_{t+1}|s_t, a_t, m)$ and $p_{\phi_d}(s_{t+1}|s_t, a_t, d)$ to be $0.062$ and $0.183$ respectively. On test datasets, we find the mean squared error of $p_{\phi_m}(s_{t+1}|s_t, a_t, m)$ and $p_{\phi_d}(s_{t+1}|s_t, a_t, d)$ to be $0.081$ and $0.237$ respectively. Since $p_{\phi_m}(s_{t+1}|s_t, a_t, m)$ has smaller generalization error, we consider object $\mass$ as simple. Intuitively, object $\mass$ only affects the dynamics of the system when the robotic arm is in contact with the object. In contrast, the joints' $\dampness$ affects the way the robotic arm moves and hence affects the system's dynamics independent of whether the robotic arm is in contact with the object.

\paragraph{Implementation details} We parameterize dynamics model ($p_{\phi_i}(s_{t+1} | s_t, a_t, i)$, $i \in \{d,m\}$), expert policy $\pi_{\text{exp}}(a|s,m,d)$ and imitator policy $\pi_\theta(a|s,m,d)$ with a $3$-layered Multi-layer perception (MLP) having hidden dimension of $512$ and ReLU activation. We train dynamics model and imitator policy with Adam optimizer~\citep{kingma2014adam} for $100$ epochs using a learning rate of $0.001$ and a batch size of $128$. 

\end{document}